\documentclass{article}

\usepackage{microtype}
\usepackage{graphicx}
\usepackage{booktabs} % for professional tables

\usepackage[accepted]{icml2024}
\usepackage{hyperref}
\definecolor{mydarkblue}{rgb}{0,0.08,0.45}
  \hypersetup{ %
    pdftitle={},
    pdfsubject={Proceedings of the International Conference on Machine Learning 2024},
    pdfkeywords={},
    pdfborder=0 0 0,
    pdfpagemode=UseNone,
    colorlinks=true,
    linkcolor=mydarkblue,
    citecolor=mydarkblue,
    filecolor=mydarkblue,
    urlcolor=mydarkblue,
}

\usepackage{graphicx}
\usepackage{subfigure}
\usepackage{subcaption}
\usepackage{float}
\usepackage{overpic}
\usepackage{makecell}
\usepackage{algorithm,algorithmic}
\usepackage{paralist,amsmath, amssymb,bm}
\usepackage{multirow}
\usepackage{braket}
\usepackage{threeparttable}
\usepackage{enumitem}
\usepackage{textcase}
\usepackage{booktabs}
\usepackage{fancybox}
\usepackage{tabularx}
\usepackage{adjustbox}

\usepackage{amsmath}
\usepackage{amssymb}
\usepackage{mathtools}
\usepackage{amsthm}

\usepackage[capitalize,noabbrev]{cleveref}

\theoremstyle{plain}
\newtheorem{theorem}{Theorem}[section]
\newtheorem{proposition}[theorem]{Proposition}
\newtheorem{lemma}[theorem]{Lemma}
\newtheorem{corollary}[theorem]{Corollary}
\theoremstyle{definition}
\newtheorem{definition}[theorem]{Definition}
\newtheorem{assumption}[theorem]{Assumption}
\theoremstyle{remark}
\newtheorem{remark}[theorem]{Remark}
\crefname{assumption}{Assumption}{Assumptions}

\usepackage[textsize=tiny]{todonotes}

\newcommand{\E}{\mathbb E}
\newcommand{\R}{\mathbb R}

\newcommand{\w}{\mathbf{w}}
\newcommand{\W}{\mathcal{W}}
\newcommand{\OO}{\mathcal{O}}
\newcommand{\q}{\mathbf{q}}
\newcommand{\z}{\mathbf{z}}
\newcommand{\Z}{\mathcal{Z}}
\newcommand{\g}{\mathbf{g}}
\newcommand{\x}{\mathbf{x}}
\newcommand{\y}{\mathbf{y}}
\newcommand{\argmin}{\mathop{\arg\min}}
\newcommand{\norm}[1]{\left\Vert#1\right\Vert}
\newcommand{\sumt}{\sum_{t=1}^T}
\newcommand{\summ}{\sum_{i=1}^m}
\newcommand{\sumn}{\sum_{j=1}^{n_i}}
\newcommand{\sums}{\sum_{s=0}^{S-1}}

\newcommand{\sumks}{\sum_{k=1}^{K_s}}
\newcommand{\sumkz}{\sum_{k=0}^{K-1}}
\newcommand{\sumkzs}{\sum_{k=0}^{K_s-1}}
\newcommand{\sqbrac}[1]{\left[#1\right]}
\newcommand{\brac}[1]{\left(#1\right)}
\newcommand{\cbrac}[1]{\left\{#1\right\}}
\newcommand{\maxz}{\max_{\z\in\Z}}

\icmltitlerunning{Efficient Algorithms for Empirical Group Distributionally Robust Optimization and Beyond}

\begin{document}

\twocolumn[
\icmltitle{Efficient Algorithms for Empirical\\ Group Distributionally Robust Optimization and Beyond}





\begin{icmlauthorlist}
\icmlauthor{Dingzhi Yu}{nju,fd}
\icmlauthor{Yunuo Cai}{fd}
\icmlauthor{Wei Jiang}{nju}
\icmlauthor{Lijun Zhang}{nju,pzlab}
% \icmlauthor{Firstname1 Lastname1}{equal,yyy}
% \icmlauthor{Firstname2 Lastname2}{equal,yyy,comp}
% \icmlauthor{Firstname3 Lastname3}{comp}
\end{icmlauthorlist}



\icmlaffiliation{fd}{School of Data Science, Fudan University, Shanghai, China}
\icmlaffiliation{nju}{National Key Laboratory for Novel Software Technology, Nanjing University, Nanjing, China}
\icmlaffiliation{pzlab}{Pazhou Laboratory (Huangpu), Guangzhou, China}

\icmlcorrespondingauthor{Lijun Zhang}{zhanglj@lamda.nju.edu.cn}

\icmlkeywords{Group Distributionally Robust Optimization, Minimax Excess Risk Optimization}

\vskip 0.3in
]



\printAffiliationsAndNotice{}  % leave blank if no need to mention equal contribution

\begin{abstract}
In this paper, we investigate the empirical counterpart of Group Distributionally Robust Optimization (GDRO), which aims to minimize the maximal empirical risk across $m$ distinct groups. We formulate empirical GDRO as a \textit{two-level} finite-sum convex-concave minimax optimization problem and develop an algorithm called ALEG to benefit from its special structure. ALEG is a double-looped stochastic primal-dual algorithm that incorporates variance reduction techniques into a modified mirror prox routine. To exploit the two-level finite-sum structure, we propose a simple group sampling strategy to construct the stochastic gradient with a smaller Lipschitz constant and then perform variance reduction for all groups. Theoretical analysis shows that ALEG achieves $\varepsilon$-accuracy within a computation complexity of $\mathcal{O}\left(\frac{m\sqrt{\bar{n}\ln{m}}}{\varepsilon}\right)$, where $\bar n$ is the average number of samples among $m$ groups. Notably, our approach outperforms the state-of-the-art method by a factor of $\sqrt{m}$. Based on ALEG, we further develop a two-stage optimization algorithm called ALEM to deal with the empirical Minimax Excess Risk Optimization (MERO) problem. The computation complexity of ALEM nearly matches that of ALEG, surpassing the rates of existing methods.
\end{abstract}

\section{Introduction}

Recently, a popular class of Distributionally Robust Optimization (DRO) problem named as Group DRO (GDRO), has drawn significant attention in machine learning~\citep{oren-etal-2019-distributionally,mohri2019agnostic,carmon2022distributionally,zhang2023stochastic,mehta2024distributionally}. GDRO optimizes the maximal risk over a group of $m$ distributions, which can be formulated as the following stochastic minimax optimization problem:
\begin{equation}
    \min_{\w\in\W}\max_{i\in[m]}\ \E_{\xi\sim\mathcal{P}_i}\sqbrac{\ell(\w;\xi)},\label{eq:gdro}
\end{equation}
where $\{\mathcal{P}_i\}_{i\in[m]}$ is a group of distributions and $\ell(\w;\xi)$ is the loss function, measuring the predictive error of the model $\w\in\W$ for a sample $\xi$. Similar to Empirical Risk Minimization (ERM)~\citep{shalev2014understanding}, which replaces the risk minimization problem over a distribution by its empirical counterpart, this paper focuses on the empirical variant of problem~\eqref{eq:gdro}:
\begin{equation}
    \min_{\w\in\W}\max_{i\in[m]}\ \left\{R_i(\w):=\frac{1}{n_i}\sumn\ell(\w;\xi_{ij})\right\},
    \label{eq:empirical-gdro}
\end{equation}
where $\cbrac{R_i(\cdot)}_{i\in[m]}$ are empirical risk functions for each group, $n_i$ is the number of samples for group $i$, and $\xi_{ij}$ is the $j$-th sample in group $i$. We further denote $\bar n:=\frac{1}{m}\summ n_i$ as the average number of samples for $m$ groups. Many popular machine learning tasks can be mathematically modeled by (empirical) GDRO, including federated learning~\citep{mohri2019agnostic,laguel2021superquantile,laguel2021superquantiles}, robust language modeling~\citep{oren-etal-2019-distributionally}, robust neural network training~\citep{sagawa2019distributionally}, and collaborative PAC learning~\citep{NIPS2017_186a157b,NEURIPS2018_3569df15,rothblum2021multi}. 
 
To leverage the powerful first-order optimization methods, we transform~\eqref{eq:empirical-gdro} into an equivalent finite-sum convex-concave saddle-point problem~\citep{nemirovski2009robust}:
\begin{equation}
\min_{\w\in\W}\max_{\q\in\Delta_m}\left\{F(\w,\q):=\summ\q_iR_i(\w)\right\},
    \label{eq:empirical-gdro-wq}
\end{equation}
where $\Delta_m=\{\q\ge\mathbf{0},\mathbf{1}^T\q=1|\q\in\R^m\}$ is the $(m-1)$-dimensional simplex, and $\cbrac{R_i(\cdot)}_{i\in[m]}$ are assumed to be convex. 

The existing vast amount of work on general convex-concave optimization can be applied to~\eqref{eq:empirical-gdro-wq}. For instance, Stochastic Mirror Descent (SMD)~\citep{nemirovski2009robust}, originally targeting problem~\eqref{eq:gdro}, can be utilized for~\eqref{eq:empirical-gdro-wq} with a complexity of $\OO\brac{\frac{m\ln{m}}{\varepsilon^2}}$. However, the quadratic dependency on $\varepsilon$ prevents us from getting a high-accuracy solution. 
A recent work of~\citet{alacaoglu2022stochastic} proposes Mirror Prox with Variance Reduction (MPVR) to solve the general finite-sum convex-concave optimization problem. Running MPVR on~\eqref{eq:empirical-gdro-wq}, we get an $\OO\brac{m\bar n+\frac{m\sqrt{m\bar n\ln{m}}}{\varepsilon}}$ complexity, but this rate is suboptimal for~\eqref{eq:empirical-gdro-wq}, as shown below.

Inspired by MPVR, this paper goes go one step further by exploiting the \textit{two-level} structure of~\eqref{eq:empirical-gdro-wq} to derive a better $\OO\brac{\frac{m\sqrt{\bar n\ln{m}}}{\varepsilon}}$ complexity. We call~\eqref{eq:empirical-gdro-wq} ``two-level finite-sum'' because $F(\w,\q)$ can be decomposed into two nested finite summations. Similar to MPVR, we also incorporate variance reduction into the stochastic mirror prox algorithm, but with a novel way to construct the stochastic gradients. Specifically, we develop a simple yet effective group sampling technique, which uses $m$ samples, one for each group, to construct the stochastic gradients, and in this way, the Lipschitz constant is reduced by a factor of $1/m$. Then, we employ variance reduction techniques within each group to attain faster convergence. Our algorithm, named as ALEG, improves the state-of-the-art computation complexity by $\sqrt{m}$. Another advantage of ALEG is its support for changeable hyperparameters, achieved by leveraging the one-index-shifted weighted average to compute (mirror) snapshot points and Lyapunov functions. This enhances ALEG's practical effectiveness, as variable learning rates have proven beneficial in real-world machine learning tasks~\citep{liner2021improving}.

Furthermore, we extend our methodology to solve the empirical Minimax Excess Risk Optimization (MERO) problem~\citep{agarwal2022minimax} given below:
\begin{equation}
\min_{\w\in\W}\max_{i\in[m]}\cbrac{\underline{R}_i(\w):=R_i(\w)-R_i^*},\label{eq:empirical-mero}
\end{equation}' 
where we define $R_i^*:=\min_{\w\in\W}R_i(\w)$ as the minimal empirical risk of group $i$, usually unknown. Compared with empirical GDRO, empirical MERO replaces the raw empirical risk $R_i(\w)$ with the \textit{excess empirical risk} $\underline{R}_i(\w):=R_i(\w)-R_i^*$ to cope with heterogeneous noise. We provide an efficient two-stage \textbf{Al}gorithm for \textbf{E}mpirical \textbf{M}ERO (ALEM) by following the two-stage schema~\citep{zhang2023efficient}, and making use of our empirical GDRO algorithm. In the first stage, $R_i^*$ is estimated by an approximate minimal empirical risk $\hat R_i^*$ through running ALEG for $m$ groups. In the second stage, we approximate the original problem by replacing the minimal empirical risks $\cbrac{R_i^*}_{i\in[m]}$ with the estimated values $\{\hat{R}_i^*\}_{i\in[m]}$, and optimize the new problem by ALEG. We demonstrate that the computation complexity of ALEM is $\tilde{\OO}\brac{\frac{m\sqrt{\bar n\ln{m}}}{\varepsilon}}$\footnote{We use the $\tilde\OO$ notation to hide constant factors as well as logarithmic factors in $\varepsilon$.}, which nearly matches that of the empirical GDRO problem.

\section{Related Work}

\begin{table*}[t]
\renewcommand\arraystretch{2}
    \begin{center}
    \begin{threeparttable}
    \caption{Comparisons of computation complexities for empirical GDRO.\tnote{*}}
        % }
        % \protect\footnotemark[2]}
        \begin{tabular}{cc}
            % \hline
\toprule
            \textbf{Algorithm} & \textbf{Computation Complexity} \\
            % \hline
\midrule
            SMD~\citep{nemirovski2009robust}  & $\OO\brac{\frac{m\ln m}{\varepsilon^2}}$\\
% \hline
            AL-SVRE~\citep{luo2021near}& $\OO\brac{\brac{m\bar n+ \frac{m\sqrt{m\bar n\ln m}}{\varepsilon}  +m^{5/4}\bar n^{3/4}\sqrt{\frac{\ln m}{\varepsilon}}}\ln{\frac{m}{\varepsilon}}}$\\
% \hline
            BROO-KX~\citep{carmon2022distributionally}& $\OO\brac{\frac{m\bar n}{\varepsilon^{2/3}}\ln^{14/3}(\frac{\ln m}{\varepsilon})+\frac{(m\bar n)^{3/4}}{\varepsilon}\ln^{7/2}(\frac{\ln m}{\varepsilon})}$\tnote{**}\\
% \hline
            MPVR~\citep{alacaoglu2022stochastic} & $\OO\brac{m\bar n+\frac{m\sqrt{m\bar n\ln{m}}}{\varepsilon}}$\\
% \hline
            \textbf{ALEG}~(\cref{thm:gdro-convergence,cor:gdro-complexity}) & $\OO\brac{\frac{m\sqrt{\bar n\ln{m}}}{\varepsilon}}$\\
            % \hline
\bottomrule
        \end{tabular}        
        \begin{tablenotes}
        \footnotesize
        \item[*] The results are transformed in our formulation in terms of $m,\bar n$.
        \item[**] BROO-KX relies on expensive bisection operations, which is not reflected in this complexity measure.
        % This result is measured in terms of the number of $\ell(\w;\xi_{ij})$ and $\nabla\ell(\w;\xi_{ij})$ evaluations, which is smaller than overall computation complexity.        
        % \item[$\dagger$] my exp is ...
      \end{tablenotes}\label{tab:complexity-comparison}
    \end{threeparttable}
    \end{center}
\end{table*}

In this section, we review and compare some previous work that can be used to solve the empirical GDRO problem in~\eqref{eq:empirical-gdro-wq} and the empirical MERO problem in~\eqref{eq:empirical-mero-wq}. 

\subsection{Group Distributionally Robust Optimization}

For the original GDRO,~\citet{nemirovski2009robust} propose a stochastic approximation (SA)~\citep{robbins1951stochastic} approach named Stochastic Mirror Descent (SMD) and obtain a sample complexity of $\OO\brac{\frac{m\ln{m}}{\varepsilon^2}}$, which nearly matches the lower bound $\Omega\brac{\frac{m}{\varepsilon^2}}$~\citep{soma2022optimal} up to a logarithmic factor.~\citet{soma2022optimal} cast problem~\eqref{eq:gdro} as a two-player zero-sum game to utilize SMD and no-regret algorithms from multi-armed bandits (MAB). Recently,~\citet{zhang2023stochastic} refine their analysis by borrowing techniques from non-oblivious MAB~\citep{NIPS2015neu} to establish a sample complexity of $\OO\brac{\frac{m\ln{m}}{\varepsilon^2}}$.
 
Although SA approaches primarily aim at solving~\eqref{eq:gdro}, they can also be directly applied to empirical GDRO. Running SMD or non-oblivious online learning algorithm on~\eqref{eq:empirical-gdro-wq} yields a computation complexity of $\OO\brac{\frac{m\ln{m}}{\varepsilon^2}}$, which serves as the baseline of the problem. Although the complexity does not depend on the average number of samples $\bar n$, it suffers from a quadratic dependency on $\varepsilon$.

\subsection{Empirical GDRO and Empirical MERO}

\citet{carmon2022distributionally} optimize~\eqref{eq:empirical-gdro-wq} directly by leveraging exponentiated group-softmax and the Nesterov smoothing~\citep{nesterov2005smooth} technique. Then they run a Ball Regularized Optimization Oracle~\citep{carmon2020acceleration,carmon2021thinking} on Katyusha~X~\citep{pmlr-v80-allen-zhu18a} (BROO-KX) to get a complexity of $\OO\brac{\frac{m\bar n}{\varepsilon^{2/3}}\ln^{14/3}(\frac{\ln m}{\varepsilon})+\frac{(m\bar n)^{3/4}}{\varepsilon}\ln^{7/2}(\frac{\ln m}{\varepsilon})}$. We note that their complexity is measured by the number of oracle queries, i.e., $\ell(\cdot;\xi_{ij})$ and $\nabla\ell(\cdot;\xi_{ij})$ evaluations, instead of total computations. In fact, the BROO-KX algorithm includes expensive bisections, which are computationally expensive.

As an extension of the empirical GDRO problem,~\citet{agarwal2022minimax} study the empirical MERO problem. They cast~\eqref{eq:empirical-mero} as a two-player zero-sum game:\begin{equation}
    \min_{\w\in\W}\max_{\q\in\Delta_m}\left\{\underline{F}(\w,\q):=\summ\q_i\underline{R}_i(\w)\right\},
    \label{eq:empirical-mero-wq}
\end{equation}
where $\underline{R}_i(\w):=R_i(\w)-R_i^*$ is the excess empirical risk. 
In each iteration, the maximizing player calls the exponentiated gradient algorithm~\citep{kivinen1997exponentiated} to update $\q$ while the minimizing player calls an expensive ERM oracle to minimize $\underline{F}(\cdot,\q_t)$ for the current $\q_t$. Under the condition that the ERM oracle is perfect, their algorithm converges at a rate of $\OO\brac{\sqrt{\frac{\ln\brac{m\bar n}}{T}}}$. For simplicity, we assume the cost of the ERM oracle is proportional to the number of samples $m\bar n$. This leads to a computation complexity of $\OO\brac{\frac{m\bar n\ln(m\bar n)}{\varepsilon^2}}$ in total.

\subsection{Finite-Sum Convex-Concave Optimization}
The general finite-sum convex-concave optimization problem is given by
\begin{equation}
\min_{\x\in\mathcal{X}}\max_{\y\in\mathcal{Y}}\cbrac{G(\x,\y):=\sum_{i=1}^n G_i(\x,\y)},\label{eq:1-level-cc}
\end{equation}
where $G(\x,\y)$ is convex w.r.t.~$\x$ and concave w.r.t.~$\y$. Comparing with~\eqref{eq:empirical-gdro-wq}, we call~\eqref{eq:1-level-cc} ``\textit{one-level} finite-sum'' because the objective $G(\x,\y)$ is written as a single finite summation. Existing algorithms for this problem can also be applied to~\eqref{eq:empirical-gdro-wq}, but they suffer from a suboptimal complexity due to neglecting the special two-level finite-sum structure of $F(\w,\q)$. We review two representative works in the sequel.

\citet{alacaoglu2022stochastic} develop the Mirror Prox with Variance Reduction (MPVR) algorithm. To incorporate variance reduction, they modify the classic mirror prox by replacing a stochastic gradient at the last iterated point with a full gradient at the current snapshot. Moreover, they use ``negative momentum''~\citep{driggs2022accelerating} in the dual space to further accelerate the algorithm. MPVR can be applied to the empirical GDRO problem with its uniform sampling or importance sampling technique. However, both of them fail to capture the intrinsic two-level finite-sum structure in~\eqref{eq:empirical-gdro-wq} and therefore suffer an additional $\sqrt{m}$ factor in the total complexity of $\OO\brac{m\bar n+\frac{m\sqrt{m\bar n\ln{m}}}{\varepsilon}}$.

Based on MPVR,~\citet{luo2021near} use Accelerated Loopless Stochastic Variance-Reduced Extragradient (AL-SVRE) to tackle with the one-level finite-sum convex-concave problems. AL-SVRE uses MPVR as the subproblem solver and then conducts a catalyst acceleration scheme~\citep{lin2015catalyst}. Applying it to the empirical GDRO problem gives a complexity of $\OO\brac{\brac{m\bar n+ \frac{m\sqrt{m\bar n\ln m}}{\varepsilon}+m^{5/4}\bar n^{3/4}\sqrt{\frac{\ln m}{\varepsilon}}}\ln{\frac{m}{\varepsilon}}}$, which is also worse than our result by $\sqrt{m}$. 
% One major drawback of AL-SVRE is that it is a pure Euclidean algorithm which doesn't support general norms. Projecting with Euclidean norm for $\q\in\Delta_m$ is inappropriate, as it fails to adapt to the simplex domain.

\subsection{Complexity Comparisons}

We summarize the existing results for empirical GDRO in~\cref{tab:complexity-comparison}. Under the circumstance of moderately high accuracy $\varepsilon\le\OO\brac{\sqrt{\frac{\ln m}{\bar n}}}$, our $\OO\brac{\frac{m\sqrt{\bar n\ln{m}}}{\varepsilon}}$ complexity beats the baseline $\OO\brac{\frac{m\ln m}{\varepsilon^2}}$ of SMD. AL-SVRE has a complexity whose dominating term $\OO\brac{\frac{m\sqrt{m\bar n\ln{m}}}{\varepsilon}\ln{\frac{m}{\varepsilon}}}$ is worse than us by a factor of $\sqrt{m}\ln{\frac{m}{\varepsilon}}$. 
BROO-KX provides an oracle complexity with a leading term of $\OO\brac{\frac{(m\bar{n})^{3/4}}{\varepsilon}\ln^{7/2}(\frac{\ln m}{\varepsilon})}$, which is less favorable than our approach under the common situations where $m\le\OO(\bar{n})$.
MPVR exhibits a complexity of $\OO\brac{m\bar{n}+\frac{m\sqrt{m\bar{n}\ln{m}}}{\varepsilon}}$, which remains inferior to our approach by a factor of $\sqrt{m}$. 

% Besides equipped with a better complexity, we highlight that the assumptions of our algorithm are weaker because we only require each risk function $R_i(\cdot)$ to be convex instead of assuming that the loss function $\ell(\cdot;\xi_{ij})$ on every sample is convex as~\citet{luo2021near} and~\citet{carmon2022distributionally}.

% \begin{table*}[t]
% \renewcommand\arraystretch{2}
%     % \begin{center}
%     \centering
%      \caption{Comparisons of computation complexity in the empirical MERO problem.}
%     \begin{adjustbox}{max width=\textwidth}
%     \begin{threeparttable}  
%         % }
%         % \protect\footnotemark[2]}
%         \begin{tabular}{|c|c|}
%             \hline
%             \textbf{Algorithm} & \textbf{Computation Complexity} \\
%             \hline
%             ERMEG~\citep{agarwal2022minimax} & $\OO\brac{\frac{m\bar n\ln(m\bar n)}{\varepsilon^2}}$\\
% \hline
%             TSA~\citep{zhang2023efficient}  & $\OO\brac{\frac{m\ln m}{\varepsilon^2}}$\\
% \hline
%             \textbf{ALEM}~(\cref{thm:mero-convergence,cor:mero-complexity}) & $\tilde{\OO}\brac{\frac{m\sqrt{\bar n\ln{m}}}{\varepsilon}}$\\
%             \hline
%         \end{tabular}
% \label{tab:MERO-complexity-comparison}
%     \end{threeparttable}
%     \end{adjustbox}
%     % \end{center}
% \end{table*}

We also briefly discuss the empirical MERO algorithms.~\citet{agarwal2022minimax} use Empirical Risk Minimization with Exponential Gradient algorithm (ERMEG) to optimize~\eqref{eq:empirical-mero-wq}, thereby exhibiting a complexity of $\OO\brac{\frac{m\bar n\ln(m\bar n)}{\varepsilon^2}}$. 
Note that this result underestimates the optimization complexity of ERM oracles and thus is excessively conservative.~\citet{zhang2023efficient} develop a two-stage stochastic approximation (TSA) approach to target the population MERO. TSA can also be utilized in the empirical scenario, yielding a computation complexity of $\OO\brac{\frac{m\ln m}{\varepsilon^2}}$. We observe that both methods suffer from the quadratic dependency on $\varepsilon$. Based on our ALEG algorithm, we further develop a two-stage \textbf{Al}gorithm for \textbf{E}mpirical \textbf{M}ERO (ALEM), which attains an $\tilde{\OO}\brac{\frac{m\sqrt{\bar n\ln{m}}}{\varepsilon}}$ complexity, outperforming existing algorithms.

\section{Preliminaries\label{sec:preliminaries}}

In this section, we present notations, definitions, assumptions, and the Bregman setup used in the paper. 

\subsection{Notations} 

% Denote by $\norm{\cdot}_x$ a general norm on finite dimensional Banach space $\mathcal{E}_x$ and $\norm{\x}_{x,*}=\sup_{\y\in\mathcal{E}_x}\{\braket{\x,\y}|\norm{\y}_x\le 1\}$. We use $[S]=\{1,2,\cdots,S\}$ and $[S]^0=\{0,1,\cdots,S-1\}$ for some positive integer $S$. We denote $\W\times\Delta_m$ by $\Z$. In view of $\z=(\w;\q)\in\Z$ as the concatenation of $\w$ and $\q$, we use $F(\z)=F(\w,\q)$ and $\nabla F(\z)=(\nabla_\w F(\w,\q);-\nabla_\q F(\w,\q))$ to denote the merged function value and the merged gradient, respectively.

Denote by $\norm{\cdot}_x$ a general norm on finite dimensional Banach space $\mathcal{E}_x$ and its dual norm $\norm{\x}_{x,*}=\sup_{\y\in\mathcal{E}_x}\{\braket{\x,\y}|\norm{\y}_x\le 1\}$. We use $[S]=\{1,2,\cdots,S\}$ and $[S]^0=\{0,1,\cdots,S-1\}$ for some positive integer $S$. We denote $\W\times\Delta_m$ by $\Z$. In view of $\z=(\w;\q)\in\Z$ as the concatenation of $\w$ and $\q$, we use $F(\z)=F(\w,\q)$ and $\nabla F(\z)=\left(\nabla_\w F(\w,\q);-\nabla_\q F(\w,\q)\right)$ to denote the merged function value and the merged gradient, respectively.

\subsection{Definitions and Assumptions\label{sec:assumptions}}

For mirror descent~\citep{beck2003mirror} type of primal-dual methods, we need to construct the distance-generating function and the corresponding Bregman divergence. 

\begin{definition}
    We call a continuous function $\psi_x:X\mapsto\R$ a distance-generating function modulus $\alpha_x$ w.r.t.~$\norm{\cdot}_x$, if (i) the set $X^o=\{\x\in X|\partial\psi_x(\x)\neq 0\}$ is convex; (ii) $\psi_x$ is continuously differentiable and $\alpha_x$-strongly convex w.r.t.~$\norm{\cdot}_x$, i.e.,~$\braket{\nabla\psi_x(x_1)-\nabla\psi_x(x_2),x_1-x_2}\ge \alpha_x\norm{x_1-x_2}_x^2,\ \forall x_1,x_2\in X^o$.
\end{definition}

\begin{definition}
    Define Bregman function $B_x:X\times X^o\mapsto\R_+$ associated with distance-generating function $\psi_x$ as
    \begin{equation}
        B_x(x,x^o)=\psi_x(x)-\psi_x(x^o)-\braket{\nabla\psi_x(x^o),x-x^o}.
    \end{equation}
\end{definition}

We equip $\W$ with a distance-generating function $\psi_w(\cdot)$ modulus $\alpha_w$ w.r.t.~a norm $\norm{\cdot}_w$ endowed on $\mathcal{E}$. Similarly, we have $\psi_q(\cdot)$ modulus $\alpha_q$ w.r.t.~$\norm{\cdot}_q$. The choice of such $\psi_x$ and $\norm{\cdot}_x$ should rely on the geometric structure of our domain. In this paper, we stick to $\psi_q(\q)=\summ \q_i\ln{\q_i}$ as the entropy function, which is 1-strongly convex w.r.t~$\norm{\cdot}_1$.

The following assumptions will be used in our Bregman setup analysis, which are commonly adopted in the existing literature~\citep{nemirovski2009robust}.

\begin{assumption}
    (Boundness on the domain) The domain $\W$ is convex and its diameter measured by $\psi_w(\cdot)$ is bounded by a positive constant $D_w$, i.e.\begin{equation}
        \max_{\w\in\W}\psi_w(\w)-\min_{\w\in\W}\psi_w(\w)\le D_w^2.
    \end{equation}
    Similarly, we assume $\Delta_m$ is bounded by $D_q$. Since $\psi_q$ is specified as the entropy function, we have $D_q=\sqrt{\ln{m}}$.\label{asp:boundness}
\end{assumption}

\begin{assumption}
    (Smoothness and Lipschitz continuity) For any $i\in[m]$, $j\in[n_i]$, $\ell(\cdot;\xi_{ij})$ is $L$-smooth and $G$-Lipschitz continuous.\label{asp:smooth-Lipschitz}
\end{assumption}

\begin{remark}
     In the context of stochastic convex optimization, smoothness is of the essence to obtain a variance-based convergence rate~\citep{lan2012optimal}.
\end{remark}

\begin{assumption}
    (Convexity) For every $i\in[m]$, empirical risk function $R_i(\cdot)$ is convex.\label{asp:convexity}
\end{assumption}

\begin{remark}
    Our convexity assumption is \textit{weaker} than~\citet{luo2021near} and \citet{carmon2022distributionally}, due to not requiring each component loss function $\ell(\cdot;\xi_{ij})$ to be convex.\label{remark:convexity}
\end{remark}

\subsection{Bregman Setups\label{sec:performance-measure}}

We endow the Cartesian product space $\mathcal{E}\times\R^m$ and its dual space  $\mathcal{E}^*\times\R^m$ with the following norm and dual norm~\citep{nemirovski2009robust}. For any $\z=(\w;\q)\in\mathcal{E}\times\R^m$ and any $\z^*=(\w^*;\q^*)\in\mathcal{E}^*\times\R^m$,
\begin{equation}
\begin{aligned}
    \norm{\z}&:=\sqrt{\frac{\alpha_w}{2D_w^2}\norm{\w}_w^2+\frac{\alpha_q}{2D_q^2}\norm{\q}_q^2},\\
    \norm{\z^*}_*&:=\sqrt{\frac{2D_w^2}{\alpha_w}\norm{\w^*}_{w,*}^2+\frac{2D_q^2}{\alpha_q}\norm{\q^*}_{q,*}^2}.
\end{aligned}\label{eq:normz}
\end{equation}
% \begin{equation}
%     \norm{\z}:=\sqrt{\frac{\alpha_w}{2D_w^2}\norm{\w}_w^2+\frac{\alpha_q}{2D_q^2}\norm{\q}_q^2},\quad
%     \norm{\z^*}_*:=\sqrt{\frac{2D_w^2}{\alpha_w}\norm{\w^*}_{w,*}^2+\frac{2D_q^2}{\alpha_q}\norm{\q^*}_{q,*}^2}.\label{eq:normz}
% \end{equation}
The corresponding distance-generating function has the following form:
% \begin{equation}
%     \psi(\z):=\frac{1}{2D^2_w}\psi_w(\w)+\frac{1}{2D_q^2}\psi_q(\q),\quad\forall\ \z=(\w;\q)\in\Z=\W\times\Delta_m\label{eq:psi-z}.
% \end{equation}
\begin{equation}
    \psi(\z):=\frac{1}{2D^2_w}\psi_w(\w)+\frac{1}{2D_q^2}\psi_q(\q)\label{eq:psi-z}.
\end{equation}
It's easy to verify that $\psi(\z)$ is 1-strongly w.r.t.~the norm $\norm{\cdot}$ in~\eqref{eq:normz}. So now we can define the Bregman divergence $B:\Z\times\Z^o\mapsto\R_+$ used in our algorithm:
% \begin{equation}
%     B(\z,\z^o):=\psi(\z)-\psi(\z^o)-\braket{\nabla \psi(\z^o),\z-\z^o},\quad\forall\ \z\in\Z,\forall\ \z^o\in\Z^o\label{eq:def-z-bregman-divergence}.
% \end{equation} 
\begin{equation}
    B(\z,\z^o):=\psi(\z)-\psi(\z^o)-\braket{\nabla \psi(\z^o),\z-\z^o}\label{eq:def-z-bregman-divergence}.
\end{equation} 
To analyze the quality of an approximate solution, we adopt a commonly used performance measure in existing literature~\citep{luo2021near,carmon2022distributionally,zhang2023stochastic}, known as the duality gap of any given $\bar\z=(\bar\w;\bar\q)$ for~\eqref{eq:empirical-gdro-wq}:
\begin{equation}
\epsilon(\bar\z):=\max_{\q\in\Delta_m}F(\bar\w,\q)-\min_{\w\in\W}F(\w,\bar\q).\label{eq:duality-gap}
\end{equation}
We aim to find a solution $\bar\z=(\bar\w;\bar\q)$ that is $\varepsilon$-accuracy of~\eqref{eq:empirical-gdro-wq}, i.e., $\epsilon(\bar\z)\le\varepsilon$. It's obvious that $\epsilon(\bar\z)$ is an upper bound for the optimality of $\w$ to~\eqref{eq:empirical-gdro-wq}, since
% \begin{equation}
%     \begin{aligned}
%         \max_{i\in[m]}R_i(\bar\w)-\min_{\w\in\W}\max_{i\in[m]}R_i(\w)=&\max_{\q\in\Delta_m}\summ\q_iR_i(\bar\w)-\min_{\w\in\W}\max_{\q\in\Delta_m}\summ\q_iR_i(\w)\\\le& \max_{\q\in\Delta_m}\summ\q_iR_i(\bar\w)-\min_{\w\in\W}\summ\bar\q_iR_i(\w)=\epsilon(\bar\z).
%     \end{aligned}
% \end{equation}
\begin{equation}
    \begin{aligned}
        &\max_{i\in[m]}R_i(\bar\w)-\min_{\w\in\W}\max_{i\in[m]}R_i(\w)\\=&\max_{\q\in\Delta_m}\summ\q_iR_i(\bar\w)-\min_{\w\in\W}\max_{\q\in\Delta_m}\summ\q_iR_i(\w)\\\le& \max_{\q\in\Delta_m}\summ\q_iR_i(\bar\w)-\min_{\w\in\W}\summ\bar\q_iR_i(\w)=\epsilon(\bar\z).
    \end{aligned}
\end{equation}

\section{Algorithm for Empirical GDRO\label{sec:gdro-solution}}

Inspired by~\citet{alacaoglu2022stochastic}, we follow the common double-loop structure of variance reduction. The outer loop computes snapshot points in the primal space and the dual space, respectively. The inner loop runs a modified mirror prox scheme with a full gradient plus a stochastic gradient, rather than a pair of stochastic gradients in the classic mirror prox algorithm~\citep{nemirovski2004prox,juditsky2011solving}. We emphasize the following two techniques that separate our algorithm from other similar work~\citep{luo2021near,carmon2022distributionally,alacaoglu2022stochastic}.

\paragraph{Variance Reduction Based on Group Sampling}
To improve complexity bounds, we propose the group sampling technique which samples uniformly from all groups per iteration and further reduces the variance of the stochastic gradient for each group. This strategy captures the two-level finite-sum structure by eliminating the randomness on the first level of the finite-sum structure, i.e., the summation $\summ\q_iR_i(\w)$. Although the group sampling of ALEG queries $m$ times more stochastic gradients than uniform sampling or importance sampling of MPVR, it produces a better stochastic gradient with a $1/m$ lower Lipschitz constant (cf.~\cref{lem:Lipschitz}), which ultimately improves the complexity by a factor of $\sqrt{m}$. 
The motivation behind this technique is twofold: (i) identifying the nested finite-sum structure of the objective $F$, and (ii) leveraging the inherent properties of the stochastic gradient under $\norm{\cdot}_\infty$. Comprehensive discussions can be found in~\cref{app:sec:compare-sampling}.

\paragraph{Alterable Hyperparameters}
To support variable algorithmic hyperparameters, we use the one-index-shifted weighted average rather than the naive ergodic average to construct the (mirror) snapshot points and the Lyapunov functions. In this way, we provide theoretical support for non-constant learning rates~\citep{liner2021improving}, which have been proven competitive in many practical scenarios.
% Our alterable hyperparameters offer theoretical support and rationale for non-constant learning rates\footnote{Our formulation of learning rates naturally supports variable settings~\citep{liner2021improving}, which have been proven to be competitive in many practical scenarios.}, which contributes to the practicality of ALEG.

We introduce the proposed ALEG in~\cref{alg:gdro}. The detailed algorithm is elaborated in~\cref{sec:algorithm-details} and the corresponding theoretical guarantee is presented in~\cref{sec:theoretical-guarantee}.

\subsection{Our Algorithm\label{sec:algorithm-details}}
% \subsection{Outer Loop}
In the outer loop of ALEG, we follow the standard variance reduction procedure~\citep{NIPS:2013:Zhang,10.5555/2999611.2999647} by periodically computing the snapshot points as well as the full gradient. In the inner loop of ALEG, we maintain two sets of solutions as mirror prox algorithm~\citep{nemirovski2004prox,juditsky2011solving} and further accelerate it via snapshot points.

At the beginning of the $s$-th outer loop, we compute the snapshot $\z^{s}$ using the weighted average from the previous trajectory. The mirror snapshot $\nabla\psi(\bar\z^{s})$ is constructed similarly by the weighted average of the current trajectory mapped in the dual space\footnote{Actually, the value of $\nabla\psi(\bar\z^{s})$ is sufficient for~\cref{alg:gdro}. The inverse or conjugate of $\nabla\psi$ is not necessary to calculate and thus no additional cost is incurred.}. Formally, we compute them via the one-index-shifted weighted average as follows
% \begin{equation}
%     \z^{s}=\brac{\sumks \alpha_{k-1}^{s-1}}^{-1}\sum_{k=1}^{K_{s-1}}\alpha_{k-1}^{s-1}\z_k^{s-1},\label{eq:snapshot}
% \end{equation}
% \begin{equation}
%     \nabla \psi(\bar\z^{s})=\brac{\sumks \alpha_{k-1}^{s-1}}^{-1}\sum_{k=1}^{K_{s-1}}\alpha_{k-1}^{s-1}\nabla\psi(\z^{s-1}_k).\label{eq:mirror-snapshot}
% \end{equation}
\begin{align}
    \z^{s}&=\brac{\sumks \alpha_{k-1}^{s-1}}^{-1}\sum_{k=1}^{K_{s-1}}\alpha_{k-1}^{s-1}\z_k^{s-1}\label{eq:snapshot},\\
    \nabla \psi(\bar\z^{s})&=\brac{\sumks \alpha_{k-1}^{s-1}}^{-1}\sum_{k=1}^{K_{s-1}}\alpha_{k-1}^{s-1}\nabla\psi(\z^{s-1}_k).\label{eq:mirror-snapshot}
\end{align}
% \begin{equation}
%     \z^{s}=\brac{\sumks \alpha_{k-1}^{s-1}}^{-1}\sum_{k=1}^{K_{s-1}}\alpha_{k-1}^{s-1}\z_k^{s-1},\quad\nabla \psi(\bar\z^{s})=\brac{\sumks \alpha_{k-1}^{s-1}}^{-1}\sum_{k=1}^{K_{s-1}}\alpha_{k-1}^{s-1}\nabla\psi(\z^{s-1}_k).\label{eq:snapshots}
% \end{equation}
Then, the full gradient $\nabla F(\z^s)$ is computed as
\begin{equation}
    \nabla F(\z^s)=\begin{pmatrix}
        \summ \q^s_i\nabla R_i(\w^s)\\-\left[R_1(\w^s),\cdots,R_m(\w^s)\right]^T
    \end{pmatrix}.\label{eq:z-fullgradient}
\end{equation}
In the inner loop $k$, we first compute the prox point using the full gradient from the last epoch:
% \begin{equation}
%     \z_{k+1/2}^s=\argmin_{\z\in\Z}\{\eta_k^s\braket{\nabla F(\z^s),\z}+\alpha_k^s B(\z,\bar\z^s)+(1-\alpha_k^s)B(\z,\z_k^s)\}.
% \label{eq:z_k+1/2^s}
% \end{equation}
\begin{equation}
\begin{aligned}
    \z_{k+1/2}^s=\argmin_{\z\in\Z}\{\eta_k^s\braket{\nabla F(\z^s),\z}+\alpha_k^s &B(\z,\bar\z^s)\\+(1-\alpha_k^s)&B(\z,\z_k^s)\}.
\end{aligned}\label{eq:z_k+1/2^s}
\end{equation}
The above update is different from the traditional mirror prox algorithm~\citep{nemirovski2004prox,juditsky2011solving} because it uses the full gradient $\nabla F(\z^s)$ instead of the stochastic gradient at~$\z_{k-1}^s$. 

\begin{remark}
    The update in~\eqref{eq:z_k+1/2^s} utilizes the ``negative momentum''~\citep{driggs2022accelerating}  technique to achieve acceleration with the help of snapshot points. Similar ideas are also used in recent studies on variance reduction~\citep{shang2017fast,zhou2018simple}. However, their momentum is performed in the primal space, which is different from MPVR and our method.
\end{remark}

% The above mechanism is intricately designed to better integrate with the variance reduction scheme. 

After $\z_{k+1/2}^s$ is calculated, we leverage the two-level structure to construct the stochastic gradient based on our group sampling technique. For each group $i$, we sample uniformly from the $i$-th group's samples $\{\xi_{ij}\}_{j=1}^{n_i}$. The $m$ samples generated by the group sampling technique are denoted as:
\begin{equation}
    \xi_{k}^s:=\{\xi_{k,i}^s\}_{i=1}^m,\quad \xi_{k,i}^s\sim\text{Unif}\brac{\{\xi_{ij}\}_{j=1}^{n_i}},\forall\ i\in[m].\label{eq:per-group-sampling}
\end{equation}
\begin{remark}
    Compared to uniform sampling or importance sampling in~\citet{alacaoglu2022stochastic}, we incorporate the finite-sum components of the objective in~\eqref{eq:empirical-gdro-wq} so that our stochastic gradients make use of a total of $m$ samples.
\end{remark}

Then, we construct our stochastic gradient based on the group sampling technique, as specified below:
\begin{equation}
    \nabla F(\z^s;\xi_k^s):=\begin{pmatrix}
        \summ \q^s_i\nabla\ell(\w^s;\xi_{k,i}^s)\\-\left[\ell(\w^s;\xi_{k,1}^s),\cdots,\ell(\w^s;\xi_{k,m}^s)\right]^T
    \end{pmatrix}.\label{eq:stochastic-gradient}
\end{equation}
The above construction equally weighs the randomness from every group, which is the key step to exploit the two-level finite-sum structure of~\eqref{eq:empirical-gdro-wq}. Next, we introduce the following lemma to quantify its Lipschitz continuity.

\begin{algorithm}[t]
   \caption{Variance-Reduced Stochastic Mirror Prox \textbf{Al}gorithm for \textbf{E}mpirical \textbf{G}DRO (ALEG)}
   \label{alg:gdro}
  {\bf Input}: 
  Risk functions $\{R_i(\w)\}_{i\in[m]}$, epoch number $S$, iteration numbers $\{K_s\}$, learning rates $\{\eta_{k}^{s}\}$, and weights $\{\alpha_{k}^{s}\}$.
\begin{algorithmic}[1]
   % \STATE Initialize~starting~point~$\z_0=(\w_0;\q_0)=\argmin_{\z\in\Z}\psi(\z)$.
   \STATE Initialize $\z_0=(\w_0;\q_0)=\argmin_{\z\in\Z}\psi(\z)$ as the starting point.
   \STATE For each $j\in[K_{-1}]$, set $\z_j^{-1}=\z_0^0=\z_0$.
   \FOR{$s=0$ {\bfseries to} $S-1$}
   % \STATE Compute snapshot $\z^{s}$ and mirror snapshot $\nabla \psi(\bar\z^{s})$ according to~\eqref{eq:snapshots}.
   \STATE Compute the snapshot $\z^{s}$ and the mirror snapshot $\nabla \psi(\bar\z^{s})$ according to~\eqref{eq:snapshot} and~\eqref{eq:mirror-snapshot}, respectively.
   \STATE Compute the full gradient $\nabla F(\z^{s})$ according to~\eqref{eq:z-fullgradient}.
   \FOR{$k=0$ {\bfseries to} $K_s-1$}
   \STATE Compute $\z_{k+1/2}^s$ according to~\eqref{eq:z_k+1/2^s}.
   \STATE For each $i\in[m]$, sample $\xi_{k,i}^s$ uniformly from $\{\xi_{ij}\}_{j=1}^{n_i}$. 
   \STATE Compute the variance-reduced stochastic gradient estimator $\g_k^s$ defined in~\eqref{eq:g_k^s}.
   \STATE Compute $\z_{k+1}^s$ according to~\eqref{eq:z_k+1^s}.
    \ENDFOR
    \STATE Set $\z_0^{s+1}=\z_{K_s}^s$.
    \ENDFOR
    \STATE \textbf{Return} $\z_S$ according to~\eqref{eq:output-point}.
\end{algorithmic}
\end{algorithm}

\begin{lemma}For any $s\in[S]^0, k\in[K_s]^0$, $\nabla F(\z;\xi_{k}^s)$ is $L_z$-Lipschitz continuous, where
\begin{equation*}
L_z:=2D_w\max\cbrac{\sqrt{2D^2_wL^2+G^2\ln{m}},G\sqrt{2\ln{m}}}.
\end{equation*}
\label{lem:Lipschitz}
\end{lemma}
Comparing with MPVR~\citep[Definition~7]{alacaoglu2022stochastic},~\cref{lem:Lipschitz} shows that our group sampling technique can reduce the Lipschitz constant of the stochastic gradient by a factor of $m$, which is the reason for the $\sqrt{m}$ improvement of the complexity.

Inspired by variance reduction~\citep{NIPS:2013:Zhang,10.5555/2999611.2999647} techniques, we construct the following variance-reduced stochastic gradient estimator using the stochastic gradient at~the snapshot $\z^s$ and the full gradient in~\eqref{eq:z-fullgradient}:
\begin{equation}
    \g_k^s=\nabla F(\z_{k+1/2}^s;\xi_k^s)-\nabla F(\z^s;\xi_k^s)+\nabla F(\z^s)\label{eq:g_k^s}.
\end{equation}
With expectation taken over $\xi_k^s$, it is easy to verify that $\g_k^s$ is an unbiased estimator. The final step is to compute $\z_{k+1}^s$ according to the mirror prox scheme:
% \begin{equation}
%      \z_{k+1}^s=\argmin_{\z\in\Z}\{\eta_k^s\braket{ \g_k^s,\z}+\alpha_k^s B(\z,\bar\z^s)+(1-\alpha_k^s)B(\z,\z_k^s)\}.\label{eq:z_k+1^s}
% \end{equation}
\begin{equation}
\begin{aligned}
    \z_{k+1}^s=\argmin_{\z\in\Z}\{\eta_k^s\braket{ \g_k^s,\z}+\alpha_k^s &B(\z,\bar\z^s)\\+(1-\alpha_k^s)&B(\z,\z_k^s)\}.
\end{aligned}\label{eq:z_k+1^s}
\end{equation}
In~\eqref{eq:z_k+1^s}, we use the variance-reduced stochastic gradient estimator $\g_k^s$ instead of raw stochastic gradient at $\z_k^s$ to achieve variance reduction. Upon completion of the double loop procedure, ALEG returns solutions in a different manner compared to MPVR, as shown below
\begin{equation}
     \z_S=\brac{\sums\sumkzs\eta_k^s}^{-1}\sums \sumkzs \eta_k^s\z_{k+1/2}^s\label{eq:output-point}.
\end{equation}
Specifically, our utilization of alterable learning rates necessitates the computation of $\z_S$ through weighted averaging, as opposed to the ergodic averaging approach employed by~\citet{alacaoglu2022stochastic}.
 
\subsection{Theoretical Guarantee\label{sec:theoretical-guarantee}}
Now we present our theoretical result for~\cref{alg:gdro}. The proofs are deferred to~\cref{app:sec:gdro-analysis}. 
\begin{theorem}
    Under~\cref{asp:boundness,asp:smooth-Lipschitz,asp:convexity}, by setting $K_s=K, \alpha_k^s=\frac{1}{K}, \frac{1}{10L_z\sqrt{K}}\le\eta_k^s\le\frac{1}{L_z\sqrt{5K}}$, 
    our~\cref{alg:gdro} ensures that
    \begin{equation}
        \E\sqbrac{\epsilon(\z_S)}\le\OO\brac{\frac{1}{S}\sqrt{\frac{\ln m}{K}}}.\label{eq:expectation-bound}
    \end{equation}
    % and with probability at least $1-\delta$, 
    % \begin{equation}
    %     \epsilon(\z_S)\le\OO\brac{\frac{\sqrt{\ln m}}{S\sqrt{K}}\sqrt{\ln \frac{S}{\delta}}+\frac{\sqrt{\ln m}}{SK}\ln{\frac{S}{\delta}}}.
    %     \label{eq:probability-bound}
    % \end{equation}
    \label{thm:gdro-convergence}
\end{theorem}
\begin{remark}
    We stick to the constant parameter setting in terms of $\cbrac{K_s}$ and $\cbrac{\alpha_k^s}$, while allowing the learning rates $\cbrac{\eta_k^s}$ to remain adjustable. 
    Variable $\cbrac{K_s}$ and $\cbrac{\alpha_k^s}$ complicate analysis and presentation with an additional summation term in overall complexity.
    For brevity, we set $\cbrac{K_s}$ and $\cbrac{\alpha_k^s}$ to be constant.
\end{remark}
\begin{corollary}
    Under conditions in~\cref{thm:gdro-convergence}, by setting $K=\Theta(\bar n)$, the computation complexity for~\cref{alg:gdro} to reach $\varepsilon$-accuracy of~\eqref{eq:empirical-gdro-wq} is $\OO\brac{\frac{m\sqrt{\bar n\ln{m}}}{\varepsilon}}$.\label{cor:gdro-complexity}
\end{corollary}
Our complexity in~\cref{cor:gdro-complexity} is better than the state-of-the-art rate~\citep{alacaoglu2022stochastic} by a factor of $\sqrt{m}$. From a practical perspective, ALEG still maintains a low per-iteration complexity. The main step~\eqref{eq:z_k+1/2^s} and~\eqref{eq:z_k+1^s} only involves projections onto $\W$ and $\Delta_m$ respectively. With $\psi_w(\w)=\frac{1}{2}\norm{\w}_2^2$ and $\psi_q(\q)=\summ\q_i\ln{\q_i}$, the updates are equivalent to Stochastic Gradient Descent (SGD) and Hedge~\citep{freund1997decision}.

% \begin{figure*}[tbp]
% 	\centering
%  \captionsetup{justification=centering}
%  % \captionsetup[subfigure]{singlelinecheck=off,justification=centering}
%  \subcaptionsetup[figure]{justification=centering}
%  % \captionsetup[sub]{justification=centering}
% 	\subfigure[\centering Training set of the synthetic data\centering]{\includegraphics[width=.245\textwidth]{fig/gdro/synthetic_train.pdf}\label{fig:synthetic_train}}
%  % \hfill
% 	\subfigure[\centering Test set of the synthetic data]{\includegraphics[width=.245\textwidth]{fig/gdro/synthetic_test.pdf}\label{fig:synthetic_test}}
% 	\subfigure[\centering Training set of CIFAR-100]{\includegraphics[width=.245\textwidth]{fig/gdro/cifar_train.pdf}\label{fig:cifar_train}}
%  % \hfill
% 	\subfigure[\centering Test set of CIFAR-100]{\includegraphics[width=.245\textwidth]{fig/gdro/cifar_test.pdf}\label{fig:cifar_test}}
% 	\caption{Comparison of the max empirical risk $\max_{i\in[m]}R_i(\cdot)$ with respect to the number of stochastic gradient evaluations $\# \text{ of } \nabla\ell(\cdot;\xi_{ij})$ for synthetic dataset and CIFAR-100 dataset.\label{fig:synthetic-plus-cifar}}
% \end{figure*}

\section{Algorithm for Empirical MERO\label{sec:mero-solution}}

\begin{figure*}[tbp]
	\centering
 \captionsetup{justification=centering}
 % \captionsetup[subfigure]{singlelinecheck=off,justification=centering}
 % \subcaptionsetup[figure]{justification=centering}
 % \captionsetup[sub]{justification=centering}
	\subfigure[\centering Training on the synthetic data\centering]{\includegraphics[width=.245\textwidth]{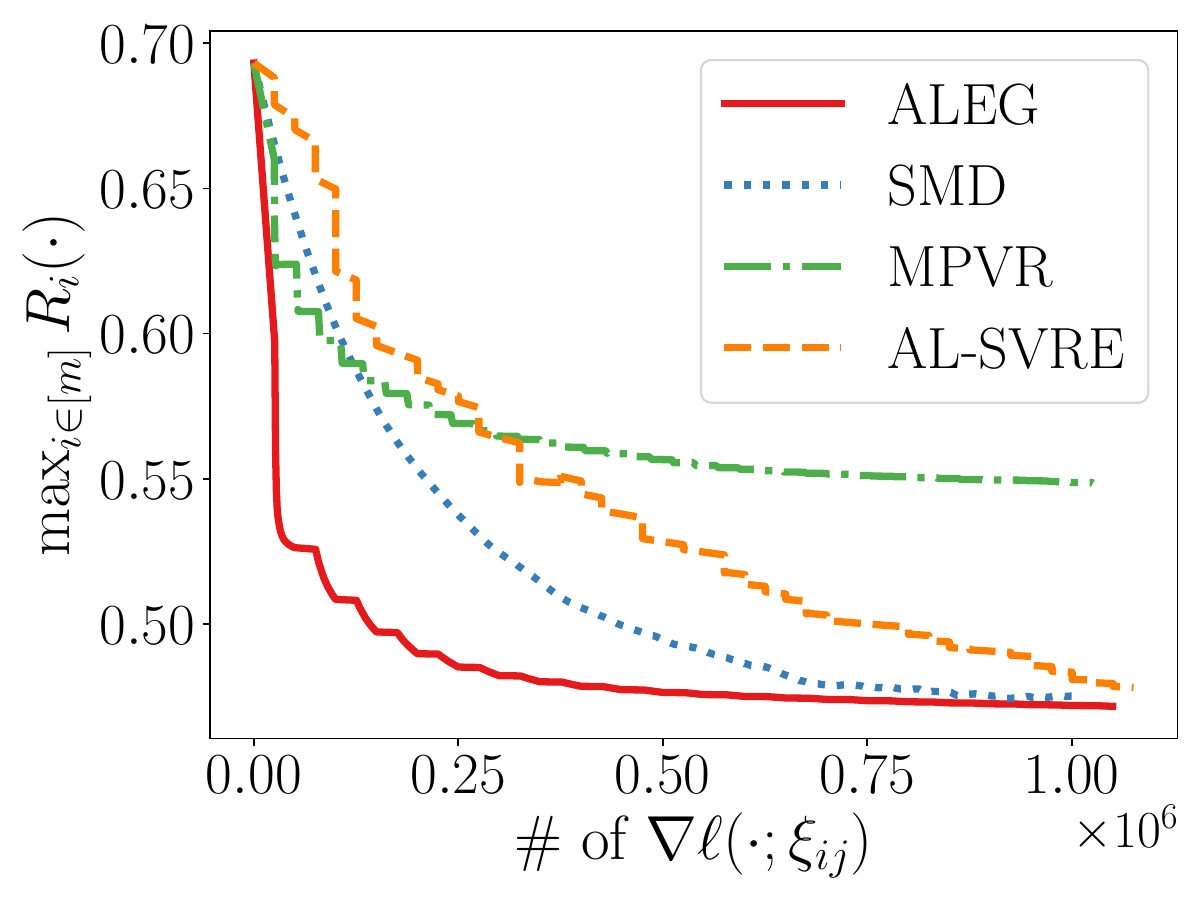}\label{fig:synthetic_train}}
 % \hfill
	\subfigure[\centering Test on the synthetic data]{\includegraphics[width=.245\textwidth]{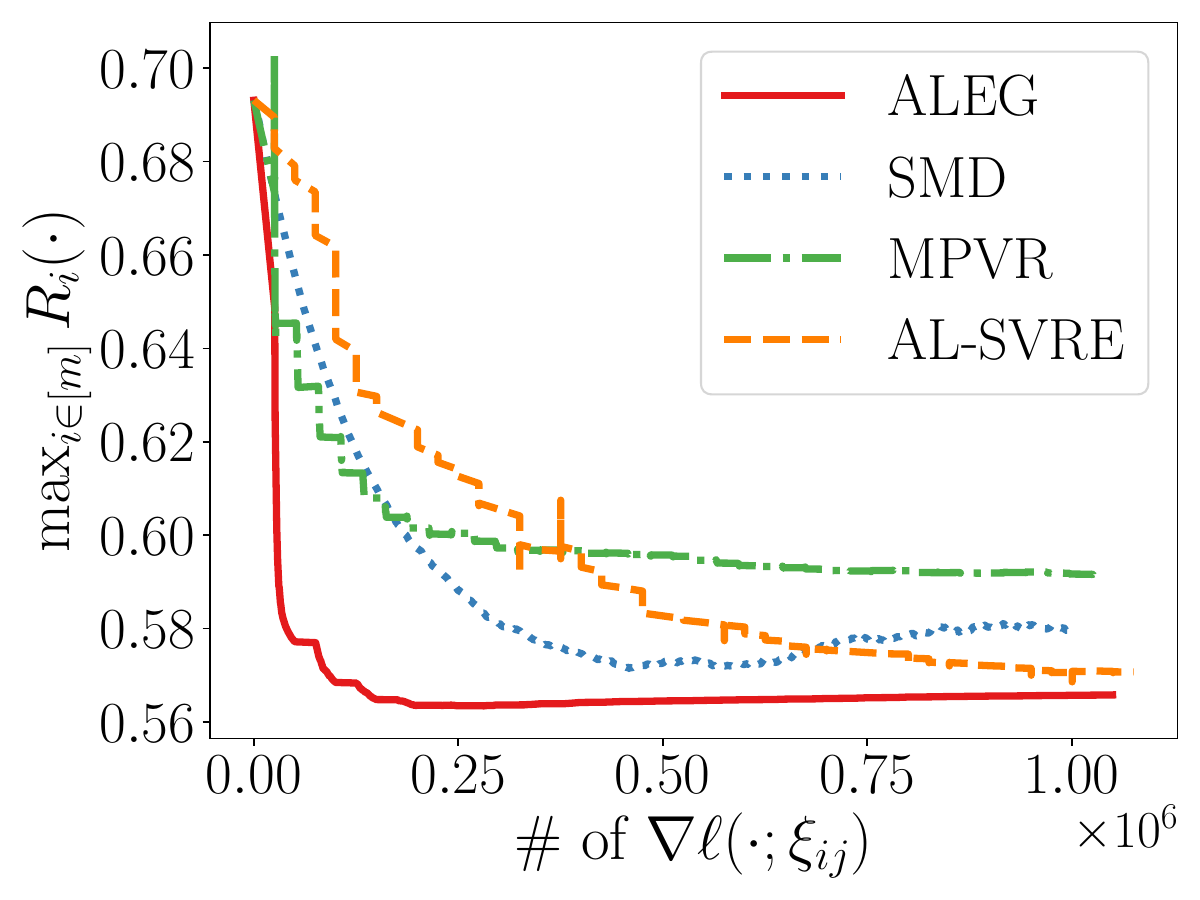}\label{fig:synthetic_test}}
	\subfigure[\centering Training on CIFAR-100]{\includegraphics[width=.245\textwidth]{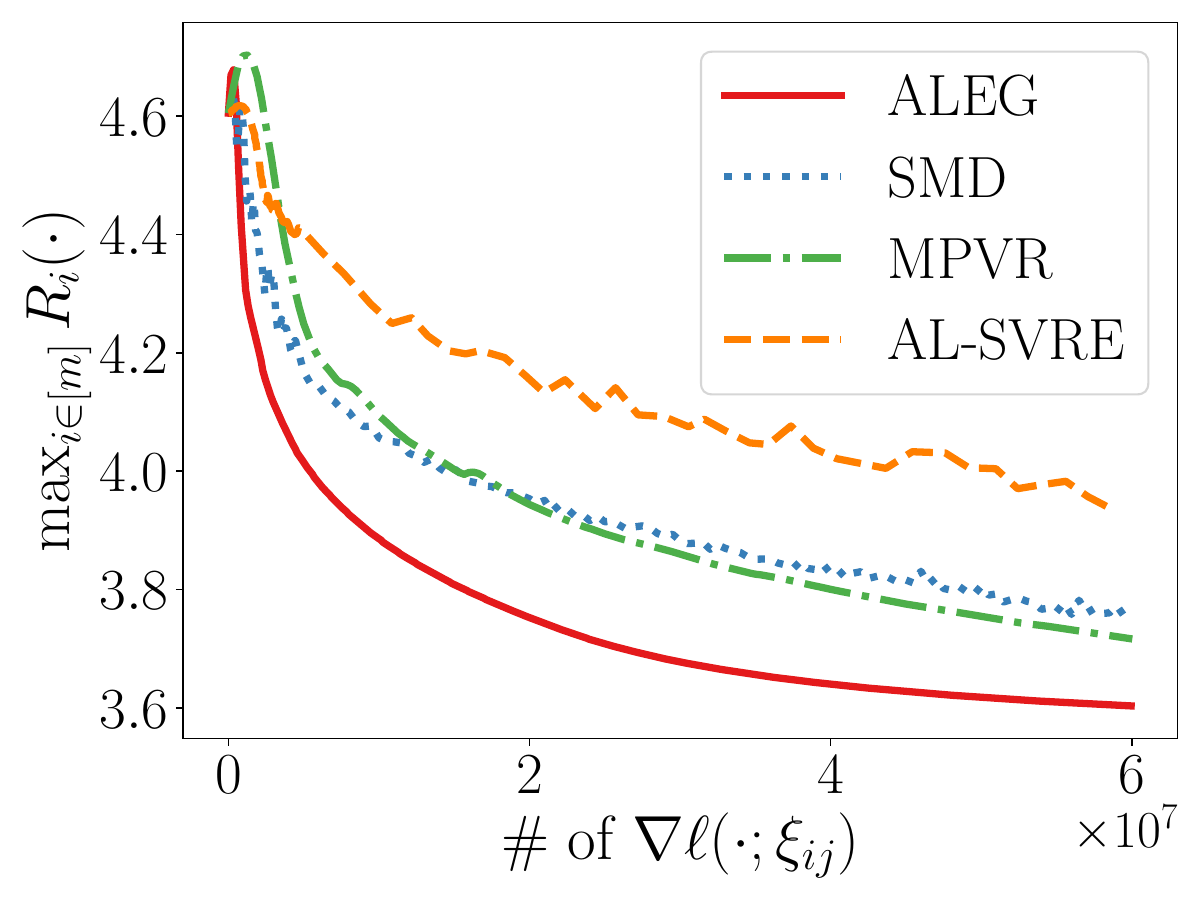}\label{fig:cifar_train}}
 % \hfill
	\subfigure[\centering Test on CIFAR-100]{\includegraphics[width=.245\textwidth]{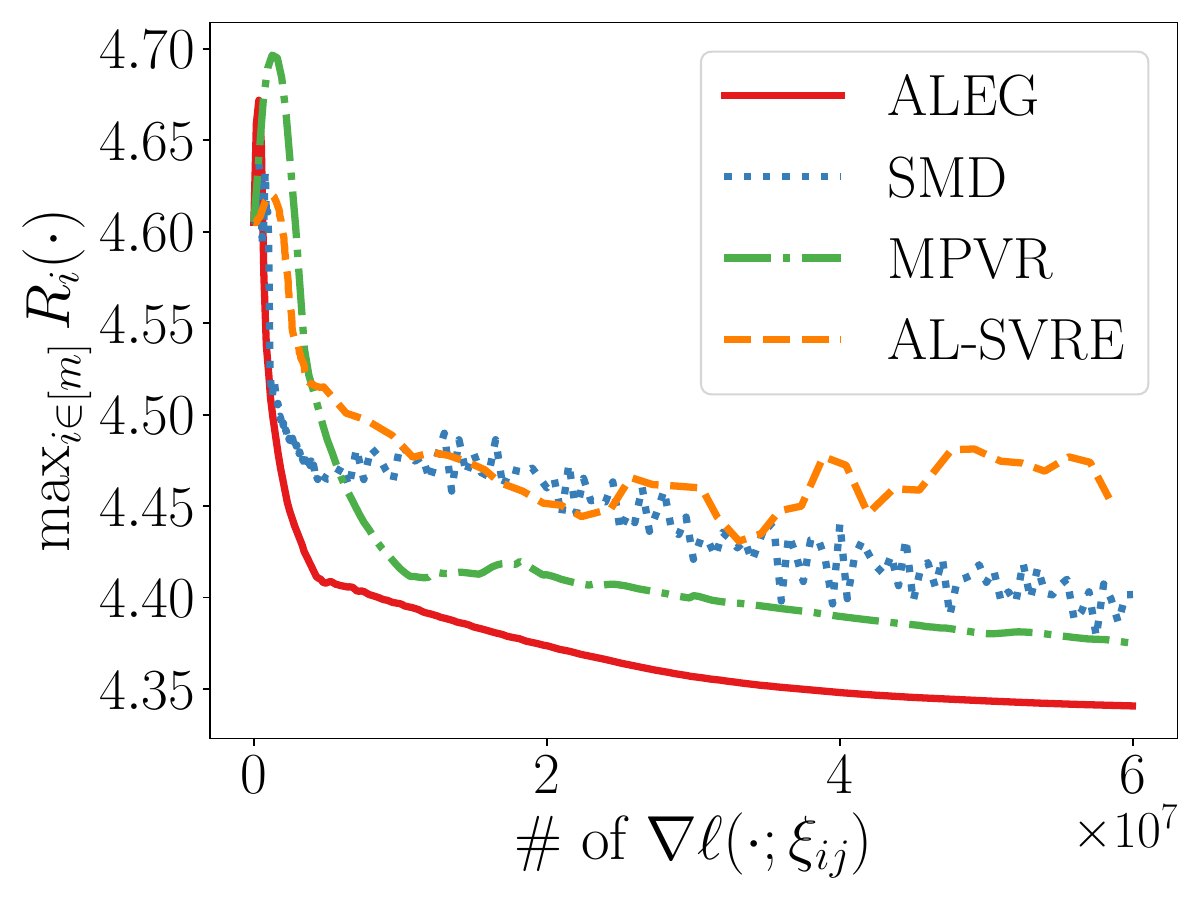}\label{fig:cifar_test}}
	\caption{Comparison of the max empirical risk $\max_{i\in[m]}R_i(\cdot)$ with respect to the number of stochastic gradient evaluations $\# \text{ of } \nabla\ell(\cdot;\xi_{ij})$ on the synthetic dataset and the CIFAR-100 dataset.\label{fig:synthetic-plus-cifar}}
\end{figure*}

In this section, we extend our methodology in~\cref{sec:gdro-solution} to solve the empirical MERO in~\eqref{eq:empirical-mero-wq}. Inspired by the efficient stochastic approximation approach to MERO~\citep{zhang2023efficient}, we propose a two-stage \textbf{Al}gorithm for \textbf{E}mpirical \textbf{M}ERO (ALEM) as shown in~\cref{alg:mero}. In the first stage, ALEM runs $m$ ALEG instances to estimate the minimal empirical risk as $\{\hat{R}_i^*\}_{i\in[m]}$. In the second stage, we focus on the following approximate problem:
\begin{equation}
    % \min_{\w\in\W}\max_{\q\in\Delta_m}\left\{\underline {\hat F}(\w,\q):=\summ\q_i\underline{\hat R}_i(\w)=\summ\q_i\sqbrac{R_i(\w)-\hat R_i^*}\right\},
\min_{\w\in\W}\max_{\q\in\Delta_m}\left\{\underline {\hat F}(\w,\q):=\summ\q_i\underline{\hat R}_i(\w)\right\},
    \label{eq:empirical-mero-wq-surrogate}
\end{equation}
which can be regarded as a substitution of the excess empirical risk $\underline{R}_i(\w)$ in~\eqref{eq:empirical-mero-wq} with the approximated excess empirical risk $\underline{\hat R}_i(\w):=R_i(\w)-\hat{R}_i^*$. Then ALEM treats~\eqref{eq:empirical-mero-wq-surrogate} as an empirical GDRO problem and calls ALEG again to optimize it.

Similar to~\citet[Lemma~4.2]{zhang2023efficient}, we present the following lemma to show that the optimization error for~\eqref{eq:empirical-mero-wq} is under control, provided that the optimization error for~\eqref{eq:empirical-mero-wq-surrogate} is small and the approximation error $\hat{R}_i^*-R_i^*$ is close to zero for all groups. All the proofs are deferred to~\cref{app:sec:mero-analysis}.
\begin{lemma}
For any $\bar\z=(\bar\w;\bar\q)\in\Z$, define the duality gap for the empirical MERO problem~\eqref{eq:empirical-mero-wq} and the approximated empirical MERO problem~\eqref{eq:empirical-mero-wq-surrogate} as
% \begin{equation}
% \underline\epsilon(\bar\z):=\max_{\q\in\Delta_m}\underline{F}(\bar\w,\q)-\min_{\w\in\W}\underline{F}(\w,\bar\q),\quad\underline{\hat\epsilon}(\bar\z):=\max_{\q\in\Delta_m}\underline{\hat{F}}(\bar\w,\q)-\min_{\w\in\W}\underline{\hat{F}}(\w,\bar\q),
% \end{equation}
\begin{equation}
\begin{aligned}
\underline\epsilon(\bar\z):=\max_{\q\in\Delta_m}\underline{F}(\bar\w,\q)-\min_{\w\in\W}\underline{F}(\w,\bar\q),\\\underline{\hat\epsilon}(\bar\z):=\max_{\q\in\Delta_m}\underline{\hat{F}}(\bar\w,\q)-\min_{\w\in\W}\underline{\hat{F}}(\w,\bar\q),
\end{aligned}
\end{equation}
respectively. It holds that
\begin{equation}
        \underline\epsilon(\bar\z)\le \underline{\hat\epsilon}(\bar\z)+2\max_{i\in[m]}\{\hat{R}_i^*-R_i^*\}.\label{eq:duality-gap-bound}
\end{equation}\label{lem:opt-error}
\end{lemma}
\paragraph{Stage 1: Empirical Risk Minimization} Noticing that when the number of groups reduces to 1, the minimax problem in~\eqref{eq:empirical-gdro} degenerates to the classical one-level finite-sum empirical risk minimization problem, which ALEG can still handle. Therefore, we can apply ALEG to each group to get an estimator $\hat{R}_i^*$ for $R_i^*$. The extra input $T$ can be viewed as a budget to control the cost of each group. To deal with the max operator in~\eqref{eq:duality-gap-bound}, we provide the following high-probability bound.
\begin{algorithm}[t]
   \caption{Two-Stage \textbf{Al}gorithm for \textbf{E}mpirical \textbf{M}ERO (ALEM)}
   \label{alg:mero}
  {\bf Input}: 
  Risk functions $\{R_i(\w)\}_{i\in[m]}$, epoch number $S$, iteration numbers $\{K_s\}$, learning rates $\{\eta_{k}^{s}\}$, weights $\{\alpha_{k}^{s}\}$, and budget $T$.
\begin{algorithmic}[1]
    \FOR{$i=1$ {\bfseries to} $m$}
    \STATE Empirical risk minimization: \\ \quad\protect{$\bar{\w}_i\leftarrow$  ALEG$\brac{R_i(\w),S,\{K_s\},\{\eta_{k}^{s}\},\{\alpha_{k}^{s}\}}$}.
    \STATE Estimate the minimal empirical risk: $\hat R_i^*=R_i(\bar{\w}_i)$.
    \ENDFOR
   \STATE Run our empirical GDRO solver to optimize~\eqref{eq:empirical-mero-wq-surrogate}: 
   % \begin{equation}
   %     \bar\z\leftarrow\text{ALEG}\brac{\{R_i(\w)-\hat R_i^*\}_{i\in[m]},S,\{K_s\},\{\eta_{k}^{s}\},\{\alpha_{k}^{s}\}}.
   % \end{equation}
   $\bar\z\leftarrow$ ALEG$\left(\{R_i(\w)-\hat R_i^*\}_{i\in[m]},S,\{K_s\},\{\eta_{k}^{s}\},\{\alpha_{k}^{s}\}\right)$.
    \STATE \textbf{Return} $\bar\z,\{\bar{\w}_i\}_{i\in[m]},\{\hat{R}_i^*\}_{i\in[m]}$.
\end{algorithmic}
\end{algorithm}
\begin{theorem}
    Under~\cref{asp:boundness,asp:smooth-Lipschitz,asp:convexity}, running~\cref{alg:gdro} as an ERM oracle by setting $S=\left\lceil\frac{T}{\sqrt{\bar n}}\right\rceil, K_s=\bar n,\alpha_k^s=\frac{1}{\bar n},\frac{1}{10L\sqrt{\bar n}}\le\eta_k^s\le\frac{1}{L\sqrt{5\bar n}}$, for any group $i$, the following holds with probability at least $1-\delta$,
    \begin{equation}
        \hat{R}_i^*-R_i^*\le\OO\brac{\frac{1}{T}\brac{\sqrt{\ln \frac{mT}{\sqrt{\bar n}\delta}}+\frac{1}{\sqrt{\bar n}}\ln{\frac{mT}{\sqrt{\bar n}\delta}}}}.
        \label{eq:excess-risk-probability-bound}
    \end{equation}
    % , we have
    % \begin{equation}
    %     \E\sqbrac{\hat{R}_i^*-R_i^*}\le \OO\brac{\frac{1}{T}},\quad\forall i\in[m],\label{eq:excess-risk-expectation-bound}
    % \end{equation}
    % and with probability at least $1-\delta$,
    % \begin{equation}
    %     \hat{R}_i^*-R_i^*\le\OO\brac{\frac{1}{T}\sqbrac{\sqrt{\ln \frac{mT}{\sqrt{\bar n}\delta}}+\frac{1}{\sqrt{\bar n}}\ln{\frac{mT}{\sqrt{\bar n}\delta}}}}, \quad\forall i\in[m].
    %     \label{eq:excess-risk-probability-bound}
    % \end{equation}
    \label{thm:excess-risk-convergence}
\end{theorem}
\paragraph{Stage 2: Solving Empirical GDRO} After we managed to minimize the empirical risks for $m$ groups, we can approximate the excess empirical risk $\underline{R}_i(\w)$ by $R_i(\w)-\hat R_i^*$. Then, we send the modified risk function $\underline{\hat{R}}_i(\cdot)$ into~\cref{alg:gdro} to get a solution for~\eqref{eq:empirical-mero-wq-surrogate}. We stick to the aforementioned budget $T$ to regulate the cost. Based on~\cref{lem:opt-error} and~\cref{thm:gdro-convergence}, we have the following convergence guarantee for~\cref{alg:mero}.
\begin{theorem}
    Under the conditions in~\cref{thm:excess-risk-convergence}, our~\cref{alg:mero} ensures that 
    with probability at least $1-\delta$,
    \begin{equation}
    \begin{aligned}
   \underline\epsilon(\bar\z)\le\OO\left(\frac{1}{T}\left(\sqrt{\ln m\ln{\frac{T}{\sqrt{\bar n}\delta}}}+\sqrt{\ln{\frac{mT}{\sqrt{\bar n}\delta}}}\right.\right.\\\left.\left.+\sqrt{\frac{\ln m}{\bar n}}\ln{\frac{T}{\sqrt{\bar n}\delta}}\right)\right).
    \end{aligned}\label{eq:mero-probability-bound}
    \end{equation}
   %  \begin{equation}
   %      \E\sqbrac{ \underline\epsilon(\bar\z)}\le\OO\brac{\frac{1}{T}\sqrt{\ln {\frac{mT}{\sqrt{\bar n}}}}},\label{eq:mero-expectation-bound}
   %  \end{equation}
   %  and with probability at least $1-\delta$, 
   %  \begin{equation}
   %  \begin{aligned}
   % \underline\epsilon(\bar\z)\le\OO\brac{\frac{1}{T}\sqbrac{\sqrt{\ln m\ln{\frac{T}{\sqrt{\bar n}\delta}}}+\sqrt{\ln{\frac{mT}{\sqrt{\bar n}\delta}}}+\sqrt{\frac{\ln m}{\bar n}}\ln{\frac{T}{\sqrt{\bar n}\delta}}}}.
   %  \end{aligned}\label{eq:mero-probability-bound}
   %  \end{equation}
    \label{thm:mero-convergence}
\end{theorem}
% \begin{remark}
%     Both the expectation bound and the probability bound rely on the high probability bound~\eqref{eq:excess-risk-probability-bound} in~\cref{thm:excess-risk-convergence}. Running other ERM or empirical GDRO algorithms fails to provide such high-probability convergence.
% \end{remark}
\begin{corollary}
    Based on~\cref{thm:excess-risk-convergence,thm:mero-convergence}, the total computation complexity for~\cref{alg:mero} to reach $\varepsilon$-accuracy of~\eqref{eq:empirical-mero-wq} is $\tilde{\OO}\brac{\frac{m\sqrt{\bar n\ln{m}}}{\varepsilon}}$.\label{cor:mero-complexity}
\end{corollary}
\begin{remark}
    \cref{cor:mero-complexity} shows that the complexity of ALEM nearly matches that of the empirical GDRO problem, which significantly improves over the $\OO\brac{\frac{m\bar n\ln(m\bar n)}{\varepsilon^2}}$ computation complexity of ERMEG~\citep{agarwal2022minimax}.
\end{remark}

\section{Experiments\label{sec:experiments}}

In this section, we conduct numerical experiments on empirical GDRO and empirical MERO to evaluate the performance of our algorithms.

\subsection{Setup}
For the empirical GDRO problem, We follow the setup in previous literature~\citep{namkoong2016stochastic,soma2022optimal,zhang2023stochastic}, using both synthetic and real-world datasets. Our goal is to find a single classifier to minimize the maximal empirical risk across all categories.

For the synthetic dataset, we set the number of groups to be 25. For each $i\in[25]$, we draw $\w_i^*\in\R^{1024}$ from the uniform distribution over the unit sphere. The data sample $\cbrac{\xi_{ij}}_{j\in[n_i]}$ of group $i$ is generated by $\xi_{ij}=\brac{\x_{ij},y_{ij}}$, where
% \begin{equation}
% \xi_{ij}=\brac{\x_{ij},y_{ij}};\quad\x_{ij}\sim\mathcal{N}\brac{0,I},
%     y_{ij}=\begin{cases}
%     \text{sign}\brac{\x_{ij}^T\w_i^*},&\text{with probability 0.9},\\
%     -\text{sign}\brac{\x_{ij}^T\w_i^*},&\text{with probability 0.1}.\end{cases}
% \end{equation}
\begin{equation*}
\begin{aligned}
&\x_{ij}\sim\mathcal{N}\brac{0,I},\\
    &y_{ij}=\begin{cases}
    \text{sign}\brac{\x_{ij}^T\w_i^*},&\text{with probability 0.9},\\
    -\text{sign}\brac{\x_{ij}^T\w_i^*},&\text{with probability 0.1}.\end{cases}
\end{aligned}
\end{equation*}
We set $\ell(\cdot;\cdot)$ as the logistic loss and use different methods to train a linear model for this binary classification problem.

For the real-world dataset, we use CIFAR-100~\citep{krizhevsky2009learning}, which has 100 classes containing 500 training images and 100 testing images for each class. Our goal is to determine the class for each image. We set $m=100$ according to the number of categories and therefore the empirical risk function for group $i$ is exactly the average loss function amongst all images of this class. We set $\ell(\cdot;\cdot)$ as the softmax loss function for this multi-class classification problem. The underlying predictive model remains to be linear, which satisfies the convex-concave setting.

For the empirical MERO problem, we aim to conduct a similar task as before. We stick to CIFAR-100 as the real-world dataset. However, to simulate the scenario where the groups of distributions differ from each other, we introduce heterogeneous noise in the synthetic dataset, and generate $\xi_{ij}=(\mathbf x_{ij},y_{ij})$, where
% \begin{equation}
%     \xi_{ij}=(\mathbf x_{ij},y_{ij});\mathbf x_{ij}\sim\mathcal{N}(0,I),
%     y_{ij}=\begin{cases}
%     \text{sign}(\mathbf x_{ij}^T\mathbf w_i^*),&\text{with probability }p_i=0.95-\frac{i}{160},\\
%     -\text{sign}(\mathbf x_{ij}^T\mathbf w_i^*),&\text{with probability }1-p_i.\end{cases}
% \end{equation}
\begin{equation*}
\begin{aligned}
    &\mathbf x_{ij}\sim\mathcal{N}(0,I),\\&
    y_{ij}=\begin{cases}
    \text{sign}(\mathbf x_{ij}^T\mathbf w_i^*),&\text{with probability }p_i=0.95-\frac{i}{160},\\
    -\text{sign}(\mathbf x_{ij}^T\mathbf w_i^*),&\text{with probability }1-p_i,\end{cases}
\end{aligned}
\end{equation*}
for $i=0,1,\cdots,24$. The rest of the construction process follows the same steps as those used in the empirical GDRO experiments.

Different from the empirical GDRO experiments, we need to calculate the minimal empirical risk for all groups so as to evaluate our performance. We pass the data of each group to an Empirical Risk Minimization (ERM) oracle. Due to the convexity of the problem, running the oracle adequately long ensures the solution will closely approximate the true minimal empirical risk. After this process, we regard the outputs of the oracle as the true values $\{R_i^*\}_{i\in[m]}$.

\begin{figure*}[tbp]
	\centering
 \captionsetup{justification=centering}
 % \captionsetup[subfigure]{singlelinecheck=off,justification=centering}
 % \subcaptionsetup[figure]{justification=centering}
 % \captionsetup[sub]{justification=centering}
	\subfigure[\centering Training on the synthetic data\centering]{\includegraphics[width=.245\textwidth]{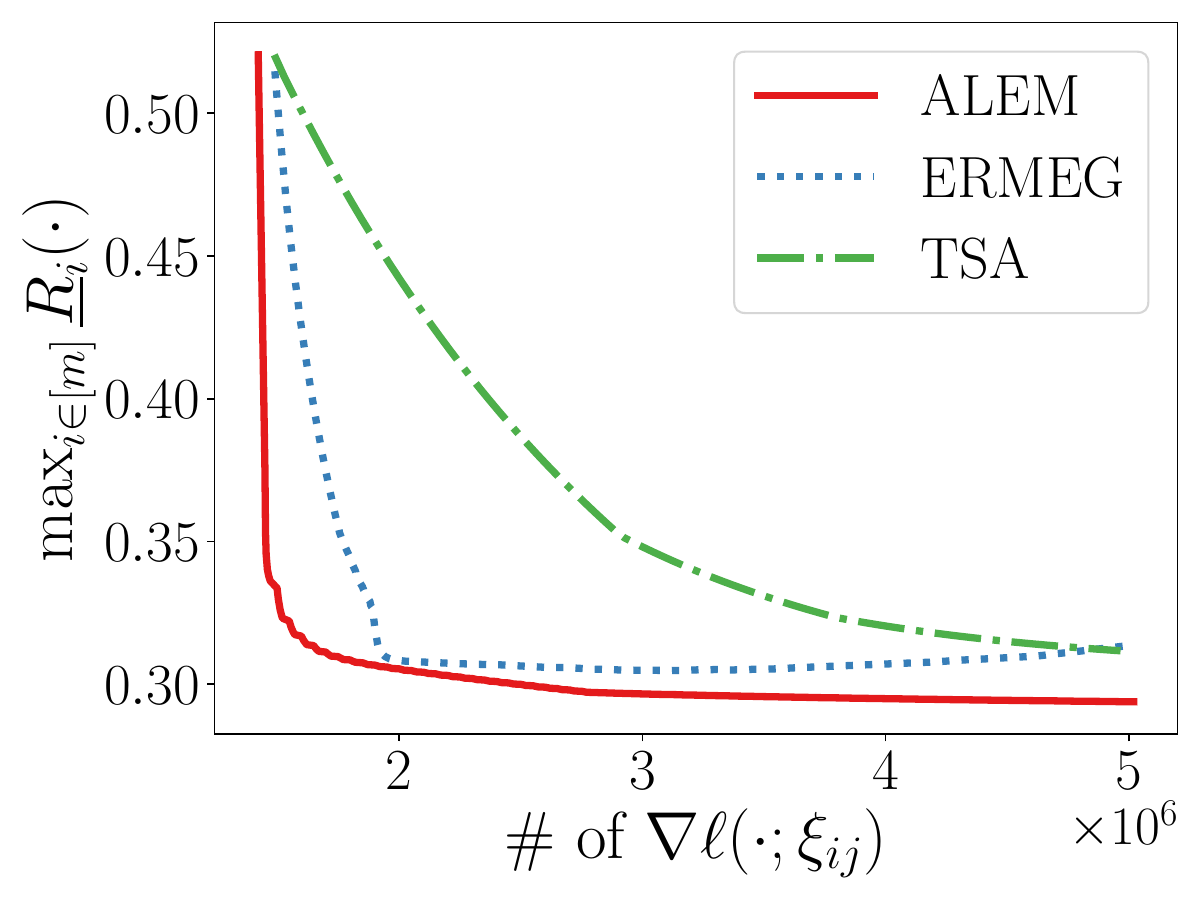}\label{fig:synthetic_train-mero}}
 % \hfill
	\subfigure[\centering Test on the synthetic data]{\includegraphics[width=.245\textwidth]{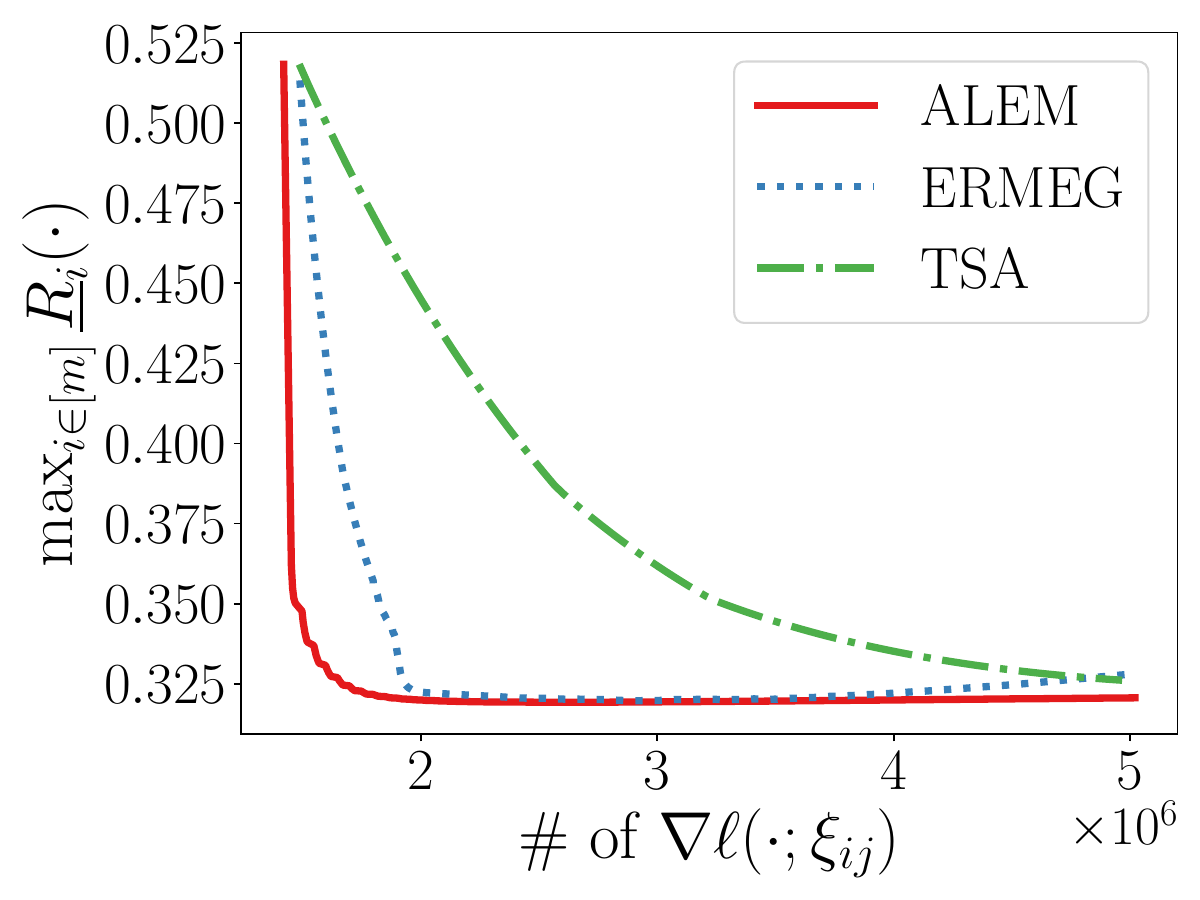}\label{fig:synthetic_test-mero}}
	\subfigure[\centering Training on CIFAR-100]{\includegraphics[width=.245\textwidth]{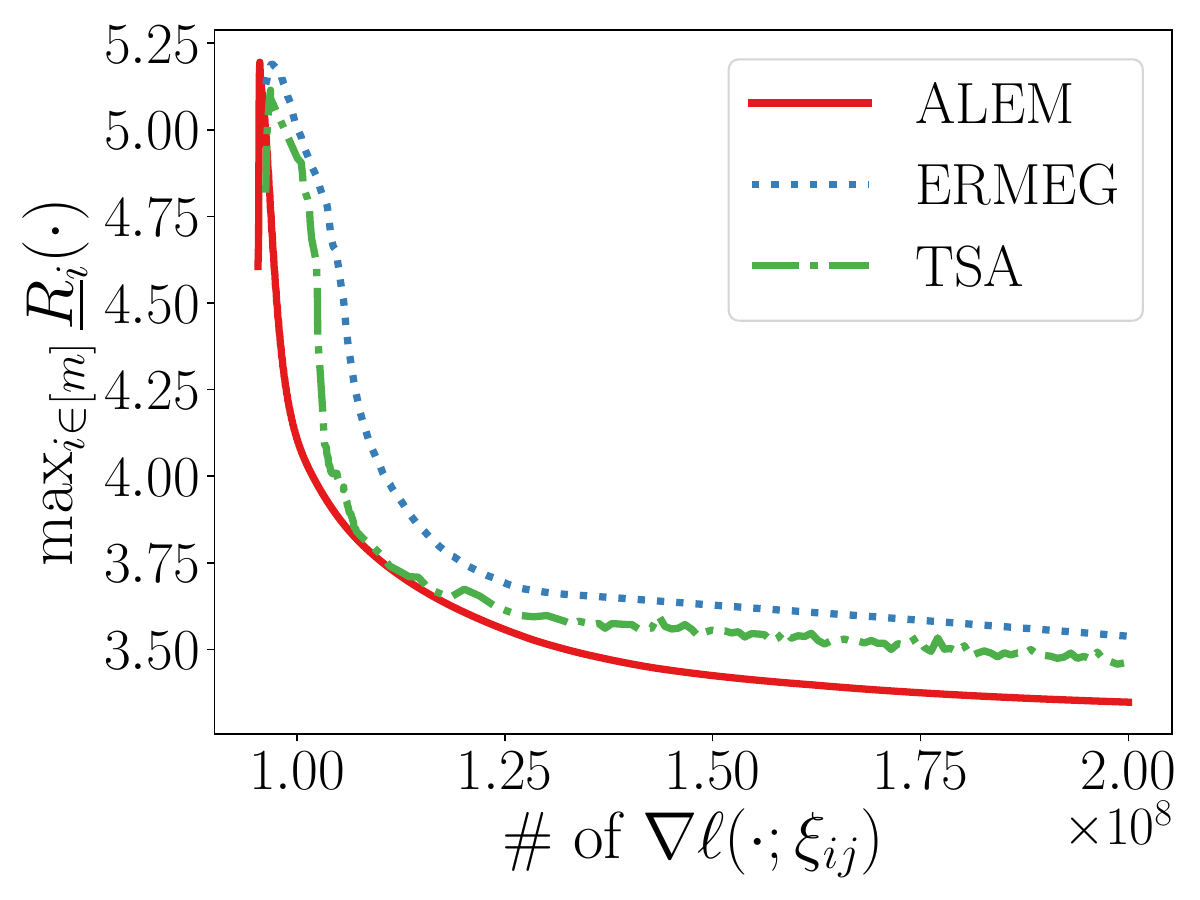}\label{fig:cifar_train-mero}}
 % \hfill
	\subfigure[\centering Test on CIFAR-100]{\includegraphics[width=.245\textwidth]{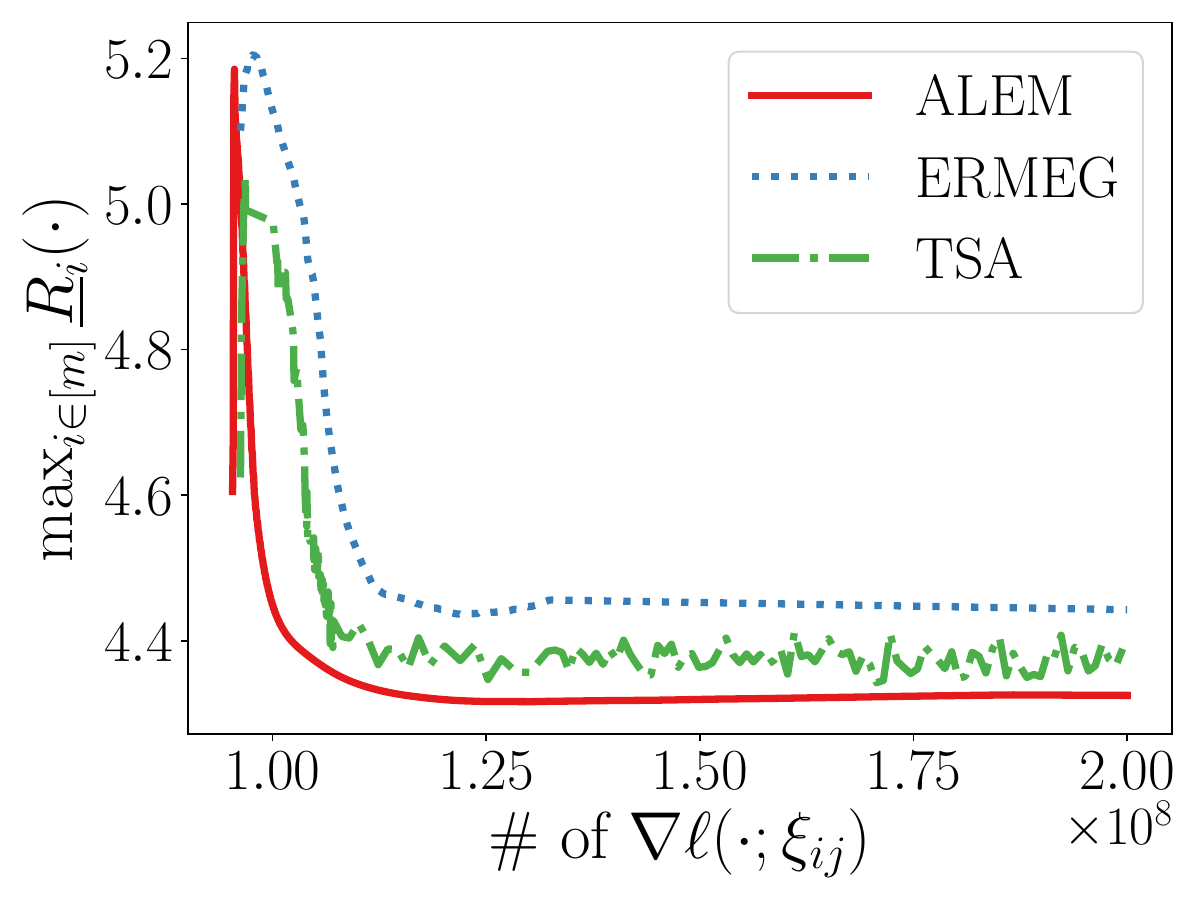}\label{fig:cifar_test-mero}}
	\caption{Comparison of the max excess empirical risk $\max_{i\in[m]}\underline{R}_i(\cdot)$ with respect to the number of stochastic gradient evaluations $\# \text{ of } \nabla\ell(\cdot;\xi_{ij})$ on the synthetic dataset and the CIFAR-100 dataset.\label{fig:synthetic-plus-cifar-mero}}
\end{figure*}

\subsection{Results for Empirical GDRO}

% \begin{figure}[tbp]
% 	\centering
%  % \captionsetup{justification=centering}
% 	\subfigure[\centering Training set of the synthetic data]{\includegraphics[width=.49\columnwidth]{fig/gdro/synthetic_train.pdf}\label{fig:synthetic_train}}\hfill
% 	\subfigure[\centering Test set of the synthetic data]{\includegraphics[width=.49\columnwidth]{fig/gdro/synthetic_test.pdf}\label{fig:synthetic_test}}\\
% 	\subfigure[\centering Training set of CIFAR-100]{\includegraphics[width=.49\columnwidth]{fig/gdro/cifar_train.pdf}\label{fig:cifar_train}}\hfill
% 	\subfigure[\centering Test set of CIFAR-100]{\includegraphics[width=.49\columnwidth]{fig/gdro/cifar_test.pdf}\label{fig:cifar_test}}
% 	\caption{Comparison of the max empirical risk $\max_{i\in[m]}R_i(\cdot)$ with respect to the number of stochastic gradient evaluations $\# \text{ of } \nabla\ell(\cdot;\xi_{ij})$ for synthetic dataset and CIFAR-100 dataset.\label{fig:synthetic-plus-cifar}}
% \end{figure}

To evaluate the performance measure, we report the maximal empirical risk $\max_{i\in[m]}R_i(\cdot)$ on the training set. In order to show the generalization abilities, we also report the maximal empirical risk on the test set. To evaluate different algorithms, we use the number of stochastic gradient evaluations to reflect the computation complexity. 

We compare our algorithm ALEG with SMD~\citep{nemirovski2009robust}, MPVR~\citep{alacaoglu2022stochastic} and AL-SVRE~\citep{luo2021near}. The results are shown in~\cref{fig:synthetic-plus-cifar}. We emphasize that our implementation of ALEG supports changeable hyperparameters in terms of $\cbrac{K_s},\cbrac{\alpha_k^s},\cbrac{\eta_k^s}$, which helps to boost its overall performance according to our observations. However, to fairly compare ALEG with others, we stick to the settings in~\cref{thm:gdro-convergence}, i.e., constant $\cbrac{K_s},\cbrac{\alpha_k^s}$ and alterable $\cbrac{\eta_k^s}$.

On the synthetic dataset, ALEG demonstrates notably faster convergence compared to other methods, in terms of both the training set and test set. Additionally, it achieves a lower maximal empirical risk than the alternatives. For the CIFAR-100 dataset, finding a single classifier becomes challenging because the computation complexity is proportional to the number of groups $m$. Under such a challenging task, ALEG still performs significantly better than others. 

On the training set of CIFAR-100, ALEG also significantly outperforms other methods in terms of convergence speed and quality of the solution. On the CIFAR-100 test set, ALEG demonstrates faster convergence and greater stability than the other three algorithms, showcasing its superior generalization capability. While SMD and MPVR behave similarly on the training set, MPVR demonstrates its robustness on the test set, indicating variance reduction has an edge over pure stochastic algorithms like SMD. 
% In contrast, ALEG progressively decreases the learning rates, which exhibits competitive performance throughout the experiment.

% \begin{figure}[tbp]
% 	\centering
%  % \captionsetup{justification=centering}
% 	\subfigure[\centering Training set of the synthetic data]{\includegraphics[width=.49\columnwidth]{fig/mero/synthetic_train.pdf}\label{fig:synthetic_train-mero}}\hfill
% 	\subfigure[\centering Test set of the synthetic data]{\includegraphics[width=.49\columnwidth]{fig/mero/synthetic_test.pdf}\label{fig:synthetic_test-mero}}\\
% 	\subfigure[\centering Training set of CIFAR-100]{\includegraphics[width=.49\columnwidth]{fig/mero/cifar_train.pdf}\label{fig:cifar_train-mero}}\hfill
% 	\subfigure[\centering Test set of CIFAR-100]{\includegraphics[width=.49\columnwidth]{fig/mero/cifar_test.pdf}\label{fig:cifar_test-mero}}
% 	\caption{Comparison of the max excess empirical risk $\max_{i\in[m]}\underline{R}_i(\cdot)$ with respect to the number of stochastic gradient evaluations $\# \text{ of } \nabla\ell(\cdot;\xi_{ij})$ for synthetic dataset and CIFAR-100 dataset.\label{fig:synthetic-plus-cifar-mero}}
% \end{figure}

\subsection{Results for Empirical MERO}

We compare ALEM with ERMEG~\citep[Section~6]{agarwal2022minimax} and TSA~\citep[Algorithm~2]{zhang2023efficient}. Similarly, we report the max excess empirical risk $\max_{i\in[m]} \underline{R}_i(\cdot)$ on both the training set and the test set. The results are presented in~\cref{fig:synthetic-plus-cifar-mero}.

Note that all three algorithms follow the two-stage schema, which means they have to estimate $m$ minimal empirical risks before solving the approximated empirical GDRO problem. This explains why the $x$-axis of the four figures in~\cref{fig:synthetic-plus-cifar-mero} doesn't start at zero. The result in~\cref{fig:synthetic_train-mero,fig:cifar_train-mero} shows that our algorithm converges more rapidly and is more stable than ERMEG or TSA. The result in~\cref{fig:synthetic_test-mero,fig:cifar_test-mero} validates the strong generalization ability of ALEM.

In terms of the running time, we also observe that our algorithm performs significantly faster than TSA and ERMEG, especially for the latter. ERMEG not only needs to estimate the minimal empirical risk for all groups, but it also runs an ERM oracle every iteration, which is highly time-consuming and impractical.

\section{Conclusion}
We develop a variance-reduced stochastic mirror prox algorithm called ALEG to target the empirical GDRO problem. 
% ALEG employs a typical double-loop structure of variance reduction algorithms and modifies the traditional mirror prox algorithm by using a full gradient at~the snapshot point and a variance-reduced stochastic gradient at~the past trajectory to compute two sets of solutions. 
Specifically, we propose a simple yet effective group sampling strategy and a novel variable routine to reduce the complexity and enhance the algorithmic flexibility, respectively. 
% The group sampling technique reduces the Lipschitz constant of the stochastic gradient by $m$. We further conduct variance reduction for every group to capture the two-level finite-sum structure of empirical GDRO. The alterable hyperparameters, which are supported by the novel one-index-shifted weighted average of the (mirror) snapshot points and the Lyapunov functions, endow us with more flexibility in practical applications.
ALEG attains an $\OO\brac{\frac{m\sqrt{\bar n\ln{m}}}{\varepsilon}}$ computation complexity, which improves the state-of-the-art result by a factor of $\sqrt{m}$. 

Based on ALEG, we develop a two-stage algorithm called ALEM to cope with the empirical MERO problem. In the first stage, ALEM runs ALEG to estimate the minimal empirical risk for all groups. In the second stage, ALEM utilizes ALEG to solve an approximate empirical MERO problem. We also establish an $\tilde{\OO}\brac{\frac{m\sqrt{\bar n\ln{m}}}{\varepsilon}}$ computation complexity, improving over the existing methods. Finally, we conduct experiments on the synthetic dataset as well as the real-world dataset to validate the effectiveness of our algorithms.

\section*{Acknowledgements}

This work was partially supported by the National Science and Technology Major Project (2022ZD0114801) and NSFC (62122037, 61921006).

\section*{Impact Statement}

This paper presents work whose goal is to advance the field of Machine Learning. There are many potential societal consequences of our work, none of which we feel must be specifically highlighted here.

% In the unusual situation where you want a paper to appear in the references without citing it in the main text, use \nocite
\nocite{NEURIPS2023_c242f2b7,mp21ouyang,pmlr-v125-lin20a,pmlr-v119-xie20d,rockafellar2015convex,nesterov2009primal,gao2023distributionally,yazdandoost2023stochastic,chen2022faster,yang2020global,Bertsimas2014RobustSA,hu2018does,shapiro2017distributionally,duchi2021learning,duchi2021statistics,dang2015convergence,nesterov2018lectures,palaniappan2016stochastic,jiang2022multi,jin2021nonconvex,zhao2022accelerated,Lee2021fastextragradient,mokhtari2020convergence,ouyang2021lower,liu2021first,carmon2019variance,allen2016variance,allen2017linear,sagawa2019distributionally,reddi2016stochastic,condat2016fast,shalev2009stochastic,namkoong2016stochastic,krizhevsky2009learning,rothblum2021multi,JMLR:v25:21-0264,rakhlin2012making,sadiev2023high,kakade2008generalization,pmlr-v119-lin20a,NIPS2017_186a157b,laguel2021superquantile,laguel2021superquantiles,pillutla2023federated,mehta2023stochastic,mehta2024distributionally}

\bibliography{example_paper}
\bibliographystyle{icml2024}

%%%%%%%%%%%%%%%%%%%%%%%%%%%%%%%%%%%%%%%%%%%%%%%%%%%%%%%%%%%%%%%%%%%%%%%%%%%%%%%
%%%%%%%%%%%%%%%%%%%%%%%%%%%%%%%%%%%%%%%%%%%%%%%%%%%%%%%%%%%%%%%%%%%%%%%%%%%%%%%
% APPENDIX
%%%%%%%%%%%%%%%%%%%%%%%%%%%%%%%%%%%%%%%%%%%%%%%%%%%%%%%%%%%%%%%%%%%%%%%%%%%%%%%
%%%%%%%%%%%%%%%%%%%%%%%%%%%%%%%%%%%%%%%%%%%%%%%%%%%%%%%%%%%%%%%%%%%%%%%%%%%%%%%

\newpage
\appendix
\onecolumn
\section{Analysis for Empirical GDRO\label{app:sec:gdro-analysis}}
We present the omitted proofs for empirical GDRO in this section. Firstly, we conduct some necessary technical preparations in~\cref{app:sec:preparations}. Then, we analyze the algorithmic variance-reduced behaviors in~\cref{app:sec:vr-routine} and give an implicit upper bound for the variances of the stochastic gradients along the trajectory. Finally, we proceed to show the convergence and the derived complexity in~\cref{app:sec:expectation-bound}.

\subsection{Preparations\label{app:sec:preparations}}
Here we provide some definitions to facilitate understanding and bring convenience.
\begin{definition}
    (Saddle point) Define any solution to~\eqref{eq:empirical-gdro-wq} as $\z_*=(\w_*;\q_*)$.
\end{definition}

\begin{definition}
    (Martingale difference sequence) Define $\Delta_k^s=\g_k^s-\nabla F(\z_{k+1/2}^s)$.\label{def:Delta_k^s}
\end{definition}

\begin{definition}
    (Periodically decaying sequence) We call $\{\alpha_{k-1}^{s-1}\}_{s\in[S],k\in[K_s]}$ a periodically decaying sequence if it satisfies: (i) $\alpha_{K_s-1}^s\le \alpha_0^{s+1}$; (ii) $\sumks \alpha_{k-1}^s\le\sum_{k=1}^{K_{s-1}}\alpha_{k-1}^{s-1}$. \label{def:periodically-decaying-sequence}
\end{definition}

\begin{definition}
    (Lyapunov function) For Bregman divergence defined in~\eqref{eq:def-z-bregman-divergence}, we define
    \begin{equation}
        \Psi_s(\z):=(1-\alpha_0^s) B(\z,\z_0^s)+\sum_{k=1}^{K_{s-1}} \alpha_{k-1}^{s-1}B(\z,\z_k^{s-1}).
    \end{equation}\label{def:lyapunov}
\end{definition}

\begin{remark}
    For~\cref{def:Delta_k^s}, we know that $\Delta_k^s$ equals zero in conditional expectation, which is verified in~\cref{prop:unbiased-sg}.~\cref{def:periodically-decaying-sequence} could be easily satisfied both in theoretical analysis and real-world experimental settings. Intuitively, our design of $\{\alpha_{k}^{s}\}_{s\in[S],k\in[K_s]}$ is inspired by cyclical learning rate \citep{smith2017cyc}, with the property of periodic decay.
\end{remark}

 In the following analysis, we stick to the choice $\norm{\cdot}_q=\norm{\cdot}_1$ not just for simplicity, but also more practical when it comes to implementation. This choice enables the mirror descent step for $\q$ to have a closed-form solution. Note that in this case, we have $D_q^2=\ln{m}$ and $\alpha_q=1$. Without loss of generality, we also assume $\alpha_w=1$. There are some important facts shown in the following lemmas.
 
\begin{proposition}
    For any $\z_1,\z_2\in\Z$, $\norm{\z_1-\z_2}\le 2\sqrt{2}$.\label{prop:z-norm-boundness}
\end{proposition}

\begin{proof}
Define $\z_0=\argmin_{\z\in\Z}\psi(\z)$. By the fact that $\max_{z\in\Z}B(\z,\z_0)\le \max_{z\in\Z}\psi(\z)-\min_{z\in\Z}\psi(\z)\le 1$, we have
    \begin{equation}
        \begin{aligned}
            \norm{\z_1-\z_2}&\le \norm{\z_1-\z_0}+\norm{\z_2-\z_0}\le \sqrt{2B(\z_1,\z_0)}+\sqrt{2B(\z_2,\z_0)}\\&\le \sqrt{2\max_{z\in\Z}B(\z,\z_0)}+\sqrt{2\max_{z\in\Z}B(\z,\z_0)}\le 2\sqrt{2}.
        \end{aligned}
    \end{equation}
\end{proof}

\begin{proposition}
    (Unbiasedness of the merged stochastic gradient operator) The stochastic gradient $\nabla F(\z^s;\xi_k^s)$ defined in~\eqref{eq:stochastic-gradient} is unbiased.\label{prop:unbiased-sg}
\end{proposition}

\begin{proof}
    From our group sampling strategy defined in~\eqref{eq:per-group-sampling}, we have \begin{equation}
        \forall\z\in\Z,\forall\ i\in[m]:\quad\E\sqbrac{\ell(\w;\xi^s_{k,i})}=R_i(\w),\quad\E\sqbrac{\nabla\ell(\w;\xi^s_{k,i})}=\nabla R_i(\w).
    \end{equation}
    From the linearity of expectation, it's easy to deduce that $\E\sqbrac{\nabla F(\z;\xi_k^s)}=\nabla F(\z)$ for any $\z\in\Z$. Hence the unbiasedness of the stochastic gradient $\nabla F(\z^s;\xi_k^s)$ is proved naturally. 
\end{proof}

% \begin{lemma}(Restatement of~\cref{lem:Lipschitz}) Define \begin{equation}
% L_z:=2D_w\max\cbrac{\sqrt{2D^2_wL^2+G^2\ln{m}},G\sqrt{2\ln{m}}}.
% \end{equation}
% For any $s\in[S]^0, k\in[K_s]^0$, $\nabla F(\z;\xi_{k}^s)$ is $L_z$-Lipschitz continuous.
% \end{lemma}

\begin{proof}[Proof of~\cref{lem:Lipschitz}]
    Pick two arbitrary points $\z^+=(\w^+,\q^+)\in\Z,z=(\w,\q)\in\Z$. First, we bound the gradient in $\w$ as follows:
    \begin{equation}
        \begin{aligned}
            &\norm{\nabla_\w F(\z^+;\xi_{k}^s)-\nabla_\w F(\z;\xi_{k}^s)}_{w,*}^2\\=&\norm{\summ \q^+_i[\nabla\ell(\w^+;\xi_{k,i}^s)-\nabla\ell(\w;\xi_{k,i}^s)]+\summ (\q^+_i-\q_i)\nabla\ell(\w;\xi_{k,i}^s)}_{w,*}^2\\\le&2\norm{\summ \q^+_i[\nabla\ell(\w^+;\xi_{k,i}^s)-\nabla\ell(\w;\xi_{k,i}^s)]}_{w,*}^2+2\norm{\summ (\q^+_i-\q_i)\nabla\ell(\w;\xi_{k,i}^s)}_{w,*}^2\\\le&2\summ \q^+_i\norm{\nabla\ell(\w^+;\xi_{k,i}^s)-\nabla\ell(\w;\xi_{k,i}^s)}_{w,*}^2+2\brac{\summ|\q^+_i-\q_i|\norm{\nabla\ell(\w;\xi_{k,i}^s)}_{w,*}}^2\\\le&2\summ\q^+_i L^2\norm{\w^+-\w}_w^2+2\brac{\summ|\q^+_i-\q_i|G}^2\\=&2L^2\norm{\w^+-\w}_w^2+2G^2\norm{\q^+-\q}_1^2.
        \end{aligned}
    \end{equation}
    The second inequality uses~\cref{asp:smooth-Lipschitz}. Next, we bound the gradient in $\q$. Again from~\cref{asp:smooth-Lipschitz}, we have
    % \begin{equation}
    %     \begin{aligned}
    %         \forall\ i\in[m]:\exists\ \bar\w_i,\quad\text{s.t.} \quad \ell(\w^+;\xi_{k,i}^s)-\ell(\w;\xi_{k,i}^s)&=\braket{\nabla\ell(\bar\w_i;\xi_{k,i}^s),\w^+-\w}\\&\le\norm{\nabla\ell(\bar\w_i;\xi_{k,i}^s)}_{w,*}\norm{\w^+-\w}_{w}\\&\le G\norm{\w^+-\w}_{w}.
    %     \end{aligned}\label{eq:iGw^+-w}
    % \end{equation}
     \begin{equation}
        \begin{aligned}
            \forall\ i\in[m]:,\quad \ell(\w^+;\xi_{k,i}^s)-\ell(\w;\xi_{k,i}^s)\le G\norm{\w^+-\w}_{w}.
        \end{aligned}\label{eq:iGw^+-w}
    \end{equation}   
    Notice that $\nabla_\q F(\z;\xi_k^s)=\left[\ell(\w;\xi_{k,1}^s),\cdots,\ell(\w;\xi_{k,m}^s)\right]^T$. Squaring it on both sides of~\eqref{eq:iGw^+-w} and taking maximum over all $i\in[m]$, we have
    \begin{equation}
        \norm{\nabla_\q F(\z^+;\xi_k^s)-\nabla_\q F(\z;\xi_k^s)}_\infty^2=\max_{i\in[m]}[\ell(\w^+;\xi_{k,i}^s)-\ell(\w;\xi_{k,i}^s)]^2\le G^2\norm{\w^+-\w}_w^2.
    \end{equation}
    By merging the two component's gradients, we get the desired result by simple calculation:
    \begin{equation}
        \begin{aligned}
            &\norm{\nabla F(\z^+;\xi_k^s)-\nabla F(\z;\xi_k^s)}_*^2\\=&2D_w^2\norm{\nabla_\w F(\z^+;\xi_{k}^s)-\nabla_\w F(\z;\xi_{k}^s)}_{w,*}^2+2\ln{m}\norm{\nabla_\q F(\z^+;\xi_k^s)-\nabla_\q F(\z;\xi_k^s)}_\infty^2\\\le&(4D_w^2L^2+2G^2\ln m)\norm{\w^+-\w}_w^2+4D^2_wG^2\norm{\q^+-\q}_1^2\\\le&\frac{1}{2D^2_w}\sqbrac{4D_w^2(2D_w^2L^2+G^2\ln{m})\norm{\w^+-\w}_w^2}+\frac{1}{2\ln{m}}\sqbrac{4D_w^2(2G^2\ln{m})\norm{\q^+-\q}_1^2}\\\le&\frac{L_z^2}{2D^2_w}\norm{\w^+-\w}_w^2+\frac{L_z^2}{2\ln{m}}\norm{\q^+-\q}_1^2\\=&L_z^2\norm{\z^+-\z}^2.
        \end{aligned}\label{eq:simple-merge}
    \end{equation}
\end{proof}

\begin{lemma} (Optimality condition for mirror descent)
    For any $\g\in\mathcal{E}^*\times\R^m$, let $\z^{t+1}=\argmin_{\z\in\Z}\{\braket{\g,\z}+\alpha B(\z,\z_1)+(1-\alpha)B(\z,\z_2)\}$. It holds that
    % \begin{equation}\begin{aligned}
    %     \braket{\g,\z^{t+1}-\z}\le-B(\z,\z^{t+1})+\alpha&\sqbrac{B(\z,\z_1)-B(\z^{t+1},\z_1)}\\+(1-\alpha)&\sqbrac{B(\z,\z_2)-B(\z^{t+1},\z_2)},\quad \forall\ \z\in\Z.
    % \end{aligned}\end{equation}
    \begin{equation}\begin{aligned}
        \braket{\g,\z^{t+1}-\z}\le-B(\z,\z^{t+1})+\alpha\sqbrac{B(\z,\z_1)-B(\z^{t+1},\z_1)}+(1-\alpha)\sqbrac{B(\z,\z_2)-B(\z^{t+1},\z_2)},\quad \forall\ \z\in\Z.
    \end{aligned}\end{equation}
    \label{lem:opt-mirror-descent}
\end{lemma}

\begin{proof}
    The proof is straightforward by applying the first-order optimality condition for the definition of $\z^{t+1}$. Then a direct application of three point equality of Bregman divergence yields the result. By the first order optimality of $\z^{t+1}$,
    \begin{equation}
        \mathbf{0}\in \g+\alpha\sqbrac{\nabla\psi(\z^{t+1})-\nabla\psi(\z_{1})}+\brac{1-\alpha}\sqbrac{\nabla\psi(\z^{t+1})-\nabla\psi(\z_{2})}+N_{\Z}(\z^{t+1})
    \end{equation}
    where $N_{\Z}(\z^{t+1}):=\{\g\in\mathcal{E}^*\times\R^m|\braket{\g,\z-\z^{t+1}}\le0,\forall\z\in\Z\}$ is the normal cone (subdifferential of indicator function) at point $\z^{t+1}$ for convex set $\Z$. The above relation further implies\begin{equation}
        \braket{\g+\alpha\sqbrac{\nabla\psi(\z_{1})-\nabla\psi(\z^{t+1})}+\brac{1-\alpha}\sqbrac{\nabla\psi(\z_{2})-\nabla\psi(\z^{t+1})},\z-\z^{t+1}}\le0,\ \forall\ \z\in\Z.\label{eq:subdifferential}
    \end{equation}
    According to the generalized triangle inequality for Bregman divergence, we have
    \begin{equation}
        \braket{\nabla\psi(\z_{i})-\nabla\psi(\z^{t+1}),\z-\z^{t+1}}=B(\z,\z^{t+1})+B(\z^{t+1},\z_i)-B(\z,\z_i),\quad\forall\ i\in\{1,2\}.\label{eq:bregman-triangle}
    \end{equation}
    Applying~\eqref{eq:bregman-triangle} to the LHS of~\eqref{eq:subdifferential} to derive that for any $\z\in\Z$,
    % \begin{equation}
    % \begin{aligned}
    %     &\braket{\g,\z-\z^{t+1}}+\alpha\sqbrac{B(\z,\z^{t+1})+B(\z^{t+1},\z_1)-B(\z,\z_1)}\\&+(1-\alpha)\sqbrac{B(\z,\z^{t+1})+B(\z^{t+1},\z_2)-B(\z,\z_2)}\\=&\braket{\g,\z-\z^{t+1}}+B(\z,\z^{t+1})+\alpha\sqbrac{B(\z^{t+1},\z_1)-B(\z,\z_1)}\\&+(1-\alpha)\sqbrac{B(\z^{t+1},\z_2)-B(\z,\z_2)}\le 0,
    % \end{aligned}        
    % \end{equation}
    \begin{equation}
    \begin{aligned}
        &\braket{\g,\z-\z^{t+1}}+\alpha\sqbrac{B(\z,\z^{t+1})+B(\z^{t+1},\z_1)-B(\z,\z_1)}+(1-\alpha)\sqbrac{B(\z,\z^{t+1})+B(\z^{t+1},\z_2)-B(\z,\z_2)}\\=&\braket{\g,\z-\z^{t+1}}+B(\z,\z^{t+1})+\alpha\sqbrac{B(\z^{t+1},\z_1)-B(\z,\z_1)}+(1-\alpha)\sqbrac{B(\z^{t+1},\z_2)-B(\z,\z_2)}\le 0,
    \end{aligned}        
    \end{equation}
    which concludes our proof by a simple rearrangement.
\end{proof}

\subsection{Variance-Reduced Routine\label{app:sec:vr-routine}}
\begin{lemma}
    If $\{\alpha_k^s\}_{s\in[S]^0,k\in[K_s]^0}$ is a periodically decaying sequence as defined in~\cref{def:periodically-decaying-sequence}, for any $\z\in\Z$: \begin{equation}\begin{aligned}
        &\sumkzs \eta_k^s\braket{\nabla F(\z_{k+1/2}^s),\z_{k+1/2}^s-\z}\\\le& \Psi_s(\z)-\Psi_{s+1}(\z)+\sumkzs\sqbrac{\eta_k^s\braket{\Delta_k^s,\z-\z_{k+1/2}^s}+\frac{(\eta_k^sL_z)^2-\alpha_k^s}{2}\norm{\z_{k+1/2}^s-\z^s}^2}.
    \end{aligned}
    \end{equation}
    \label{lem:s-epoch}
\end{lemma}

\begin{proof}
    Using~\cref{lem:opt-mirror-descent} on $\z_{k+1/2}^s$ and taking arbitrary $\z$ as $\z_{k+1}^s$, we have
    \begin{equation}
    \begin{aligned}
        \eta_k^s\braket{\nabla F(\z^s),\z^s_{k+1/2}-\z^s_{k+1}}\le-B(\z^s_{k+1},\z^s_{k+1/2})+\alpha_k^s&\sqbrac{B(\z^s_{k+1},\bar\z^s)-B(\z^s_{k+1/2},\bar\z^s)}\\+(1-\alpha_k^s)&\sqbrac{B(\z^s_{k+1},\z_k^s)-B(\z^s_{k+1/2},\z_k^s)}.
    \end{aligned}
    \end{equation}
    Using~\cref{lem:opt-mirror-descent} on $\z_{k+1}^s$, we have for any $\z\in\Z$,
    \begin{equation}
        \begin{aligned}
            \eta_k^s\braket{\g_k^s,\z_{k+1}^s-\z}\le-B(\z,\z_{k+1}^s)+\alpha_k^s &\sqbrac{B(\z,\bar\z^s)-B(\z_{k+1}^s,\bar\z^s)}\\+(1-\alpha_k^s)&\sqbrac{B(\z,\z_k^s)-B(\z_{k+1}^s,\z_k^s)}.
        \end{aligned}
    \end{equation}
    % \begin{equation}
    %     \begin{aligned}
    %         \eta_k^s\braket{\g_k^s,\z_{k+1}^s-\z}\le-B(\z,\z_{k+1}^s)+\alpha_k^s \sqbrac{B(\z,\bar\z^s)-B(\z_{k+1}^s,\bar\z^s)}+(1-\alpha_k^s)\sqbrac{B(\z,\z_k^s)-B(\z_{k+1}^s,\z_k^s)}.
    %     \end{aligned}
    % \end{equation}
    Adding them together: 
    \begin{equation}
        \begin{aligned}
            &\eta_k^s\braket{\nabla F(\z^s),\z^s_{k+1/2}-\z^s_{k+1}}+\eta_k^s\braket{\g_k^s,\z_{k+1}^s-\z}\\\le &-B(\z^s_{k+1},\z^s_{k+1/2})-B(\z,\z_{k+1}^s)+\alpha_k^s \sqbrac{B(\z,\bar\z^s)-B(\z_{k+1/2}^s,\bar\z^s)}\\&+(1-\alpha_k^s) \sqbrac{B(\z,\z_{k}^s)-B(\z_{k+1/2}^s,\z_{k}^s)}.
        \end{aligned}\label{eq:add-updates}
    \end{equation}
    According to~\eqref{eq:g_k^s} and~\cref{def:Delta_k^s} we have
    \begin{equation}
        \begin{aligned}
            &\eta_k^s\braket{\nabla F(\z_{k+1/2}^s),\z_{k+1/2}^s-\z}\\=&\eta_k^s\braket{\g_k^s,\z_{k+1/2}^s-\z}+\eta_k^s\braket{\nabla F(\z_{k+1/2}^s)-\g_k^s,\z_{k+1/2}^s-\z}\\=&\eta_k^s\braket{\g_k^s,\z_{k+1/2}^s-\z_{k+1}^s}+\eta_k^s\braket{\g_k^s,\z_{k+1}^s-\z}-\eta_k^s\braket{\Delta_k^s,\z_{k+1/2}^s-\z}\\=& \eta_k^s\braket{\nabla F(\z^s),\z^s_{k+1/2}-\z^s_{k+1}}+\eta_k^s\braket{\g_k^s,\z_{k+1}^s-\z}\\&+\eta_k^s\braket{\nabla F(\z_{k+1/2}^s;\xi_k^s)-\nabla F(\z^s;\xi_k^s),\z_{k+1/2}^s-\z_{k+1}^s}-\eta_k^s\braket{\Delta_k^s,\z_{k+1/2}^s-\z}.
        \end{aligned}\label{eq:simple-algebra}
    \end{equation}
    Adding~\eqref{eq:add-updates} and~\eqref{eq:simple-algebra}, we have
    \begin{equation}
        \begin{aligned}
            &\eta_k^s\braket{\nabla F(\z_{k+1/2}^s),\z_{k+1/2}^s-\z}\\\le&\eta_k^s\braket{\nabla F(\z^s;\xi_k^s)-\nabla F(\z_{k+1/2}^s;\xi_k^s),\z_{k+1}^s-\z_{k+1/2}^s}\\&+\alpha_k^s \sqbrac{B(\z,\bar\z^s)-B(\z_{k+1/2}^s,\bar\z^s)}+(1-\alpha_k^s) \sqbrac{B(\z,\z_{k}^s)-B(\z_{k+1/2}^s,\z_{k}^s)}\\&-B(\z_{k+1}^s,\z_{k+1/2}^s)-B(\z,\z_{k+1}^s)+\eta_k^s\braket{\Delta_k^s,\z-\z_{k+1/2}^s}.
        \end{aligned}\label{eq:add2updates}
    \end{equation}
    Applying Young's inequality to the inner product and further using the smoothness of the stochastic gradient (cf.~\cref{lem:Lipschitz}), the following estimation holds:
    \begin{equation}
        \begin{aligned}
            &\eta_k^s\braket{\nabla F(\z^s;\xi_k^s)-\nabla F(\z_{k+1/2}^s;\xi_k^s),\z_{k+1}^s-\z_{k+1/2}^s}
            \\\le&\frac{(\eta_k^s)^2}{2}\norm{\nabla F(\z^s;\xi_k^s)-\nabla F(\z_{k+1/2}^s;\xi_k^s)}^2+\frac{1}{2}\norm{\z_{k+1}^s-\z_{k+1/2}^s}^2
            \\\le&\frac{(\eta_k^sL_z)^2}{2}\norm{\z_{k+1/2}^s-\z^s}^2+\frac{1}{2}\norm{\z_{k+1}^s-\z_{k+1/2}^s}^2.
        \end{aligned}
        \label{eq:inner-product-bound}
    \end{equation}
    Recall the definition of $\nabla \psi(\bar\z^s)$ and use the linearity of Bregman functions to yield:
    \begin{equation}
        B(\z,\bar\z^s)-B(\z_{k+1/2}^s,\bar\z^s)=\brac{\sum_{k=1}^{K_{s-1}}\alpha_{k-1}^{s-1}}^{-1}\sum_{j=1}^{K_{s-1}} \alpha_{j-1}^{s-1}\sqbrac{B(\z,\z^{s-1}_j)-B(\z_{k+1/2}^s,\z^{s-1}_j)}.\label{eq:Bregman-linearity}
    \end{equation}
    According to the definition of $\z^s$, Jensen's Inequality and the strong-convexity of $\psi(\cdot)$, we get
    % \begin{equation}
    % \begin{aligned}
    %     &\brac{\sum_{k=1}^{K_{s-1}}\alpha_{k-1}^{s-1}}^{-1}\sum_{j=1}^{K_{s-1}}-\alpha_{j-1}^{s-1}B(\z_{k+1/2}^s,\z^{s-1}_j)\\\le& \brac{\sum_{k=1}^{K_{s-1}}\alpha_{k-1}^{s-1}}^{-1}\sum_{j=1}^{K_{s-1}}-\frac{\alpha_{j-1}^{s-1}}{2}\norm{\z_{k+1/2}^s-\z^{s-1}_j}^2\\\le&-\frac{1}{2}\norm{\z_{k+1/2}^s-\z^{s}}^2,
    % \end{aligned}\label{eq:strong-convexity-Jensen}
    % \end{equation}
    \begin{equation}
    \begin{aligned}
        &\brac{\sum_{k=1}^{K_{s-1}}\alpha_{k-1}^{s-1}}^{-1}\sum_{j=1}^{K_{s-1}}-\alpha_{j-1}^{s-1}B(\z_{k+1/2}^s,\z^{s-1}_j)\\\le& \brac{\sum_{k=1}^{K_{s-1}}\alpha_{k-1}^{s-1}}^{-1}\sum_{j=1}^{K_{s-1}}-\frac{\alpha_{j-1}^{s-1}}{2}\norm{\z_{k+1/2}^s-\z^{s-1}_j}^2\le-\frac{1}{2}\norm{\z_{k+1/2}^s-\z^{s}}^2,
    \end{aligned}\label{eq:strong-convexity-Jensen}
    \end{equation}
    and
    \begin{equation}
        -B(\z_{k+1}^s,\z_{k+1/2}^s)\le -\frac{1}{2}\norm{\z_{k+1}^s-\z_{k+1/2}^s}^2.\label{eq:strong-convexity}
    \end{equation}
    Combining~\eqref{eq:inner-product-bound}~\eqref{eq:Bregman-linearity}~\eqref{eq:strong-convexity-Jensen}~\eqref{eq:strong-convexity} with~\eqref{eq:add2updates} and casting out $-B(\z_{k+1/2}^s,\z_k^s)$, we have
    \begin{equation}
        \begin{aligned}
            &\eta_k^s\braket{\nabla F(\z_{k+1/2}^s),\z_{k+1/2}^s-\z}\\\le&\frac{(\eta_k^sL_z)^2}{2}\norm{\z_{k+1/2}^s-\z^s}^2+\frac{1}{2}\norm{\z_{k+1}^s-\z_{k+1/2}^s}^2-(1-\alpha_k^s)B(\z_{k+1/2}^s,\z_k^s)\\&+(1-\alpha_k^s) B(\z,\z_{k}^s)-B(\z,\z_{k+1}^s)+\brac{\sum_{k=1}^{K_{s-1}}\alpha_{k-1}^{s-1}}^{-1}\alpha_k^s\sum_{j=1}^{K_{s-1}} \alpha_{j-1}^{s-1}B(\z,\z_j^{s-1})\\&-\frac{1}{2}\norm{\z_{k+1/2}^s-\z_{k+1}^s}^2-\frac{\alpha_k^s}{2}\norm{\z_{k+1/2}^s-\z^s}^2+\eta_k^s\braket{\Delta_k^s,\z-\z_{k+1/2}^s}\\\le&(1-\alpha_k^s) B(\z,\z_{k}^s)-B(\z,\z_{k+1}^s)+\brac{\sum_{k=1}^{K_{s-1}}\alpha_{k-1}^{s-1}}^{-1}\alpha_k^s\sum_{j=1}^{K_{s-1}} \alpha_{j-1}^{s-1}B(\z,\z_j^{s-1})\\&+\eta_k^s\braket{\Delta_k^s,\z-\z_{k+1/2}^s}+\frac{(\eta_k^sL_z)^2-\alpha_k^s}{2}\norm{\z_{k+1/2}^s-\z^s}^2.
        \end{aligned}\label{eq:intermediate-result}
    \end{equation}
    Recall the definition of Lyapunov function in~\cref{def:lyapunov} and periodically decaying sequence in~\cref{def:periodically-decaying-sequence}. Together with the fact that $\z_{K_s}^s=\z_0^{s+1}$, we have
    \begin{equation}\begin{aligned}
        &\sumkzs\sqbrac{(1-\alpha_k^s) B(\z,\z_{k}^s)-B(\z,\z_{k+1}^s)+\brac{\sum_{k=1}^{K_{s-1}}\alpha_{k-1}^{s-1}}^{-1}\alpha_k^s\sum_{j=1}^{K_{s-1}} \alpha_{j-1}^{s-1}B(\z,\z_j^{s-1})}\\=&\sumkzs (1-\alpha_k^s)\sqbrac{B(\z,\z_{k}^s)-B(\z,\z_{k+1}^s)}-\sumkzs\alpha_k^sB(\z,\z_{k+1}^s)+\frac{\sumkzs\alpha_k^s}{\sum_{k=1}^{K_{s-1}}\alpha_{k-1}^{s-1}}\sum_{j=1}^{K_{s-1}} \alpha_{j-1}^{s-1}B(\z,\z_j^{s-1})\\=&(1-\alpha_0^s) B(\z,\z_0^s)-(1-\alpha_{K_s-1}^s) B(\z,\z_{0}^{s+1})-\sumks\alpha_{k-1}^sB(\z,\z_{k}^s)+\frac{\sumks\alpha_{k-1}^s}{\sum_{k=1}^{K_{s-1}}\alpha_{k-1}^{s-1}}\sum_{k=1}^{K_{s-1}} \alpha_{k-1}^{s-1}B(\z,\z_k^{s-1})\\\le&(1-\alpha_0^s) B(\z,\z_0^s)+\sum_{k=1}^{K_{s-1}} \alpha_{k-1}^{s-1}B(\z,\z_k^{s-1})-(1-\alpha_{0}^{s+1}) B(\z,\z_{0}^{s+1})-\sumks\alpha_{k-1}^sB(\z,\z_{k}^s)\\=&\Psi_s(\z)-\Psi_{s+1}(\z).
    \end{aligned}
    \end{equation}
    we complete the proof by summing both sides of~\eqref{eq:intermediate-result} by index $k$ from 0 to $K_s-1$ and use the above relation.
    % \begin{equation}
    %     \sumkzs \eta_k^s\braket{\nabla F(\z_{k+1/2}^s),\z_{k+1/2}^s-\z}\le \Psi_s(\z)-\Psi_{s+1}(\z)-\sumkz\sqbrac{\eta_k^s\braket{\Delta_k^s,\z_{k+1/2}^s-\z}+\frac{(1-\alpha_k^s)\delta}{2}\norm{\z_{k+1/2}^s-\z^s}^2}
    % \end{equation}
\end{proof}

\begin{lemma}
    Denote the filtration generated by our algorithm by $\mathcal{F}=\{\mathcal{F}_k^s\}_{k\in[K_s]^0,s\in[S]^0}$. Let $\eta_k^s=\frac{\sqrt{\alpha_k^s(1-\theta_k^s)}}{L_z},\ \theta_k^s\in(0.8,0.99)$. Then the following recurrence holds:
    % \begin{equation}\begin{aligned}
    %     &\E[\Psi_{s+1}(\z_*)\Big|\mathcal{F}_0^s,\cdots,\mathcal{F}_{K_s-1}^s]\\\le&\Psi_s(\z_*)-\frac{1}{2}\sumkzs\E\sqbrac{\alpha_k^s\theta_k^s\norm{\z_{k+1/2}^s-\z^s}^2\Big|\mathcal{F}_0^s,\cdots,\mathcal{F}_{K_s-1}^s}.
    % \end{aligned}\end{equation}
    \begin{equation}\begin{aligned}
        \E[\Psi_{s+1}(\z_*)\Big|\mathcal{F}_0^s,\cdots,\mathcal{F}_{K_s-1}^s]\le\Psi_s(\z_*)-\frac{1}{2}\sumkzs\E\sqbrac{\alpha_k^s\theta_k^s\norm{\z_{k+1/2}^s-\z^s}^2\Big|\mathcal{F}_0^s,\cdots,\mathcal{F}_{K_s-1}^s}.
    \end{aligned}\end{equation}
    \label{lem:lyapunov-recurrence}
\end{lemma}

\begin{proof}
    Since $\z_*=(\w_*;\q_*)$ is the solution to~\eqref{eq:empirical-gdro-wq}, then we have\begin{equation}
       F(\w_*,\q_{k+1/2}^s) \le F(\w_*,\q_*)\le F(\w_{k+1/2}^s,\q_*).
    \end{equation}
    Recall convexity assumption (\cref{asp:convexity}) and the linearity of $\q$, we have\begin{equation}
        \begin{aligned}
            &F(\w_*,\q_{k+1/2}^s)\ge F(\w_{k+1/2}^s,\q_{k+1/2}^s)+\braket{\nabla_\w F(\w_{k+1/2}^s,\q_{k+1/2}^s),\w_*-\w_{k+1/2}^s},\\
            &F(\w_{k+1/2}^s,\q_*)=F(\w_{k+1/2}^s,\q_{k+1/2}^s)+\braket{\nabla_\q F(\w_{k+1/2}^s,\q_{k+1/2}^s),\q_*-\q_{k+1/2}^s}.\\
            % \Longrightarrow&
        \end{aligned}
    \end{equation}
    Therefore,
    \begin{equation}
    \begin{aligned}
        &\braket{\nabla F(\z_{k+1/2}^s),\z_{k+1/2}^s-\z_*}\\=&\braket{\nabla_\w F(\w_{k+1/2}^s,\q_{k+1/2}^s),\w_{k+1/2}^s-\w_*}-\braket{\nabla_\q F(\w_{k+1/2}^s,\q_{k+1/2}^s),\q_{k+1/2}^s-\q_*}\\\ge& F(\w_{k+1/2}^s,\q_{k+1/2}^s)-F(\w_*,\q_{k+1/2}^s)+F(\w_{k+1/2}^s,\q_*)-F(\w_{k+1/2}^s,\q_{k+1/2}^s)\\=&F(\w_{k+1/2}^s,\q_*)-F(\w_*,\q_{k+1/2}^s) \ge0.
    \end{aligned}
    \end{equation}
    Then plugging the above inequality to the LHS of~\cref{lem:s-epoch}, we have
    \begin{equation}
        0\le \Psi_s(\z_*)-\Psi_{s+1}(\z_*)+\sumkzs\sqbrac{\eta_k^s\braket{\Delta_k^s,\z-\z_{k+1/2}^s}+\frac{(\eta_k^sL_z)^2-\alpha_k^s}{2}\norm{\z_{k+1/2}^s-\z^s}^2}.
    \end{equation} 
    For any \textit{fixed} $\z\in\Z$, $\Delta_k^s$ is conditional independent from $\z_{k+1/2}^s-\z$. By the tower rule of expectation, we have
\begin{equation}\begin{aligned}
    &\E\sqbrac{\sumkzs\eta_k^s\braket{\Delta_k^s,\z_{k+1/2}^s-\z_*}\Big |\mathcal{F}_0^s,\cdots,\mathcal{F}_{K_s-1}^s}\\=&\sumkzs\E\sqbrac{\eta_k^s\braket{\E\sqbrac{\Delta_k^s|\mathcal{F}_k^s},\z_{k+1/2}^s-\z_*}\Big |\mathcal{F}_{k+1}^s,\cdots,\mathcal{F}_{K_s-1}^s}=0.
    \end{aligned}\end{equation} 
    Notice that $\E\sqbrac{\Psi_s(\z_*)|\mathcal{F}_0^s,\cdots,\mathcal{F}_{K_s-1}^s}=\Psi_s(\z_*)$. By~\cref{lem:s-epoch} and a simple rearrangement we can get the result.
\end{proof}

\begin{corollary} Under the conditions of~\cref{lem:lyapunov-recurrence}, we have 
\begin{equation}
\sum_{s=0}^\infty\sumkzs\E\sqbrac{\alpha_k^s\theta_k^s\norm{\z_{k+1/2}^s-\z^s}^2}\le2\Psi_0(\z_*)\le2\max_{\z\in\Z}\Psi_0(\z)
\end{equation}
\label{cor:grad-variance-bound}
\end{corollary}

\begin{proof}
    Summing the inequality in~\cref{lem:lyapunov-recurrence}, noticing the non-negativity of $\Psi_s(\z)$ together with the tower rule suffices to prove this corollary.
\end{proof}

\begin{lemma}
With $\eta_k^s$ set in~\cref{lem:lyapunov-recurrence}, we have
% \begin{equation}\begin{aligned}
% \brac{\sums\sumkzs\eta_k^s}\epsilon(\z_S)\le&\max_{\z\in\Z}\Psi_0(\z)+\max_{\z\in\Z}\sums\sumkzs\eta_k^s\braket{\Delta_k^s,\z-\z_{k+1/2}^s}\\&-\sums\sumkzs\frac{\alpha_k^s\theta_k^s}{2}\norm{\z_{k+1/2}^s-\z^s}^2.
% \end{aligned}\label{eq:gap-bound-0}\end{equation}
\begin{equation}\begin{aligned}
\brac{\sums\sumkzs\eta_k^s}\epsilon(\z_S)\le\max_{\z\in\Z}\Psi_0(\z)+\max_{\z\in\Z}\sums\sumkzs\eta_k^s\braket{\Delta_k^s,\z-\z_{k+1/2}^s}-\sums\sumkzs\frac{\alpha_k^s\theta_k^s}{2}\norm{\z_{k+1/2}^s-\z^s}^2.
\end{aligned}\label{eq:gap-bound-0}\end{equation}
\label{lem:gap-bound-0}
\end{lemma}

\begin{proof}
   The following result is a direct derivation from the~\cref{asp:convexity} and~\cref{lem:s-epoch}:
    % \begin{equation}
    %     \begin{aligned}
    %         &\brac{\sums\sumkzs\eta_k^s}\epsilon(\z_S)\\\le& \sqbrac{\max_{\q\in\Delta_m}\sums\sumkzs \eta_k^sF(\w_{k+1/2}^s,\q)-\min_{\w\in\W}\sums\sumkzs \eta_k^sF(\w,\q_{k+1/2}^s)}\\\le&\max_{\z\in\Z}\sums\sumkzs\eta_k^s\braket{\nabla F(\z_{k+1/2}^s),\z_{k+1/2}^s-\z}\\\le&\max_{\z\in\Z}\Psi_0(\z)+\max_{\z\in\Z}\sums\sumkzs\braket{\eta_k^s\Delta_k^s,\z-\z_{k+1/2}^s}-\sums\sumkzs\frac{\alpha_k^s\theta_k^s}{2}\norm{\z_{k+1/2}^s-\z^s}^2.
    %     \end{aligned}
    % \end{equation}
    \begin{equation}
        \begin{aligned}
            \brac{\sums\sumkzs\eta_k^s}\epsilon(\z_S)\le& \sqbrac{\max_{\q\in\Delta_m}\sums\sumkzs \eta_k^sF(\w_{k+1/2}^s,\q)-\min_{\w\in\W}\sums\sumkzs \eta_k^sF(\w,\q_{k+1/2}^s)}\\\le&\max_{\z\in\Z}\sums\sumkzs\eta_k^s\braket{\nabla F(\z_{k+1/2}^s),\z_{k+1/2}^s-\z}\\\le&\max_{\z\in\Z}\Psi_0(\z)+\max_{\z\in\Z}\sums\sumkzs\braket{\eta_k^s\Delta_k^s,\z-\z_{k+1/2}^s}-\sums\sumkzs\frac{\alpha_k^s\theta_k^s}{2}\norm{\z_{k+1/2}^s-\z^s}^2.
        \end{aligned}
    \end{equation}
\end{proof}

\subsection{Convergence and Complexity\label{app:sec:expectation-bound}}

\begin{lemma}
    Under the conditions in~\cref{lem:lyapunov-recurrence}, we have\begin{equation}
        \E\sqbrac{\epsilon(\z_S)}\le L_z\brac{\sums\sumkzs\sqrt{\alpha_k^s(1-\theta_k^s)}}^{-1}\sqbrac{1+\maxz\Psi_0(\z)}.
    \end{equation}
    \label{lem:raw-expectation-bound}
\end{lemma}

\begin{proof}
It's obvious to note that $\sums\sumkzs\eta_k^s\braket{\Delta_k^s,\z-\z_{k+1/2}^s}$ is a martingale difference sequence for any \emph{fixed} $\z\in\Z$. The existence of the maximum operation on the RHS of~\eqref{eq:gap-bound-0} deprives the inner product $\braket{\Delta_k^s,\z-\z_{k+1/2}^s}$ from being a martingale difference sequence. We need to apply a classical technique called ``ghost iterate''~\citep{nemirovski2009robust} to switch the order of maximization and expectation, and thus eliminating the dependency for $\z$. Image there is an online algorithm performing stochastic mirror descent (SMD):
    \begin{equation}
    \y_{k+1}^s=\argmin_{\y\in\Z}\{\braket{-\eta_k^s\Delta_k^s,\y-\y_k^s}+B(\y,\y_k^s)\},\ \y_0^{s+1}=\y_{K_s}^s\quad\forall\ s\in[S]^0,k\in[K_s]^0.\label{eq:virtual-sequence}
\end{equation}
    Also, we define $\y_0^0=\z_0^0=\argmin_{\z\in\Z}\psi(\z)$. According to~\citet[Lemma~6.1]{nemirovski2009robust}, we have for any $\z\in\Z$:
    \begin{equation}
        \sums\sumkzs\braket{\eta_k^s\Delta_k^s,\z-\y_{k}^s}\le B(\z,\z_0^0)+\frac{1}{2}\sums\sumkzs(\eta_k^s)^2\norm{\Delta_k^s}^2_*.\label{eq:ghost-iterate-bound}
    \end{equation}
    Now that we have decoupled the dependency, it's safe for us to assert that $\sums\sumkzs\braket{\eta_k^s\Delta_k^s,\y_{k}^s-\z_{k+1/2}^s}$ is a martingale difference sequence, since $\y_{k}^s-\z_{k+1/2}^s$ is conditionally independent of $\Delta_k^s$. 
    
  Firstly, we show that $\Delta_k^s$ is uniformly bounded above:
    \begin{equation}
    \begin{aligned}
        \norm{\Delta_k^s}_*&=\norm{\nabla F(\z_{k+1/2}^s;\xi_k^s)-\nabla F(\z^s;\xi_k^s)+\nabla F(\z^s)-\nabla F(\z_{k+1/2}^s)}_*\\&\le \norm{\nabla F(\z_{k+1/2}^s;\xi_k^s)-\nabla F(\z^s;\xi_k^s)}_*+\norm{\E\sqbrac{\nabla F(\z_{k+1/2}^s;\xi_k^s)-\nabla F(\z^s;\xi_k^s)}}_*\\&\le\norm{\nabla F(\z_{k+1/2}^s;\xi_k^s)-\nabla F(\z^s;\xi_k^s)}_*+\E\sqbrac{\norm{\nabla F(\z_{k+1/2}^s;\xi_k^s)-\nabla F(\z^s;\xi_k^s)}_*}\\&\le L_z\norm{\z_{k+1/2}^s-\z^s}+\E\sqbrac{L_z\norm{\z_{k+1/2}^s-\z^s}}\\&\le 2\sqrt{2}L_z+\E\sqbrac{2\sqrt{2}L_z}\\&=4\sqrt{2}L_z.
    \end{aligned}\label{eq:Delta_k^s-bound}
    \end{equation}
    The above inequality is ensured by the continuity introduced in~\cref{lem:Lipschitz}. 
    % Then we specify our choice of alterable learning rate $\eta_k^s$:
    % % \begin{equation}
    % %     \begin{aligned}
    % % |V_k^s|&\le\eta_k^s\norm{\Delta_k^s}_*\norm{\y_{k}^s-\z_{k+1/2}^s}\le 2\sqrt{2}\eta_k^s\norm{\Delta_k^s}_*\\&\le 16\eta_k^sL_z=16\sqrt{\alpha_k^s(1-\theta_k^s)}\le \frac{16}{\sqrt{5}}.
    % %     \end{aligned}\label{eq:Vks-martingale-bound}
    % % \end{equation}
    % \begin{equation}
    %     \begin{aligned}
    % |V_k^s|\le\eta_k^s\norm{\Delta_k^s}_*\norm{\y_{k}^s-\z_{k+1/2}^s}\le 2\sqrt{2}\eta_k^s\norm{\Delta_k^s}_*\le 16\eta_k^sL_z=16\sqrt{\alpha_k^s(1-\theta_k^s)}\le \frac{16}{\sqrt{5\bar n}}.
    %     \end{aligned}\label{eq:Vks-martingale-bound}
    % \end{equation}
    
    Then we define $V_k^s=\braket{\eta_k^s\Delta_k^s,\y_{k}^s-\z_{k+1/2}^s}$. The following holds:
   % ~\cref{lem:raw-expectation-bound} is analogous to~\cref{lem:raw-probability-bound}. We use the same technique as in~\cref{app:sec:high-probability-bound} to decouple the dependency and further construct the martingale difference sequence. Based on the previous result~\eqref{eq:virtual-sequence}, we have
\begin{equation}
    \begin{aligned}
&\E\sqbrac{\max_{\z\in\Z}\sums\sumkzs\eta_k^s\braket{\Delta_k^s,\z-\z_{k+1/2}^s}}\\=&\E\sqbrac{\maxz\sums\sumkzs\braket{\eta_k^s\Delta_k^s,\z-\y_{k}^s}}+\E\sqbrac{\sums\sumkzs\underbrace{\braket{\eta_k^s\Delta_k^s,\y_{k}^s-\z_{k+1/2}^s}}_{V_k^s}}\\\overset{\eqref{eq:ghost-iterate-bound}}{\le}&\maxz B(\z,\z_0^0)+\sums\sumkzs\E\sqbrac{\frac{(\eta_k^s)^2}{2}\norm{\Delta_k^s}_*^2}+\sums\sumkzs\E\sqbrac{\E[V_k^s|\mathcal{F}_k^s]}\\\overset{\eqref{eq:Delta_k^s-bound}}{\le}&1+\sums\sumkzs\E\sqbrac{2(\eta_k^sL_z)^2\norm{\z_{k+1/2}^s-\z^s}^2}\\=&1+\sums\sumkzs\E\sqbrac{2\alpha_k^s(1-\theta_k^s)\norm{\z_{k+1/2}^s-\z^s}^2}.
        \end{aligned}
    \end{equation}
    Combining it with~\cref{lem:gap-bound-0}:
    \begin{equation}
        \begin{aligned}            &\brac{\sums\sumkzs\eta_k^s}\E\sqbrac{\epsilon(\z_S)}\\\le&\max_{\z\in\Z}\Psi_0(\z)+\E\sqbrac{\max_{\z\in\Z}\sums\sumkzs\eta_k^s\braket{\Delta_k^s,\z-\z_{k+1/2}^s}}-\sums\sumkzs\E\sqbrac{\frac{\alpha_k^s\theta_k^s}{2}\norm{\z_{k+1/2}^s-\z^s}^2}\\\le&\maxz\Psi_0(\z)+1+\sums\sumkzs\E\sqbrac{\alpha_k^s(2-\frac{5}{2}\theta_k^s)\norm{\z_{k+1/2}^s-\z^s}^2}\\\le&1+\maxz\Psi_0(\z),
        \end{aligned}
    \end{equation}
    which concludes our proof by dividing $\sums\sumkzs\eta_k^s$ to both sides of the above inequality.
\end{proof}

\begin{proof}[Proof of~\cref{thm:gdro-convergence}]
First, we show that $\maxz\Psi_0(\z)$ is bounded under the given conditions:
\begin{equation}
    \begin{aligned}
        \maxz\Psi_0(\z)&=\maxz\cbrac{(1-\alpha_0^{0})B(\z,\z_0^{0})+\sum_{k=1}^{K_{-1}}\alpha_{k-1}^{-1}B(\z,\z_k^{-1})}\\&=\maxz\cbrac{(1-\alpha_0^{0})B(\z,\z_0)+\sum_{k=1}^{K_{-1}}\alpha_{k-1}^{-1}B(\z,\z_0)}\\&\le1-\alpha_0^{0}+\sum_{k=1}^{K_{-1}}\alpha_{k-1}^{-1}\le 2.
    \end{aligned}\label{eq:psi_0z-bound}
    \end{equation}
Under the given parameters, we have
    \begin{equation}
          \E\sqbrac{\epsilon(\z_S)}\le L_z\brac{\sums\sumkzs\sqrt{\frac{(1-\theta_k^s)}{K}}}^{-1}\sqbrac{1+\maxz\Psi_0(\z)}
        \overset{\eqref{eq:psi_0z-bound}}{\le}\frac{30L_z}{S\sqrt{K}}.\label{eq:intermediate-proof-expectation-bound}
    \end{equation}
    From~\cref{lem:Lipschitz} we know that $L_z=\OO(\sqrt{\ln m})$, which concludes our proof. 
\end{proof}

\begin{proof}[Proof of~\cref{cor:gdro-complexity}]
    The inner loop of~\cref{alg:gdro} consumes $\OO(m)$ computations per iteration. For outer loop $s$, the full gradient $\nabla F(\z^s)$ is calculated, requiring $\OO(m\bar n)$ computations. So the algorithms consume $\OO\brac{mKS+m\bar nS}$ in total. Aiming to set the two terms at the same order, we choose $K=\Theta(\bar n)$. As a consequence, $S=\OO\brac{\frac{L_z}{\varepsilon\sqrt{\bar n}}}$. Plugging this into $\OO\brac{mKS+m\bar nS}$ we derive an $\OO\brac{\frac{L_zm\sqrt{\bar n}}{\varepsilon}}$ computation complexity. With $L_z=\OO(\sqrt{\ln{m}})$ taken into consideration, the total computation complexity to reach $\varepsilon$-accuracy is $\OO\brac{\frac{m\sqrt{\bar n\ln{m}}}{\varepsilon}}$.
\end{proof}

% \subsection{Almost Surely Convergence}
\newpage
\section{Analysis for Empirical MERO\label{app:sec:mero-analysis}}
We present the omitted proofs for empirical MERO in this section. 

\subsection{Optimization Error}

The following proof verifies that the optimization error for the approximated problem is under control, which is proved to be essential in our two-stage schema for empirical MERO.

\begin{proof}[Proof of~\cref{lem:opt-error}]
For any $\z=(\w;\q)\in\Z$, we have\begin{equation}
    \left|\underline{F}(\z)-\underline{\hat{F}}(\z)\right|=\left|\summ\q_i\sqbrac{\underline{R}_i(\w)-\underline{\hat{R}}_i(\w)}\right|=\left|\summ\q_i\sqbrac{\hat{R}_i^*-R_i^*}\right|\le\max_{i\in[m]}\left\{\hat{R}_i^*-R_i^*\right\}.\label{eq:F_underline-bound}
\end{equation}
For convenience, we denote $\tilde\w=\argmin_{\w\in\W}\underline{F}(\w,\bar\q)$ and $\tilde\q=\argmin_{\q\in\Delta_m}\underline{F}(\bar\w,\q)$. Therefore, we can complete our proof by simple algebra.
    \begin{equation}
        \begin{aligned}
            \underline{\epsilon}(\bar\z)&=\underline{F}(\bar\w,\tilde\q)-\underline{F}(\tilde\w,\bar\q)\overset{\eqref{eq:F_underline-bound}}{\le}\underline{\hat F}(\bar\w,\tilde\q)-\underline{\hat F}(\tilde\w,\bar\q)+2\max_{i\in[m]}\left\{\hat{R}_i^*-R_i^*\right\}\\&\le \max_{\q\in\Delta_m}\underline{\hat F}(\bar\w,\q)-\min_{\w\in\W}\underline{\hat F}(\w,\bar\q)+2\max_{i\in[m]}\left\{\hat{R}_i^*-R_i^*\right\}\\&=\underline{\hat\epsilon}(\bar\z)+2\max_{i\in[m]}\left\{\hat{R}_i^*-R_i^*\right\}.
        \end{aligned}
    \end{equation}
\end{proof}

\subsection{Stage 1: Excess Empirical Risk Convergence}
At the beginning of this section, we present a useful lemma to bridge the variance-reduced property of ALEG and the martingale difference sequence.
\begin{lemma}
    (Bernstein’s Inequality for Martingales~\citep{cesa2006prediction}) Let $\{V_t\}_{t=1}^T$ be a martingale difference sequence with respect to the filtration $\mathcal{F}=\{\mathcal{F}_t\}_{t=1}^T$ bounded above by $V$, i.e.~$|V_t|\le V$. If the sum of the conditional variances is bounded, i.e.~$\sumt\E[V_t^2|\mathcal{F}_t]\le \sigma^2$, then for any $\delta\in(0,1]$,
    \begin{equation}
        \mathbb P\brac{\sumt V_t>\sigma\sqrt{2\ln{\frac{1}{\delta}}}+\frac{2}{3}V\ln{\frac{1}{\delta}}}\le \delta.
    \end{equation}\label{lem:Bernstein’s Inequality for Martingales}
\end{lemma}

Next, we present the proof of the high probability bound in~\cref{thm:excess-risk-convergence}.

\begin{proof}[Proof of~\cref{thm:excess-risk-convergence}]
    At the beginning, we briefly discuss how ALEG can be used as an ERM oracle. Under the circumstance of $m=1$, we notice that $\Delta_m$ reduces to a singleton. The original empirical GDRO problem~\eqref{eq:empirical-gdro-wq} can be rewritten as \begin{equation}
        \min_{\w\in\W}\max_{\q\in\Delta_1}\left\{F(\w,\q)=R_i(\w)\right\}\Longleftrightarrow\min_{\w\in\W}\cbrac{R_i(\w)}.\label{eq:reduced-gdro-per-group-risk}
    \end{equation}
    As a consequence, the merged gradient w.r.t.~$\q$ vanishes. In this case, when we talk about the smoothness of $F(\w,\q)$, we are actually focusing on the smoothness of $R_i(\w)$. We can conclude from~\cref{asp:smooth-Lipschitz} that $R_i(\cdot)$ is $L$-smooth. Moreover, the duality gap for the output of~\cref{alg:gdro} $\bar\z_i=(\bar\w_i,\bar\q_i)$ of~\eqref{eq:reduced-gdro-per-group-risk} satisfies:
    \begin{equation}
        \epsilon(\bar\z_i)=R_i(\bar\w_i)-\min_{\w\in\W}R_i(\w)=\hat{R}_i^*-R_i^*,\quad\forall\ i\in[m],
    \end{equation}
    % By replacing $L_z$ with $L$, $K$ with $\bar n$ in~\cref{thm:gdro-convergence}, we derive the following expectation bound based on the theoretical guarantee~\eqref{eq:expectation-bound} for general empirical GDRO problem.
    % \begin{equation}
    %     \E\sqbrac{\hat{R}_i^*-R_i^*}\overset{\eqref{eq:intermediate-proof-expectation-bound}}{\le}\frac{30L}{\left\lceil\frac{T}{\sqrt{\bar n}}\right\rceil\sqrt{\bar n}}\le\frac{30L}{T}=\OO\brac{\frac{1}{T}}.\label{eq:mero-proof-expectation-bound}
    % \end{equation}
    % For every $i\in[m]$,~\eqref{eq:mero-proof-expectation-bound} naturally holds. This finishes our proof for the expectation bound~\eqref{eq:excess-risk-expectation-bound}.
    which is exactly the excess empirical risk for group $i$.
    Recall the ``ghost iterate'' technique demonstrated in~\cref{app:sec:expectation-bound}, we define $V_k^s=\braket{\eta_k^s\Delta_k^s,\y_{k}^s-\z_{k+1/2}^s}$. Firstly, we show that $V_k^s$ is uniformly bounded above:
        % \begin{equation}
    %     \begin{aligned}
    % |V_k^s|&\le\eta_k^s\norm{\Delta_k^s}_*\norm{\y_{k}^s-\z_{k+1/2}^s}\le 2\sqrt{2}\eta_k^s\norm{\Delta_k^s}_*\\&\le 16\eta_k^sL_z=16\sqrt{\alpha_k^s(1-\theta_k^s)}\le \frac{16}{\sqrt{5}}.
    %     \end{aligned}\label{eq:Vks-martingale-bound}
    % \end{equation}
    \begin{equation}
        \begin{aligned}
    |V_k^s|\le\eta_k^s\norm{\Delta_k^s}_*\norm{\y_{k}^s-\z_{k+1/2}^s}\le 2\sqrt{2}\eta_k^s\norm{\Delta_k^s}_*\le 16\eta_k^sL=16\sqrt{\alpha_k^s(1-\theta_k^s)}\le \frac{16}{\sqrt{5\bar n}}.
        \end{aligned}\label{eq:Vks-martingale-bound}
    \end{equation}
    
   Secondly, we bound the sum of conditional variance of $\{V_k^s\}_{k\in[K_s]^0,s\in[S]^0}$. By~\eqref{eq:Delta_k^s-bound}, the definition of $\theta_k^s$ and~\cref{cor:grad-variance-bound},
    \begin{equation}
    \begin{aligned}
\sums\sumkzs\E\sqbrac{(V_k^s)^2|\mathcal{F}_k^s}&\le\sums\sumkzs\E\sqbrac{8(\eta_k^s)^2\norm{\Delta_k^s}_*^2\Big|\mathcal{F}_k^s}\\&\le \sums\sumkzs\E\sqbrac{32(\eta_k^sL)^2\norm{\z_{k+1/2}^s-\z^s}^2\Big|\mathcal{F}_k^s}\\&= \sums\sumkzs32\E\sqbrac{\alpha_k^s(1-\theta_k^s)\norm{\z_{k+1/2}^s-\z^s}^2\Big|\mathcal{F}_k^s}\\&\le \sums\sumkzs8\E\sqbrac{\alpha_k^s\theta_k^s\norm{\z_{k+1/2}^s-\z^s}^2\Big|\mathcal{F}_k^s}\\&
% \le \sums\sumkzs8\E\sqbrac{\alpha_k^s\theta_k^s\norm{\z_{k+1/2}^s-\z^s}^2}\\&
\le 32\max_{\z\in\Z}\Psi_0(\z).
    \end{aligned}\label{eq:conditional-variance-bound}
    \end{equation}

    Finally, we can use~\cref{lem:Bernstein’s Inequality for Martingales} together with the union bound to conclude that with probability at least $1-\delta$
    \begin{equation}
        \sums\sumkzs V_k^s\le 32\max_{\z\in\Z}\Psi_0(\z)\sqrt{2\ln{\frac{S}{\delta}}}+\frac{32}{3\sqrt{5\bar n}}\ln{\frac{S}{\delta}}.
        \label{eq:martingale-high-probability-bound}
    \end{equation}
    From~\cref{lem:gap-bound-0} and the previous result~\eqref{eq:ghost-iterate-bound}, along with the boundness of Bregman function and the definition of $\theta_k^s$, we have \begin{equation}
        \begin{aligned}
\epsilon(\bar\z_i)
% \epsilon(\z_S)
\le&\brac{\sums\sumkzs\eta_k^s}^{-1}\left[\max_{\z\in\Z}\Psi_0(\z)+\max_{\z\in\Z}\sums\sumkzs\braket{\eta_k^s\Delta_k^s,\z-\y_{k}^s}\right.\\&+\left.\sums\sumkzs\braket{\eta_k^s\Delta_k^s,\z_{k+1/2}^s-\y_{k}^s}-\sums\sumkzs\frac{\alpha_k^s\theta_k^s}{2}\norm{\z_{k+1/2}^s-\z^s}^2\right]\\\le& \brac{\sums\sumkzs\eta_k^s}^{-1}\left[\max_{\z\in\Z}\Psi_0(\z)+\max_{\z\in\Z}B(\z,\z_0^0)+\frac{1}{2}\sums\sumkzs(\eta_k^s)^2\norm{\Delta_k^s}^2_*\right.\\&+\left.\sums\sumkzs V_k^s-\sums\sumkzs\frac{\alpha_k^s\theta_k^s}{2}\norm{\z_{k+1/2}^s-\z^s}^2\right]\\\le&\brac{\sums\sumkzs\eta_k^s}^{-1}\left[\max_{\z\in\Z}\Psi_0(\z)+1+\sums\sumkzs\frac{2\alpha_k^s(1-\theta_k^s)}{L^2}L^2\norm{\z_{k+1/2}^s-\z^s}^2\right.\\&\left.-\sums\sumkzs\frac{\alpha_k^s\theta_k^s}{2}\norm{\z_{k+1/2}^s-\z^s}^2+\sums\sumkzs V_k^s\right]\\\le& L\brac{\sums\sumkzs\sqrt{\alpha_k^s(1-\theta_k^s)}}^{-1}\sqbrac{1+\max_{\z\in\Z}\Psi_0(\z)+\sums\sumkzs V_k^s}.
        \end{aligned}\label{eq:combinebounds}
    \end{equation}
    Combining~\eqref{eq:martingale-high-probability-bound} and~\eqref{eq:combinebounds} we derive that with probability at least $1-\delta$,
    \begin{equation}
    \begin{aligned}
        \hat{R}_i^*-R_i^*\le& L\brac{\sums\sumkzs\sqrt{\alpha_k^s(1-\theta_k^s)}}^{-1}\sqbrac{1+\max_{\z\in\Z}\Psi_0(\z)+32\max_{\z\in\Z}\Psi_0(\z)\sqrt{2\ln{\frac{S}{\delta}}}+\frac{32}{3\sqrt{5\bar n}}\ln{\frac{S}{\delta}}}\\\le&\frac{10L}{S\sqrt{\bar n}}\brac{3+64\sqrt{2\ln{\frac{S}{\delta}}}+\frac{32}{3\sqrt{5\bar n}}\ln{\frac{S}{\delta}}},
    \end{aligned}     
    \end{equation}
    where the last inequality uses~\eqref{eq:psi_0z-bound}. The above relation needs to hold for every group $i$, which necessitates the usage of the union bound tool. We can deduce that with probability at least $1-\delta$,
    \begin{equation}
        \max_{i\in[m]}\cbrac{\hat{R}_i^*-R_i^*}\le \frac{10L}{S\sqrt{\bar n}}\brac{3+64\sqrt{2\ln{\frac{mS}{\delta}}}+\frac{32}{3\sqrt{5\bar n}}\ln{\frac{mS}{\delta}}}.
    \end{equation}

    Recall the relation between $S$ and $T$ in~\cref{thm:excess-risk-convergence}, we derive that
    \begin{equation}
        \max_{i\in[m]}\cbrac{\hat{R}_i^*-R_i^*}\le\OO\brac{\frac{1}{T}\brac{\sqrt{\ln{\frac{mT}{\sqrt{\bar n}\delta}}}+\frac{1}{\sqrt{\bar n}}\ln{\frac{mT}{\sqrt{\bar n}\delta}}}},\label{eq:max-excess-risk-prob-bound}
    \end{equation}
    which is equivalent to~\eqref{eq:excess-risk-probability-bound}.
    % To prove the probability bound~\eqref{eq:excess-risk-probability-bound}, we should modify the result in~\eqref{eq:probability-bound} and combine it with the union bound tool. Similarly to the proof of expectation bound, according to~\cref{thm:gdro-convergence}, the following holds with probability at least $1-\frac{\delta}{m}$:
    % \begin{equation}
    % \begin{aligned}
    %      \hat{R}_i^*-R_i^*&\overset{\eqref{eq:probability-bound-last-eq}}{\le}\frac{10L}{T}\brac{3+32\sqrt{2\ln{\frac{m}{\delta}}}+\frac{32}{3\sqrt{5\bar n}}\ln{\frac{m}{\delta}}}.
    % \end{aligned}\label{eq:excess-risk-prob-bound-proof}
    % \end{equation}
    % Note that the RHS of~\eqref{eq:excess-risk-prob-bound-proof} is irrelevant to $i$. Hence, we make use of union bound to deduce that with probability at least $1-\delta$,
    % \begin{equation}
    %     \hat{R}_i^*-R_i^*\overset{\eqref{eq:excess-risk-prob-bound-proof}}{\le}\OO\brac{\frac{1}{T}\sqbrac{\sqrt{\ln{\frac{m}{\delta}}}+\frac{1}{\sqrt{\bar n}}\ln{\frac{m}{\delta}}}},\quad\forall\ i\in[m].\label{eq:excess-risk-probability-bound-proof}
    % \end{equation}
\end{proof}

\subsection{Stage 2: Empirical GDRO Solver Convergence\label{app:sec:mero-solver-convergence}}

The second stage of~\cref{alg:mero} solves~\eqref{eq:empirical-mero-wq-surrogate} by~\cref{alg:gdro}. The main idea is to combine the convergence results in~\cref{sec:gdro-solution} with~\cref{lem:opt-error}.

\begin{proof}[Proof of~\cref{thm:mero-convergence}]

    The proof of~\cref{thm:excess-risk-convergence} actually provides a high probability bound for empirical GDRO. By substituting $L$ with $L_z$ in the previous derivations, the following holds with probability at least $1-\frac{\delta}{2}$,
    % \begin{equation}
    %     \underline{\hat\epsilon}(\bar\z)\overset{\eqref{eq:probability-bound-last-eq}}{\le}\frac{10L_z}{T}\sqbrac{3+32\sqrt{2\ln{\frac{2}{\delta}}}+\frac{32}{3\sqrt{5\bar n}}\ln{\frac{2}{\delta}}}
    % \end{equation}
    \begin{equation}
        \underline{\hat\epsilon}(\bar\z)\le\frac{10L_z}{T}\brac{3+32\sqrt{2\ln{\frac{2T}{\sqrt{\bar n}\delta}}}+\frac{64}{3\sqrt{5\bar n}}\ln{\frac{2T}{\sqrt{\bar n}\delta}}}.
    \end{equation}
    Since we only need to provide a theoretical guarantee for a single approximate MERO problem, the union bound appeared in~\eqref{eq:max-excess-risk-prob-bound} is unnecessary.
    % \begin{equation}
    %     \max_{i\in[m]}\left\{\hat{R}_i^*-R_i^*\right\}\overset{\eqref{eq:excess-risk-probability-bound-proof}}{\le}\frac{10L}{T}\sqbrac{3+32\sqrt{2\ln{\frac{2m}{\delta}}}+\frac{32}{3\sqrt{5\bar n}}\ln{\frac{2m}{\delta}}}.
    % \end{equation}
    Based on~\eqref{eq:max-excess-risk-prob-bound}, the following holds with probability with probability at least $1-\frac{\delta}{2}$,
    \begin{equation}
        \max_{i\in[m]}\cbrac{\hat{R}_i^*-R_i^*}\le \frac{10L}{T}\brac{3+64\sqrt{2\ln{\frac{2mT}{\sqrt{\bar n}\delta}}}+\frac{32}{3\sqrt{5\bar n}}\ln{\frac{2mT}{\sqrt{\bar n}\delta}}}.
    \end{equation}
    Together, we immediately derive that with probability at least $1-\delta$,
    % \begin{equation}
    % \begin{aligned}
    %     \underline{\epsilon}(\bar\z)\overset{\eqref{eq:duality-gap-bound}}{\le}&\underline{\hat\epsilon}(\bar\z)+2\max_{i\in[m]}\left\{\hat{R}_i^*-R_i^*\right\}\\\le&\frac{10}{T}\left[3(L_z+2L)+32\sqrt{2}\brac{L_z\sqrt{\ln{\frac{2}{\delta}}}+2L\sqrt{\ln{\frac{2m}{\delta}}}}\right.\\&\left.+\frac{32}{3\sqrt{5\bar n}}\brac{L_z\ln{\frac{2}{\delta}}+2L\ln{\frac{2m}{\delta}}}\right]\\=& \OO\brac{\frac{1}{T}\sqbrac{\sqrt{\ln m\ln{\frac{1}{\delta}}}+\sqrt{\ln{\frac{m}{\delta}}}+\frac{1}{\sqrt{\bar n}}\brac{\sqrt{\ln m}\ln{\frac{1}{\delta}}+\ln{\frac{m}{\delta}}}}}\\\le& \OO\brac{\frac{1}{T}\sqbrac{\sqrt{\ln m\ln{\frac{1}{\delta}}}+\sqrt{\ln{\frac{m}{\delta}}}+\sqrt{\frac{\ln m}{\bar n}}\ln{\frac{1}{\delta}}}},
    % \end{aligned}
    % \end{equation}
    \begin{equation}
    \begin{aligned}
        \underline{\epsilon}(\bar\z)\overset{\eqref{eq:duality-gap-bound}}{\le}&\underline{\hat\epsilon}(\bar\z)+2\max_{i\in[m]}\left\{\hat{R}_i^*-R_i^*\right\}\\\le&\frac{10}{T}\left[3(L_z+2L)+64\sqrt{2}\brac{L_z\sqrt{\ln{\frac{2T}{\sqrt{\bar n}\delta}}}+2L\sqrt{\ln{\frac{2mT}{\sqrt{\bar n}\delta}}}}+\frac{32}{3\sqrt{5\bar n}}\brac{L_z\ln{\frac{2T}{\sqrt{\bar n}\delta}}+2L\ln{\frac{2m}{\sqrt{\bar n}\delta}}}\right]\\=& \OO\brac{\frac{1}{T}\brac{\sqrt{\ln m\ln{\frac{T}{\sqrt{\bar n}\delta}}}+\sqrt{\ln{\frac{mT}{\sqrt{\bar n}\delta}}}+\frac{1}{\sqrt{\bar n}}\brac{\sqrt{\ln m}\ln{\frac{T}{\sqrt{\bar n}\delta}}+\ln{\frac{mT}{\sqrt{\bar n}\delta}}}}}\\\le& \OO\brac{\frac{1}{T}\brac{\sqrt{\ln m\ln{\frac{T}{\sqrt{\bar n}\delta}}}+\sqrt{\ln{\frac{mT}{\sqrt{\bar n}\delta}}}+\sqrt{\frac{\ln m}{\bar n}}\ln{\frac{T}{\sqrt{\bar n}\delta}}}},
    \end{aligned}
    \end{equation}
    where the last inequality holds under the ordinary case where $m\le\OO(\bar n)$. 
\end{proof}

At the end of this section, we calculate the number of computations needed to run~\cref{alg:mero}.

\begin{proof}[Proof of~\cref{cor:mero-complexity}]
The proof of~\cref{cor:mero-complexity} is analogous to the proof of~\cref{cor:gdro-complexity}. According to~\cref{thm:mero-convergence}, we need a budget of $T=\OO\brac{\frac{\sqrt{\ln m}}{\varepsilon}}$ to reach $\varepsilon$-accuracy if logarithmic factors for $\frac{T}{\sqrt{\bar n}}$ are overlooked. This requires $S=\OO\brac{\frac{\sqrt{\ln m}}{\sqrt{\bar n}\varepsilon}}$ correspondingly. In the first stage of~\cref{alg:mero}, the computation complexity is $\OO\brac{\summ \bar nS+n_iS}=\OO\brac{m\bar nS}=\OO\brac{\frac{m\sqrt{\bar n\ln m}}{\varepsilon}}$. In the second stage of~\cref{alg:mero}, the computation complexity is $\OO\brac{mKS+m\bar nS}=\OO\brac{m\bar nS}=\OO\brac{\frac{m\sqrt{\bar n\ln m}}{\varepsilon}}$. As a consequence, adding the computation complexity from two stages together gives the computation complexity of $\OO\brac{\frac{m\sqrt{\bar n\ln m}}{\varepsilon}}$. To derive the final result, the neglected $\ln{\frac{T}{\sqrt{\bar n}}}$ needs to be taken into consideration. Therefore, the total complexity for~\cref{alg:mero} to reach $\varepsilon$-accuracy is $\OO\brac{\frac{m\sqrt{\bar n\ln m}}{\varepsilon}\ln{\frac{\sqrt{\ln m}}{\sqrt{\bar n}\varepsilon}}}$, which is also denoted by $\tilde{\OO}\brac{\frac{m\sqrt{\bar n\ln m}}{\varepsilon}}$.
\end{proof}
\newpage

\section{Revisiting MPVR as a Black Box\label{app:sec:compare-sampling}}

In this section, we illustrate the extension of MPVR~\citep{alacaoglu2022stochastic} to solve the empirical GDRO problem defined in~\eqref{eq:empirical-gdro-wq}. To run MPVR as a black box, we need to specify the problem formulation, the construction of stochastic gradients, and the Lipschitz continuity of the stochastic gradients. After these steps, we can derive a suboptimal complexity and then discuss how the proposed approach goes beyond the application of existing techniques.

\subsection{Problem Formulation}
Since MPVR only deals with one-level finite-sums, the objectives in the empirical GDRO problem should be rewritten as
\begin{equation}
    \min_{\w\in\W}\max_{\q\in\Delta_m}\left\{F(\w,\q):=\sum_{l=1}^{m\bar n}\frac{\q_{l_i}}{n_{l_i}}\ell(\w;\xi_{l})\right\}.\label{eq:one-level-mpvr}
\end{equation}
where\begin{equation}   \cbrac{\xi_l}_{l=1}^{m\bar n}:=\cbrac{\xi_{l_il_j}}_{l=1}^{m\bar n}=\cbrac{\xi_{ij}}_{j\in[n_i],i\in[m]}
\end{equation}
with $l_i,l_j$ determined by
\begin{equation}
    l_i\in[m]:\sum_{p=1}^{l_i-1}n_p<l\le\sum_{p=1}^{l_i}n_p\quad\text{and}\quad l_j=l-\sum_{p=1}^{l_i-1}n_p.
\end{equation}
The above relations show the obvious bijective mapping between one-level finite-sum index $l$ and two-level finite-sum index $(l_i,l_j)$.

From~\eqref{eq:one-level-mpvr} we can see that the number of finite-sum components for MPVR is $m\bar n$. In the next part, we construct the stochastic gradient, which comes from accessing one of the members of $m\bar n$ components.

\subsection{Construction of Stochastic Gradients}

The construction of the stochastic gradients is essential for variance-reduced methods. Such construction is also closely related to the sampling strategy. \citet{alacaoglu2022stochastic} introduce uniform sampling and importance sampling as two main stochastic oracles into their approach. The former strategy treats every member of $m\bar n$ equally and therefore generates a uniform distribution among all indices. The latter considers the continuity of each component, assigning each component a probability proportional to its Lipschitz constants.
\subsubsection{Uniform Sampling\label{app:sec:uniform-sampling}}
Denote by $\text{Unif}(\cdot)$ as the uniform distribution and $\nabla F(\z;\xi_k^s)$ as the stochastic gradient for the $k$-th inner loop, $s$-th outer loop. From the above discussions, the chosen index is generated by
\begin{equation}
    l\sim\text{Unif}(m\bar n).\label{eq:uniform-sampling}
\end{equation}
Consequently, the stochastic gradient can be constructed as follows
\begin{equation}
    \nabla F(\mathbf z;\xi_k^s):=\begin{pmatrix}
       \frac{m\bar n}{n_{l_i}}\mathbf q_{l_i}\nabla\ell(\mathbf w;\xi_{l})\\-\left[0,\cdots,\underbrace{\frac{m\bar n}{n_{l_i}}\ell(\mathbf w;\xi_{l})}_{l_i\text{-th\ element}},\cdots,0\right]^T
    \end{pmatrix}.\label{eq:uniform-sampling-stochastic-gradient}
\end{equation}

The following property is necessary for MPVR. The proof is deferred to~\cref{app:sec:omitted-proofs}.
\begin{lemma}
    (First-order and second-order moment information of the stochastic gradient from uniform sampling) With expectation taken over $\xi_k^s$, the following properties holds:
    \begin{enumerate}
        \item $\E\sqbrac{\nabla F(\z;\xi_k^s)}=\nabla F(\z)$;
        \item $\E\sqbrac{\norm{\nabla F(\z_1;\xi_k^s)-\nabla F(\z_2;\xi_k^s)}^2_*}\le L_{\mathbf u}^2\norm{\z_1-\z_2}^2,\forall\z_1,\z_2\in\Z;$
    \end{enumerate}
    where the Lipschitz constant $L_{\mathbf u}$ is defined as
    \begin{equation}
        L_{\mathbf u}:=2D_w\max\cbrac{\sqrt{2D^2_wL^2m\frac{\bar n}{ n_{\min}}+G^2m^2\ln m\frac{\bar n}{\bar n_h}},G\sqrt{2m\ln m\frac{\bar n}{n_{\min}}}}.\label{eq:Lu}
    \end{equation}\label{lem:sg-uniform-sampling}
    with $n_h$ being the harmonic average of the number of samples and $n_{\min}$ being the minimal number of samples amongst $m$ groups, i.e.,\begin{equation}
        \bar n_h:=\frac{m}{\summ\frac{1}{n_i}},\quad n_{\min}:=\min_{i\in[m]}n_i.
    \end{equation}
\end{lemma}

\subsubsection{Importance Sampling\label{app:sec:importance-sampling}}

At first glance, importance sampling might be the same as uniform sampling because the Lipschitz constant for each loss function is the same according to~\cref{asp:smooth-Lipschitz}. Whereas, from~\eqref{eq:one-level-mpvr} we can see that the finite-sum component has a coefficient of $\frac{1}{n_i}$. Hence, the Lipschitzness is shifted by a factor of $\frac{1}{n_i}$. Therefore, we assign the probability proportional to this coefficient, i.e., we endow a larger sampling probability for a group with more samples (larger $n_i$).

Based on the above discussions, we present the process of importance sampling as follows:
\begin{equation}
    l_i\sim\text{Unif}(m),\quad l_j|l_i\sim\text{Unif}(n_{l_i}).\label{eq:importance-sampling}
\end{equation}
Naturally, the following stochastic gradient is produced:
\begin{equation}
    \nabla F(\mathbf z;\xi_k^s):=\begin{pmatrix}
       m\mathbf q_{l_i}\nabla\ell(\mathbf w;\xi_{l})\\-\left[0,\cdots,\underbrace{m\ell(\mathbf w;\xi_{l})}_{l_i\text{-th\ element}},\cdots,0\right]^T
    \end{pmatrix},\label{eq:importance-sampling-stochastic-gradient}
\end{equation}

Similar to~\cref{app:sec:uniform-sampling}, we have the following lemma to quantify the first-order and second-order moments of the stochastic gradient in~\eqref{eq:importance-sampling-stochastic-gradient}.
\begin{lemma}
    (First-order and second-order moment information of the stochastic gradient from importance sampling) With expectation taken over $\xi_k^s$, the following properties holds:
    \begin{enumerate}
        \item $\E\sqbrac{\nabla F(\z;\xi_k^s)}=\nabla F(\z)$;
        \item $\E\sqbrac{\norm{\nabla F(\z_1;\xi_k^s)-\nabla F(\z_2;\xi_k^s)}^2_*}\le L_{\mathbf i}^2\norm{\z_1-\z_2}^2,\forall\z_1,\z_2\in\Z;$
    \end{enumerate}
    where the Lipschitz constant $L_{\mathbf i}$ is defined as
        \begin{equation}
L_{\mathbf i}:=2D_w\max\cbrac{\sqrt{2D^2_wL^2m+G^2m^2\ln{m}},G\sqrt{2m\ln{m}}}.\label{eq:Li}
\end{equation}\label{lem:sg-importance-sampling}
\end{lemma}

\begin{remark}
    Our Lipschitzness in~\cref{lem:Lipschitz} is different from the ones in~\cref{lem:sg-uniform-sampling,lem:sg-importance-sampling}. The definition of the Lipschitz continuity in~\citet{alacaoglu2022stochastic} is reflected by the second-order moment of the stochastic gradient. While ours is reflected by the square dual norm of the stochastic gradient, which makes our formulation stronger than MPVR. The high probability bound we provide justifies this formulation since the continuity w.r.t.~the second-order moment does not hold with probability 1.
\end{remark}

\begin{algorithm}[t]
   \caption{MPVR for Empirical GDRO}
   \label{alg:mpvr}
  {\bf Input}: 
  Risk function $\{R_i(\w)\}_{i\in[m]}$, epoch number $S$, iteration number $K$, learning rate $\eta$, weight $\alpha$.

\begin{algorithmic}[1]
   \STATE Initialize~starting~point~$\z_0=(\w_0;\q_0)=\argmin_{\z\in\Z}\psi(\z)$.
   \STATE For each $j\in[K]$, set $\z_j^{-1}=\z_0^0=\z_0$.
   \FOR{$s=0$ {\bfseries to} $S-1$}
   \FOR{$k=0$ {\bfseries to} $K-1$}
   \STATE $\z_{k+1/2}^s=\argmin_{\z\in\Z}\{\braket{\eta\nabla F(\z^s),\z}+\alpha B(\z,\bar\z^s)+(1-\alpha)B(\z,\z_k^s)\}.$
   \STATE \textbf{Option I:} Uniform Sampling. 
   \STATE \qquad Sample according to~\eqref{eq:uniform-sampling} and compute stochastic gradient according to~\eqref{eq:uniform-sampling-stochastic-gradient}.
   \STATE \textbf{Option II:} Importance Sampling. 
   \STATE \qquad Sample according to~\eqref{eq:importance-sampling} and compute stochastic gradient according to~\eqref{eq:importance-sampling-stochastic-gradient}.
   \STATE Compute stochastic gradient estimator $\g_k^s$ defined in~\eqref{eq:g_k^s}.
   \STATE $\z_{k+1}^s=\argmin_{\z\in\Z}\{\braket{\eta \g_k^s,\z}+\alpha B(\z,\bar\z^s)+(1-\alpha)B(\z,\z_k^s)\}.$
    \ENDFOR
    \STATE Compute full gradient $\nabla F(\z^{s})$ according to~\eqref{eq:z-fullgradient}.
    \STATE Compute snapshot point: $\z^{s}=\frac{1}{K}\sumkz \z_k^s$.
    \STATE Compute mirror snapshot point: $\nabla \psi(\bar\z^{s})=\frac{1}{K}\sumkz \nabla \psi(\z_k^s)$.
    \STATE $\z_0^{s+1}=\z_{K}^s$.
    \ENDFOR
    \STATE Return $\z_S=\frac{1}{SK}\sums\sumkz \z_{k+1/2}^s$.
\end{algorithmic}
\end{algorithm}

\subsection{Loose Complexity Bound}

With the preparations in previous sections, we can formally apply MPVR to solve empirical GDRO as shown in~\cref{alg:mpvr}. For completeness, we will present the following theoretical guarantee for~\cref{alg:mpvr}. The proofs are deferred to~\cref{app:sec:omitted-proofs}.

\begin{theorem}
    Under~\cref{asp:boundness,asp:smooth-Lipschitz,asp:convexity}, choose $0\le\alpha<1$, $0<\gamma<1$, and $\eta=\frac{\gamma\sqrt{1-\alpha}}{L_{\mathbf{c}}}$, with $L_{\mathbf{c}}=L_{\mathbf{u}}$ for uniform sampling and $L_{\mathbf{c}}=L_{\mathbf{i}}$ for importance sampling,~\cref{alg:mpvr} ensures that
    \begin{equation}
    \E\sqbrac{\epsilon(\z_S)}\le\OO\brac{\frac{L_{\mathbf{c}}}{S\sqrt{K}}}.\label{eq:mpvr-expectation-bound}
    \end{equation}
\label{thm:mpvr-convergence}
\end{theorem}

\begin{corollary}
    Under conditions in~\cref{thm:mpvr-convergence}, by setting $K=\Theta(m\bar n)$ and employing either uniform sampling or importance sampling, the computation complexity for~\cref{alg:mpvr} to reach $\varepsilon$-accuracy of~\eqref{eq:empirical-gdro-wq} is $\OO\brac{m\bar n+\frac{m\sqrt{m\bar n\ln{m}}}{\varepsilon}}$.\label{cor:mpvr-complexity}
\end{corollary}

\cref{cor:mpvr-complexity} tells us that when utilizing MPVR to solve empirical GDRO, the complexity is $\sqrt{m}$ worse than ALEG. Such discrepancy stems from the dependence on the Lipschitz constant $L_{\mathbf{c}}$. Both uniform sampling and importance sampling will produce a Lipschitz constant worse than us by a factor of $m$, as shown in both~\eqref{eq:Lu} and~\eqref{eq:Li}. One can easily discover that if the Lipschitz constant for MPVR were only $\sqrt{m}$ worse than us, then the total complexity would be the same. Unfortunately, this could not happen because the sampling pattern ignores the nested finite-sum structure of the original problem. Technically, this additional factor is inherently due to the property of the infinity norm (dual norm of $\ell_1$ norm), which scales up the factors from $m$ coordinates when added together. Mathematical details are explicitly explained in~\cref{lem:2order-moment-bound} and~\cref{remark:2order-moment-bound}.

\subsection{Omitted Proofs\label{app:sec:omitted-proofs}}

We present the omitted proofs for the lemmas and theorems in~\cref{app:sec:compare-sampling}.

\begin{proof}[Proof of~\cref{lem:sg-uniform-sampling}]

First, we study the first-order moment information of the stochastic gradient in~\eqref{eq:uniform-sampling-stochastic-gradient}. Considering the stochastic gradient w.r.t.~$\w$, it follows that
\begin{equation}
\begin{aligned}
\mathbb E[\nabla_{\mathbf w} F(\mathbf z;\xi_k^s)]=\mathbb E\left[\frac{m\bar n}{n_{l_i}}\mathbf q_{l_i}\nabla\ell(\mathbf w;\xi_{l})\right]=\sum_{l=1}^{m\bar n}\frac{1}{m\bar n}\cdot\frac{m\bar n}{n_{l_i}}\mathbf q_{l_i}\nabla\ell(\mathbf w;\xi_{l})=\nabla_{\mathbf w} F(\mathbf z).
\end{aligned}
\end{equation}
Then we consider the stochastic gradient w.r.t.~$\q$. Denote by $\mathbf{e}_i\in\R^m$ as the one-hot vector that has an $i$-th component 1 with the rest component being 0. It follows that
\begin{equation}
\begin{aligned}
\mathbb E[\nabla_{\mathbf q} F(\mathbf z;\xi_k^s)]=\mathbb E\left[\frac{m\bar n}{n_{l_i}}\ell(\w;\xi_l)\cdot \mathbf{e}_{l_i}\right]=\sum_{l=1}^{m\bar n}\frac{1}{m\bar n}\cdot\frac{m\bar n}{n_{l_i}}\ell(\w;\xi_l)\cdot \mathbf{e}_{l_i}=\nabla_{\mathbf q} F(\mathbf z).
\end{aligned}
\end{equation}
Secondly, we analyze the continuity of the second-order moment of the stochastic gradient in~\eqref{eq:uniform-sampling-stochastic-gradient}. According to~\cref{lem:2order-moment-bound}, we specify the following upper bound for the stochastic gradient w.r.t.~$\w$ with $C_i=\frac{m\bar n}{n_{l_i}}$:
% \begin{equation}
%     \begin{aligned}
%         &\E\sqbrac{\norm{\nabla_\w F(\z^+;\xi_{k}^s)-\nabla_\w F(\z;\xi_{k}^s)}_{w,*}^2}\\\le&2L^2\norm{\w^+-\w}_{w}^2\sum_{l_i=1}^m\frac{m\bar n}{n_{l_i}}\q^+_{l_i}+2G^2\sum_{l_i=1}^m\frac{m\bar n}{n_{l_i}}\brac{\q^+_{l_i}-\q_{l_i}}^2\\\le&2m\bar nL^2\norm{\w^+-\w}_{w}^2\sum_{l_i=1}^m\frac{\q^+_{l_i}}{n_{\min}}+2m\bar nG^2\sum_{l_i=1}^m\frac{\brac{\q^+_{l_i}-\q_{l_i}}^2}{n_{\min}}\\\le&2m\frac{\bar n}{n_{\min}}L^2\norm{\w^+-\w}_{w}^2+2m\frac{\bar n}{n_{\min}}G^2\norm{\q^+-\q}_1^2.
%     \end{aligned}
% \end{equation}
\begin{equation}
    \begin{aligned}
        \E\sqbrac{\norm{\nabla_\w F(\z^+;\xi_{k}^s)-\nabla_\w F(\z;\xi_{k}^s)}_{w,*}^2}\le&2L^2\norm{\w^+-\w}_{w}^2\sum_{l_i=1}^m\frac{m\bar n}{n_{l_i}}\q^+_{l_i}+2G^2\sum_{l_i=1}^m\frac{m\bar n}{n_{l_i}}\brac{\q^+_{l_i}-\q_{l_i}}^2\\\le&2m\bar nL^2\norm{\w^+-\w}_{w}^2\sum_{l_i=1}^m\frac{\q^+_{l_i}}{n_{\min}}+2m\bar nG^2\sum_{l_i=1}^m\frac{\brac{\q^+_{l_i}-\q_{l_i}}^2}{n_{\min}}\\\le&2m\frac{\bar n}{n_{\min}}L^2\norm{\w^+-\w}_{w}^2+2m\frac{\bar n}{n_{\min}}G^2\norm{\q^+-\q}_1^2.
    \end{aligned}
\end{equation}
Correspondingly, we consider the stochastic gradient w.r.t.~$\q$ based on~\cref{lem:2order-moment-bound},
\begin{equation}
    \E\sqbrac{\norm{\nabla_\q F(\z^+;\xi_{k}^s)-\nabla_\q F(\z;\xi_{k}^s)}_{\infty}^2}\le \sum_{l_i=1}^m \frac{m\bar n}{n_{l_i}}\cdot G^2\norm{\w^+-\w}_w^2=m^2\frac{\bar n}{\bar n_h}G^2\norm{\mathbf w^+-\mathbf w}_w^2.
    \end{equation}
We finish the proof by following the similar merging process in~\eqref{eq:simple-merge} (cf.~the proof of~\cref{lem:Lipschitz}).
\end{proof}

The following derivation process is analogous to the proof of~\cref{lem:sg-uniform-sampling}.

\begin{proof}[Proof of~\cref{lem:sg-importance-sampling}]
 First, we study the first-order moment information of the stochastic gradient in~\eqref{eq:importance-sampling-stochastic-gradient}. Considering the stochastic gradient w.r.t.~$\w$, it follows that
% \begin{equation}
% \begin{aligned}
% \mathbb E[\nabla_{\mathbf w} F(\mathbf z;\xi_k^s)]=&\mathbb E\left[\mathbb E\left[m\mathbf q_{l_i}\nabla\ell(\mathbf w;\xi_{l})\Big|l_i\right]\right]=\mathbb E\left[m\mathbf q_{l_i}\nabla R_{l_i}(\mathbf w)\right]\\=&\sum_{{l_i}=1}^m\frac{1}{m}\cdot m\mathbf q_{l_i}\nabla R_{l_i}(\mathbf w)=\nabla_{\mathbf w} F(\mathbf z).
% \end{aligned}
% \end{equation}
\begin{equation}
\begin{aligned}
\mathbb E[\nabla_{\mathbf w} F(\mathbf z;\xi_k^s)]=\mathbb E\left[\mathbb E\left[m\mathbf q_{l_i}\nabla\ell(\mathbf w;\xi_{l})\Big|l_i\right]\right]=\mathbb E\left[m\mathbf q_{l_i}\nabla R_{l_i}(\mathbf w)\right]=\sum_{{l_i}=1}^m\frac{1}{m}\cdot m\mathbf q_{l_i}\nabla R_{l_i}(\mathbf w)=\nabla_{\mathbf w} F(\mathbf z).
\end{aligned}
\end{equation}
Then we consider the stochastic gradient w.r.t.~$\q$. 
\begin{equation}
\begin{aligned}
\mathbb E[\nabla_{\mathbf q} F(\mathbf z;\xi_k^s)]=\mathbb E\left[\E\sqbrac{m\ell(\w;\xi_l)\cdot \mathbf{e}_{l_i}\Big|l_i}\right]=\E\sqbrac{mR_{l_i}(\w)\cdot \mathbf{e}_{l_i}}=\nabla_{\mathbf q} F(\mathbf z).
\end{aligned}
\end{equation}
Secondly, we analyze the continuity according to~\cref{lem:2order-moment-bound}. With $C_i=m$, it holds that:
% \begin{equation}
%     \begin{aligned}
%         &\E\sqbrac{\norm{\nabla_\w F(\z^+;\xi_{k}^s)-\nabla_\w F(\z;\xi_{k}^s)}_{w,*}^2}\\\le&2L^2\norm{\w^+-\w}_{w}^2+2G^2\sum_{l_i=1}^mm\brac{\q^+_{l_i}-\q_{l_i}}^2\\\le&2mL^2\norm{\w^+-\w}_{w}^2\sum_{l_i=1}^m\frac{\q^+_{l_i}}{n_{\min}}+2mG^2\norm{\q^+-\q}_1^2.
%     \end{aligned}
% \end{equation}
\begin{equation}
    \begin{aligned}
        \E\sqbrac{\norm{\nabla_\w F(\z^+;\xi_{k}^s)-\nabla_\w F(\z;\xi_{k}^s)}_{w,*}^2}\le&2L^2\norm{\w^+-\w}_{w}^2+2G^2\sum_{l_i=1}^mm\brac{\q^+_{l_i}-\q_{l_i}}^2\\\le&2mL^2\norm{\w^+-\w}_{w}^2\sum_{l_i=1}^m\frac{\q^+_{l_i}}{n_{\min}}+2mG^2\norm{\q^+-\q}_1^2.
    \end{aligned}
\end{equation}
\begin{equation}
    \E\sqbrac{\norm{\nabla_\q F(\z^+;\xi_{k}^s)-\nabla_\q F(\z;\xi_{k}^s)}_{\infty}^2}\le \sum_{l_i=1}^m m\cdot G^2\norm{\w^+-\w}_w^2=m^2G^2\norm{\mathbf w^+-\mathbf w}_w^2.
\end{equation}
The proof is finished by a simple merging process as in~\eqref{eq:simple-merge}.
\end{proof}

\begin{proof}[Proof of~\cref{thm:mpvr-convergence}]
    The theoretical guarantee is a direct application of~\citet[Theorem~8]{alacaoglu2022stochastic}. We only need to substitute the Lipschitz constant by $L_{\mathbf c}$.
\end{proof}

\begin{proof}[Proof of~\cref{cor:mpvr-complexity}]
First, we investigate the order of $L_{\mathbf{c}}$ according to the sampling strategy.

For uniform sampling, we have a \textit{data-dependent} $L_{\mathbf{u}}$ in~\eqref{eq:Lu}. The value of $\frac{\bar n}{n_{\min}}$ ranges from $1$ to $\bar n$, indicating that the order could be worse as $\Theta(m)$. Due to the fact that $\frac{\bar n}{\bar n_h}\ge 1$, we can derive that $L_{\mathbf{u}}=\OO\brac{m\sqrt{\ln m}}$. For importance sampling, we can directly assert that $L_{\mathbf{i}}=\OO\brac{m\sqrt{\ln m}}$ in view of~\eqref{eq:Li}. Therefore, we conclude that $L_{\mathbf{c}}=\OO\brac{m\sqrt{\ln m}}$. 

Before calculating the computation complexity, we need to verify the tightness of the Lipschitzness. A linear loss item for $\ell(\cdot;\xi)$ realizes the lower bound. From~\cref{lem:2order-moment-bound,remark:2order-moment-bound} we can further confirm that the order of $L_{\mathbf{c}}$ could not be improved.
    
% Using the same calculation as in the proof of~\cref{thm:complexity,thm:mero-complexity}, we can derive a suboptimal complexity of $\OO\brac{m\bar n+\frac{m\sqrt{m\bar n\ln{m}}}{\varepsilon}}$.
The inner loop and outer loop of~\cref{alg:mpvr} require $\Theta\brac{KS}$ and $\Theta\brac{m\bar nS}$ computations in total, respectively. To set the two costs in the same order, we choose $K=\Theta(m\bar n)$. Consequently, we have $S=\OO\brac{\frac{L_{\mathbf{c}}}{\varepsilon\sqrt{m\bar n}}}$. Therefore, we derive a complexity of $\OO\brac{\frac{L_{\mathbf{c}}\sqrt{m\bar n}}{\varepsilon}}$. Because the initial cost for the full gradient is $m\bar n$ and $L_{\mathbf{c}}=\OO\brac{m\sqrt{\ln m}}$, we get a total complexity of $\OO\brac{m\bar n+\frac{m\sqrt{m\bar n\ln{m}}}{\varepsilon}}$.
\end{proof}

Finally, we present a technical lemma used in the previous analysis, in which we can see why ALEG goes beyond the pure extension of MPVR.

\begin{lemma}
    Let $C_i=\frac{m\bar n}{n_{l_i}}$ for uniform sampling and $C_i=m$ for importance sampling, then for any $\z^+,\z\in\Z$, we have the following estimate
    \begin{equation}
        \E\sqbrac{\norm{\nabla_\w F(\z^+;\xi_{k}^s)-\nabla_\w F(\z;\xi_{k}^s)}_{w,*}^2}\le2L^2\norm{\w^+-\w}_{w}^2\sum_{l_i=1}^mC_i\q^+_{l_i}+2G^2\sum_{l_i=1}^mC_i\brac{\q^+_{l_i}-\q_{l_i}}^2
    \end{equation}
    for the stochastic gradient w.r.t.~$\w$, and 
    \begin{equation}
        \E\sqbrac{\norm{\nabla_\q F(\z^+;\xi_{k}^s)-\nabla_\q F(\z;\xi_{k}^s)}_{\infty}^2}\le \sum_{l_i=1}^m C_i\cdot G^2\norm{\w^+-\w}_w^2
    \end{equation}
    for the stochastic gradient w.r.t.~$\q$.\label{lem:2order-moment-bound}
\end{lemma}

\begin{proof}
% It's straightforward to verify that $\E\sqbrac{C^2}=\sum_{l_i=1}^m\frac{C^2}{C}\sum_{j=1}^{n_{l_i}}\frac{1}{n_{l_i}}=\sum_{l_i=1}^m C$ holds for both sampling techniques.
Using the tower rule, it's straightforward to verify that
\begin{equation}
    \E\sqbrac{C_i^2}=\sum_{l_i=1}^m\frac{C_i^2}{C_i}\sum_{j=1}^{n_{l_i}}\frac{1}{n_{l_i}}=\sum_{l_i=1}^m C_i
\end{equation}
holds for both sampling techniques.
    For the stochastic gradient w.r.t.~$\w$, we have
    \begin{equation}
    \begin{aligned}
        &\E\sqbrac{\norm{\nabla_\w F(\z^+;\xi_{k}^s)-\nabla_\w F(\z;\xi_{k}^s)}_{w,*}^2}\\=& \E\sqbrac{C_i^2\norm{ \q^+_{l_i}[\nabla\ell(\w^+;\xi_{l})-\nabla\ell(\w;\xi_{l})]+\brac{\q^+_{l_i}-\q_{l_i}}\nabla\ell(\w;\xi_{l})}_{w,*}^2}\\\le& 2\E\sqbrac{C_i^2\norm{ \q^+_{l_i}\sqbrac{\nabla\ell(\w^+;\xi_{l})-\nabla\ell(\w;\xi_{l})}}_{w,*}^2}+2\E\sqbrac{C_i^2\norm{ \brac{\q^+_{l_i}-\q_{l_i}}\nabla\ell(\w;\xi_{l})}_{w,*}^2}\\\le&2\sum_{l_i=1}^mC_i\sum_{j=1}^{n_{l_i}}\frac{1}{n_{l_i}}\norm{ \q^+_{l_i}\sqbrac{\nabla\ell(\w^+;\xi_{l})-\nabla\ell(\w;\xi_{l})}}_{w,*}^2+2\E\sqbrac{C_i^2\brac{\q^+_{l_i}-\q_{l_i}}^2\norm{ \nabla\ell(\w;\xi_{l})}_{w,*}^2}\\\le&2\sum_{l_i=1}^m\sum_{j=1}^{n_{l_i}}\frac{C_i}{n_{l_i}}\norm{ \sqbrac{\nabla\ell(\w^+;\xi_{l})-\nabla\ell(\w;\xi_{l})}}_{w,*}^2+2\sum_{l_i=1}^mC_i\brac{\q^+_{l_i}-\q_{l_i}}^2\sum_{j=1}^{n_{l_i}}\frac{1}{n_{l_i}}\norm{ \nabla\ell(\w;\xi_{l})}_{w,*}^2\\\le&2\sum_{l_i=1}^m\sum_{j=1}^{n_{l_i}}\frac{C_i\q^+_{l_i}}{n_{l_i}}L^2\norm{\w^+-\w}_{w}^2+2\sum_{l_i=1}^mC_i\brac{\q^+_{l_i}-\q_{l_i}}^2\sum_{j=1}^{n_{l_i}}\frac{1}{n_{l_i}}G^2\\\le&2L^2\norm{\w^+-\w}_{w}^2\sum_{l_i=1}^mC_i\q^+_{l_i}+2G^2\sum_{l_i=1}^mC_i\brac{\q^+_{l_i}-\q_{l_i}}^2.
    \end{aligned}      
    \end{equation}
    As for the stochastic gradient w.r.t.~$\q$, we have
    \begin{equation}
    \begin{aligned}
        \E\sqbrac{\norm{\nabla_\q F(\z^+;\xi_{k}^s)-\nabla_\q F(\z;\xi_{k}^s)}_{\infty}^2}\le& \E\sqbrac{\norm{C_i\sqbrac{\ell(\w^+;\xi_l)-\ell(\w;\xi_l)}\cdot \mathbf{e}_{l_i}}_{\infty}^2}\\\le& \E\sqbrac{C_i^2\sqbrac{\ell(\w^+;\xi_l)-\ell(\w;\xi_l)}^2}\\\le&\E\sqbrac{C_i^2}\cdot G^2\norm{\w^+-\w}_w^2.\label{eq:q-grad-estimate}
    \end{aligned}      
    \end{equation}
    Plug the result for $\E\sqbrac{C_i^2}$ into~\eqref{eq:q-grad-estimate} yields the desired estimate.
\end{proof}

\begin{remark}
    The subtle issue is closely related to the property of the infinity norm as shown in the second inequality of~\eqref{eq:q-grad-estimate}, where the aggregation of the infinity norm for $m$ one-hot vectors scales up by $m$ times. However, our \textit{group sampling}-based stochastic gradient construction can solve this issue by switching the order of the summation and infinity norm (check~\eqref{eq:iGw^+-w} to see our estimate), and thereby produce a better stochastic gradient with lower Lipschitz constant.\label{remark:2order-moment-bound}
\end{remark}

\end{document}